\documentclass{article}

\usepackage[final,nonatbib]{neurips_2021}

\usepackage[utf8]{inputenc} %
\usepackage[T1]{fontenc}    %
\usepackage{hyperref}       %
\usepackage{url}            %
\usepackage{booktabs}       %
\usepackage{amsfonts}       %
\usepackage{nicefrac}       %
\usepackage{microtype}      %
\usepackage{xcolor}         %
\usepackage{wrapfig}
\usepackage{makecell}
\usepackage[utf8]{inputenc} %
\usepackage[T1]{fontenc}    %
\usepackage{hyperref}       %
\usepackage{url}            %
\usepackage{booktabs}       %
\usepackage{amsfonts}       %
\usepackage{nicefrac}       %
\usepackage{microtype}      %
\usepackage{amsmath}
\usepackage{graphicx, graphics, epsfig, color}
\usepackage{subfigure}
\usepackage{tikz, pgfplots}
\usepackage{algorithm, algorithmic}
\usepackage{float}
\usepackage{multirow}
\usepackage{cuted, flushend}
\usepackage{midfloat}
\usepackage{bm}
\usepackage{fancyhdr}
\usepackage{enumerate}
\usepackage[numbers,square,sort&compress]{natbib}
\usepackage{bbding}
\usepackage[flushleft]{threeparttable}
\usepackage{wasysym}
\usepackage{amsthm}
\usepackage{chngcntr}
\usepackage{apptools}
\AtAppendix{\counterwithin{myTheo}{section}}
\AtAppendix{\counterwithin{lemma}{section}}
\AtAppendix{\counterwithin{myproof}{section}}

\newtheorem{myDef}{Definition}
\newtheorem{myTheo}{Theorem}

\newtheorem{myassum}{Assumption}
\newtheorem{lemma}{Lemma}

\newcommand{\BLACK}[1]{{\color{black} #1}}
\usepackage{color-edits}
\addauthor{ProfSun}{blue}
\addauthor{ProfLuo}{red}
\addauthor{Jiawei}{orange}
\addauthor{mh}{purple}

\usepackage{xr}
\usepackage{listings}

\definecolor{codegreen}{rgb}{0,0.6,0}
\definecolor{codegray}{rgb}{0.5,0.5,0.5}
\definecolor{codepurple}{rgb}{0.58,0,0.82}
\definecolor{backcolour}{rgb}{0.95,0.95,0.92}
\definecolor{yushun}{rgb}{0,0.6,0}%
\definecolor{lyr}{rgb}{0.52,0,0.52}
\definecolor{lzn}{rgb}{0.58,0,0.82}

\lstdefinestyle{mystyle}{
    backgroundcolor=\color{backcolour},   
    commentstyle=\color{codegreen},
    keywordstyle=\color{magenta},
    numberstyle=\tiny\color{codegray},
    stringstyle=\color{codepurple},
    basicstyle=\ttfamily\footnotesize,
    breakatwhitespace=false,         
    breaklines=true,                 
    captionpos=b,                    
    keepspaces=true,                 
    numbers=left,                    
    numbersep=5pt,                  
    showspaces=false,                
    showstringspaces=false,
    showtabs=false,                  
    tabsize=2
}

\AtBeginDocument{\setlength\abovedisplayskip{2pt}}
\AtBeginDocument{\setlength\belowdisplayskip{2pt}}

\title{When Expressivity Meets Trainability: Fewer than $n$ Neurons Can Work}

\author{%
  Jiawei Zhang$^*$ \\
 Shenzhen Research Institute of Big Data \\
 The Chinese University of Hong Kong,\\
Shenzhen, China\\
\texttt{jiaweizhang2@link.cuhk.edu.cn} \\
  \And
   Yushun Zhang\thanks{Equal contribution. These authors are listed in alphabetical order.} \\
 Shenzhen Research Institute of Big Data \\
 The Chinese University of Hong Kong,\\
Shenzhen, China\\
 \texttt{yushunzhang@link.cuhk.edu.cn} \\
   \AND
   Mingyi Hong \\
 University of Minnesota - Twin Citie\\
 MN, USA \\
 \texttt{mhong@umn.edu} \\
   \AND
   Ruoyu Sun \thanks{Corresponding author: Ruoyu Sun.}\\
University of Illinois at Urbana-Champaign \\
IL, USA \\
 \texttt{ruoyus@illinois.edu} \\
   \And
   Zhi-Quan Luo \\
 Shenzhen Research Institute of Big Data \\
 The Chinese University of Hong Kong,\\
Shenzhen, China\\
 \texttt{luozq@cuhk.edu.cn} \\
}

\begin{document}

\maketitle

\begin{abstract}
Modern neural networks are often quite wide, causing large memory and computation costs. It is thus of great interest to train a narrower network. However, training narrow neural nets remains a challenging task. 
We ask two theoretical questions: Can narrow networks have as strong expressivity as wide ones? If so, does the loss function exhibit a  benign optimization landscape? In this work,  we provide partially affirmative answers to both questions for 1-hidden-layer networks with fewer than $n$ (sample size) neurons when the activation is smooth.
  First, we prove that as long as the width $m \geq 2n/d$ (where $d$ is the input dimension), its expressivity is strong, i.e., there exists at least one global minimizer with zero training loss.
Second, 
we identify a nice local region with no local-min or
saddle points.
 Nevertheless, it is not clear whether gradient
 descent can stay in this nice region.
 Third, we consider a constrained optimization formulation where the feasible region is the nice local region, and prove that every KKT point is a nearly global minimizer. 
 It is expected that projected gradient methods
 converge to KKT points under mild technical conditions,
 but we leave the rigorous convergence analysis to future work.
 Thorough numerical results show that projected gradient methods
 on this constrained formulation significantly
 outperform SGD for training narrow neural nets. 
 
\end{abstract}

\vspace{-2mm}
\section{Introduction} \label{section:intro}
\vspace{-2mm}
Modern neural networks are huge (e.g. \cite{zagoruyko2016wide, brown2020language}). 
Reducing the size of neural nets 
is appealing for many reasons:
first, small networks are more suitable for embedded systems and portable devices;
second, using smaller networks can reduce power
consumption, contributing to ``green computing''. 
There are many ways to reduce network size,
such as quantization, sparcification and reducing
the width (e.g. \cite{zhou2017incremental,han2015deep}).
In this work, we focus on reducing the width
(training narrow nets).

Reducing the network width often leads to significantly
worse performance \footnote{This can be verified on our
empirical studies in Section \ref{section:experiment}.
Another evidence is that structure pruning (reducing the number of channels in convolutional neural nets (CNN))
is known to achieve worse performance than unstructured pruning; this is an undesirable situation since many practitioners prefer structure pruning (due to hardware reasons). 
}. 
What is the possible cause?
From the theoretical perspective, there are three possible causes: worse generalization power,
worse trainability (how effective a network can be optimized), and weaker expressivity 
(how complex the function a network can represent; see Definition  \ref{def:expressivity}). 
Our simulation shows that the training error
 deteriorates significantly as the width shrinks, 
which implies that the trainability and/or expressivity
are important causes of the worse performance (see Section \ref{section:synthetic} for more evidence 
).
We do not discuss generalization power for now,
 and leave it to future work. 
 
 \BLACK{So how is the training error related to  expressivity and trainability?  The training error is the sum of two parts (see, e.g., \cite{sun2020optimization}): the expressive error (which is the best a given network can do; also the global minimal error) and the optimization error (which is the gap between training error and the global minimal error; occurs because the algorithm may not find global-min). The two errors are of different nature, and thus need to be discussed separately. }

It is understandable that 
narrower networks might have weaker expressive power.
What about optimization?
There is also evidence that smaller
width causes optimization difficulty. 
A number of recent works show that 
increasing the width of neural networks helps create a benign empirical loss landscape (\cite{li2018benefit,safran2016quality,freeman2016topology}), 
 while narrow networks (width $m$ < sample size $n$) suffer from bad landscape 
 (\cite{auer1996exponentially,swirszcz2016local,zhou2017critical,safran2017depth,yun2018small}).
 Therefore, if we want to improve the performance
 of narrow networks, it is likely that
 both expressiveness and trainability need to be improved.

The above discussion leads to the following two questions: 
\begin{center}
 {\bf (Q1)} {\it  Can a narrow network have as strong {\bf expressivity} as a wide one? 
 
 {\bf (Q2)}  If so, can a local search method
 find a (near) globally optimal solution? 
 }
\end{center}

The key challenges in answering these questions are listed below:
\BLACK{
\vspace{-1mm}
\begin{itemize}
  \item It is not clear whether a narrow network has strong expressivity or not.
 Many existing works focus on verifying the relationship between  zero-training-error solutions and stationary points, but they neglect the (non)existence of such solutions (e.g.   \cite{xie2017diverse}, \cite{soudry2016no}). For narrow networks, the (non)-existence of zero-training-error-solution is not clear.
 
  \item Even if zero-training-error solutions do exist,  it is still not clear how to reach those solutions because the landscape of a narrow neural network can be highly non-convex. 
  
  \item Even assuming that we can identify a region that contains zero-training-error solutions and has a good landscape, it is potentially difficult to keep the iterates inside such a good region. One may think of imposing an explicit constraint, but this approach might introduce bad local minimizers on the boundary \cite{bertsekas1997nonlinear}.
\end{itemize}
}
In this work, we (partially) answer {\bf (Q1)} and {\bf (Q2)} for a  1-hidden-layer nets with fewer than $n$ neurons. Our main contributions are as follows: 
\begin{itemize}
  \item 
  \textbf{Expressiveness and nice local landscape.} 
  We prove that, as long as the width $m$ is larger than $ 2n/d$ (where $n$ is the sample size and $d$ is the input dimension),
  then the expressivity of 
  the 1-hidden-layer net is strong, i.e., w.p.1. there {\it exists} at least one global-min with zero empirical loss.
  In addition,
  such a solution is surrounded by a good local landscape with no local-min or saddles. 
  Note that our results do not exclude the possibility that
  there are sub-optimal local minimizers on the \textit{global} landscape. 
  
  \item \textbf{Every KKT point is an approximated global minimizer.}
  For the original unconstrained optimization problem,
  the nice {\it local} landscape does not guarantee the {\it global} statement of  ``every stationary point is a global minimizer''.
 We propose a constrained optimization problem that restricts the hidden weights to be close to the identified nice region. %
 We show that every Karush–Kuhn–Tucker (KKT) point 
 is an approximated global minimizer
 of the unconstrained training problem 
 \footnote{ 
 This result describes the loss landscape
 of the constrained optimization problem, not directly
 related to algorithm convergence. 
 Nevertheless, it is expected that first-order methods
 converge to KKT points and thus approximate global minimizers. A rigorous convergence analysis may require verifying extra technical conditions, which is left to future work. 
 }.
  
 \item In real-data experiments, our proposed training regime can significantly outperforms SGD for training narrow networks.  
 We also perform ablation studies to show that the new elements proposed in our method
 are useful. 
\end{itemize}

\vspace{-2mm}
\section{Background and Related Works}
\vspace{-2mm}
\label{section:relatedwork}

\BLACK{The expressivity of neural networks has been a popular topic
in machine learning for decades. There are two lines of works: One focuses on the {\it infinite-sample} expressivity, showing what functions of the entire domain can and cannot be represented by certain classes of neural networks (e.g. \cite{barron1993universal,  lu2017expressive}). Another line of works characterize the {\it finite-sample} expressivity, i.e. how many parameters are required to memorize finite samples (e.g. \cite{baum1988capabilities,huang1998upper,delalleau2011shallow,huang2003learning,hanin2019complexity,zhang2021understanding, livni2014computational}). The term ``expressivity'' in this work means the finite-sample expressivity; see Definition \ref{def:expressivity} in Section \ref{section:setting}.
A major research question regarding expressivity in the area
is to show  {\it deep} neural networks have much stronger
expressivity than the {\it shallow} ones
(e.g. \cite{bianchini2014complexity, montufar2014number,eldan2016power,telgarsky2016benefits,lin2017does, rolnick2017power, park2020provable,yun2018small,vershynin2020memory}).
However, all these works neglect the trainability.}

In the finite-sample case, wide networks (width $\text{poly}(n)$) have both strong representation power (i.e.
the globally minimal training error is zero)
and strong trainability (e.g. for wide enough nets,
Gradient Descent (GD) converges to global minima \cite{du2019gradient,jacot2018neural,allen2019convergence,zou2018stochastic}).
While these wide networks are often called ``over-parameterized'', we notice that the number of parameters of a width-$n$ network is actually at least $n d $, which is much larger than  $ n $. If comparing $n$ with the number of parameters (instead of neurons),
the transition from under-parameterization and over-parameterization for a one-hidden-layer fully-connected net (FCN) does not occur at width-$n$,
but at width-$n/d$.
In this work, we will analyze networks with width in the range
$ [n/d, n ) $, which we call ``narrow networks'' (though rigorously speaking, we shall call them ``narrow but still overparameterized networks''). 

\BLACK{There are a few works on the trainability of narrow nets (one-hidden-layer networks with $m\geq n/d$ neurons).}
Soudry and Carmon \cite{soudry2016no},  Xie et al. \cite{xie2017diverse}  show that for such networks,
stationary points with full-rank NTK (neural tangent kernel)
are zero-loss global minima. 
However, it is not clear whether the NTK stays full rank during the training trajectory. 
In addition, these two works do not discuss whether a
zero-loss global minimizer exists.

There are two interesting related works \cite{bubeck2020network,daniely2019neural} pointed out by the reviewers. 
Bubeck et al. \cite{bubeck2020network} study how many neurons are required for memorizing a finite dataset by 1-hidden-layer networks. They prove the following results.
Their first result is
an ``existence'' result:
 there exists a network with width $m \geq \frac{4 n}{d}$ which can memorize $n$ input-label pairs (their Proposition 4). However, in this setting they did not provide an algorithm to find the zero-loss solution. 
Their second result is related to algorithms: they proposed a training algorithm that achieves accuracy up to error $\epsilon$ for a neural net
with width 
$ m 
\geq  O \left(\frac{n}{d} \frac{\log (1 / \epsilon)}{\epsilon} 
\right)$. 
This result requires width dependent on the precision $\epsilon$; for instance,
when the desired accuracy $\epsilon = 1/n$, the required width 
is at least $ 
O \left(\frac{n^2}{d} \right). $
In contrast, in our work, the required number of neurons is just $ 2n/d$, which is independent of $\epsilon$.

Daniely \cite{daniely2019neural}  also studies the expressivity and trainability of 1-hidden-layer networks. To memorize $n(1-\epsilon)$ random data points via SGD, their required width is $\tilde{O} (n/d)$. 
They assumed $n=d^c$ where $c>0$ is a {\it fixed} constant (appeared in Sec. 3.3 of \cite{daniely2019neural}),
in which case the hidden factor in $\tilde{O}$ is
 $O(\log [  d (\log d) ]^c ) $. In other words, if $n, d \rightarrow \infty $ with the scaling $ n = d^c $ for a fixed constant $c$, then their bound is roughly $O(n/d)$ 
  up to a log-factor. 
Nevertheless, for more general
scaling of $n , d$, the exponent $c = 
( \log n )/ ( \log d ) $ may not be a constant and
the hidden factor may not be a log-factor (e.g. for fixed $d$ and $n \rightarrow \infty$). 
We tracked their proof and find that the width bound for general $n, d $ is $ O\left((n/d) \left(\log (d \log n) \right)^{\log n / \log d} \right)$,
 which can be larger than $O\left(n^2/d\right)$
(see detailed computation and explanation in Appendix \ref{appendix:relatedwork}).
 In contrast,  our required width is $2n/d$  for arbitrary $n$ and $d$. Our bound is always smaller than $n$ when $d>2$.
 
Additionally, there is a major difference between Daniely 
\cite{daniely2019neural} 
and our work: they analyze the original unconstrained problem and SGD; in contrast, we analyze a constrained problem.
This may be the reason why we can get a stronger bound on width. 
In the experiments in Section \ref{section:experiment}, we observe that SGD performs badly when the width is small (see the 1st column in  Figure \ref{synthetic:loss} (b)). Therefore, we suspect an algorithmic change is needed to train narrow nets with such width (due to the training difficulty), and we indeed propose a new method to train narrow nets.

\BLACK{Due to the space constraints, we defer more related works in Appendix \ref{appendix:morerelatedwork}.}

\vspace{-2mm}
\section{Challenges For Analyzing Narrow Nets}
\label{section:challange}
\vspace{-2mm}
In this section, we discuss why it is challenging
to achieve expressivity and trainability together
 for narrow nets. 
Consider a dataset $\left\{\left(x_{i}, y_{i}\right)\right\}_{i=1}^{n} \subset \mathbb{R}^d \times \mathbb{R}$ and a 1-hidden-layer neural network: 
\begin{equation} \label{eq1:simplenn}
  f(x ; \theta)=\sum_{j=1}^{m} v_{j} \sigma\left(w_{j}^{T} x\right),
\end{equation}
where $\sigma\left(w_{j}^{T} x\right)$ is the output of the $j$-th hidden nodes with hidden weights $w_{j}$, $\sigma(\cdot)$ is the activation function, and $v_{j}$ is the corresponding outer weight (bias terms are ignored for simplicity). 
To learn such a neural network, we search for the optimal parameter $\theta=(w, v)$ by minimizing the following empirical (training) loss:
\begin{equation}\label{unconstrained}
  \min _{\theta} \ell(\theta)=\frac{1}{2} \sum_{i=1}^{n}\left(y_{i}-f\left(x_{i} ; \theta\right)\right)^{2},
\end{equation}
The gradient of the above problem %
w.r.t. hidden weights $w=\{w_{i}\}_{i=1}^m$ is given by:
$
  \nabla_w\ell(\theta)=J(w;v)^T(f(w;v)-y) \in \mathbb{R}^{md \times 1},
$
where $J(w;v)\in \mathbb{R}^{n \times md} $ is the Jacobian matrix w.r.t $w$:

{\small
\begin{equation} \label{jacobian}
J(w;v):=\left[\begin{array}{c}\nabla_w f\left(w ; x_{1}, v\right)^{T} \\ \vdots \\  \nabla_w f\left(w ; x_{n}, v\right)^{T} \end{array}\right]  =\left[\begin{array}{ccc}v_{1} \sigma^{\prime}\left(w_{1}^{T} x_{1}\right) x_{1}^T & \cdots & v_{m} \sigma^{\prime}\left(w_{m}^{T} x_{1}\right) x_{1}^T \\ &\vdots & \\ v_{1} \sigma^{\prime}\left(w_{1}^{T} x_{n}\right) x_{n}^T & \cdots & v_{m} \sigma^{\prime}\left(w_{m}^{T} x_{n}\right) x_{n}^T\end{array}\right]\in \mathbb{R}^{n \times md}.
\end{equation}}

First order methods like  GD converge to a stationary point
$\theta^*=(w^*,v^*)$ (i.e. with zero gradient) under mild
conditions \cite{bertsekas1997nonlinear}. 
For problem \eqref{unconstrained}, 
it is easy to show that
if \BLACK{(i)} $(w^*,v^*)$ is a stationary point, \BLACK{(ii)}  $J(w^*;v^*) \in \mathbb{R}^{n\times md}$ is full row rank and 
$ n \leq  m d $, then  $(w^*,v^*)$ is a global-min \BLACK{(this claim can be proved by setting the partial gradient of \eqref{unconstrained} over $w$ to be zero)}. 
In other words, for training a network with width
$ m \geq n/d$, an important tool is to ensure the
 full-rankness of the Jacobian.

Recent works have shown that  it is possible to guarantee the full rankness of the Jacobian matrix along the training trajectories, however, the required width is above $\Omega(\text{poly}(n))$. Roughly speaking, the proof sketch is the following: 
(i) with high probability, $J(w;v)$ is non-singular locally around the random initialization (\cite{xie2017diverse}, \cite{soudry2016no}), 
(ii) increasing the width can effectively bound the parameter movement from initialization, so the “nice” 
property of non-singualr $J(w,v)$ holds throughout the training, leading to a linear convergence rate \cite{du2019gradient}. Under this general framework, a number of convergence results are developed for wide networks with width $\Omega(\text{poly}(n))$ (\cite{zou2018stochastic,allen2019convergence,zou2019improved,noy2021convergence,nguyen2021proof,lee2019wide,jacot2018neural,chizat2018lazy,oymak2020toward,oymak2019overparameterized}). This idea is also illustrated in Figure \ref{figure:idea} (a).

We notice that there is a huge gap between
the necessary condition $ m \geq n/d$
 and the common condition $ m \geq 
  \Omega(\text{poly}(n)) .$
 We suspect that it is possible to train a narrow net
 with width $ \Theta( n / d) $ to small loss.
To achieve this goal, we need to understand why existing arguments require a large width and cannot apply to a network with width $ \Theta( n / d) $.

The first reason is about trainability. 
The above arguments no longer hold 
when the width is not large enough to control the movement of hidden weights. In this case, the iterates may easily travel far away from the initial point and get stuck at some singular-Jacobian critical points with high training loss (see Figure \ref{figure:idea} (b). Also see Figure \ref{synthetic:loss} (b) \& (d) for more empirical evidence). 
In other words, GD may get stuck at sub-optimal
stationary points for narrow nets.

      \begin{figure}[htbp]
      \vspace{-2mm}
        \centering
        \subfigure[Wide networks]{
        \begin{minipage}[t]{0.3\linewidth}
        \centering
        \includegraphics[width=1in]{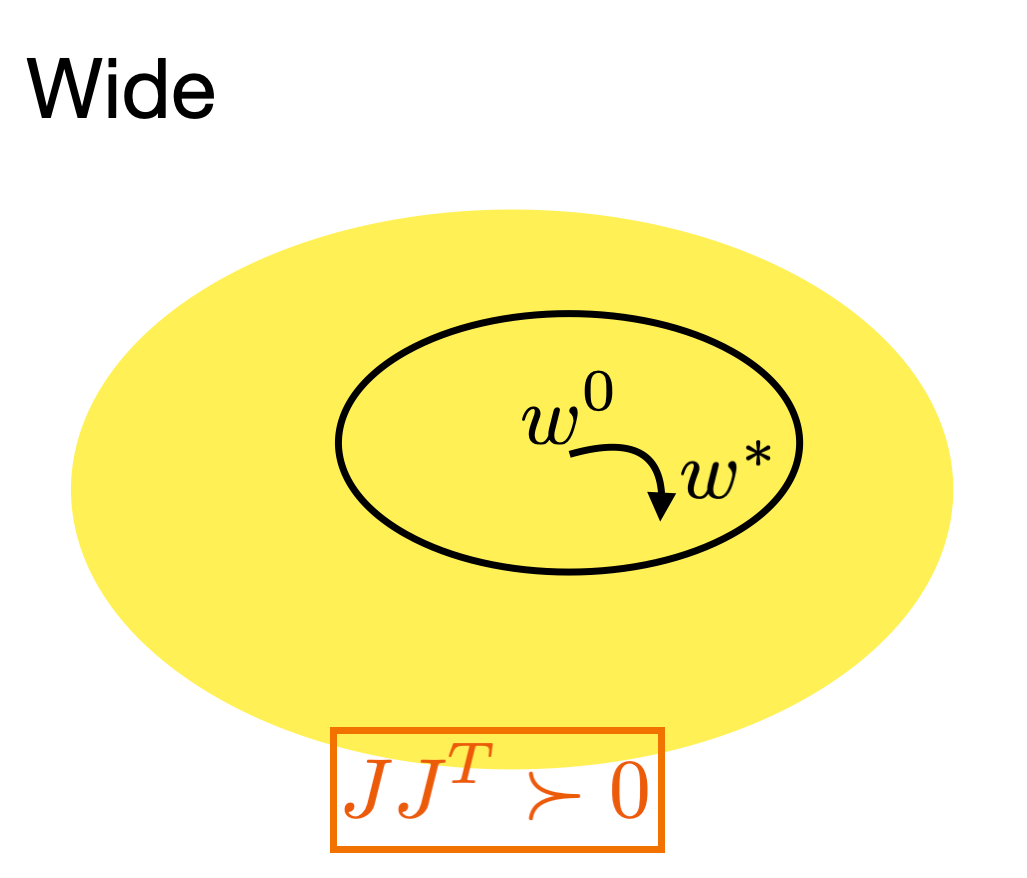}
        \end{minipage}%
        }%
        \subfigure[Narrow networks]{
          \begin{minipage}[t]{0.3\linewidth}
          \centering
          \includegraphics[width=1in]{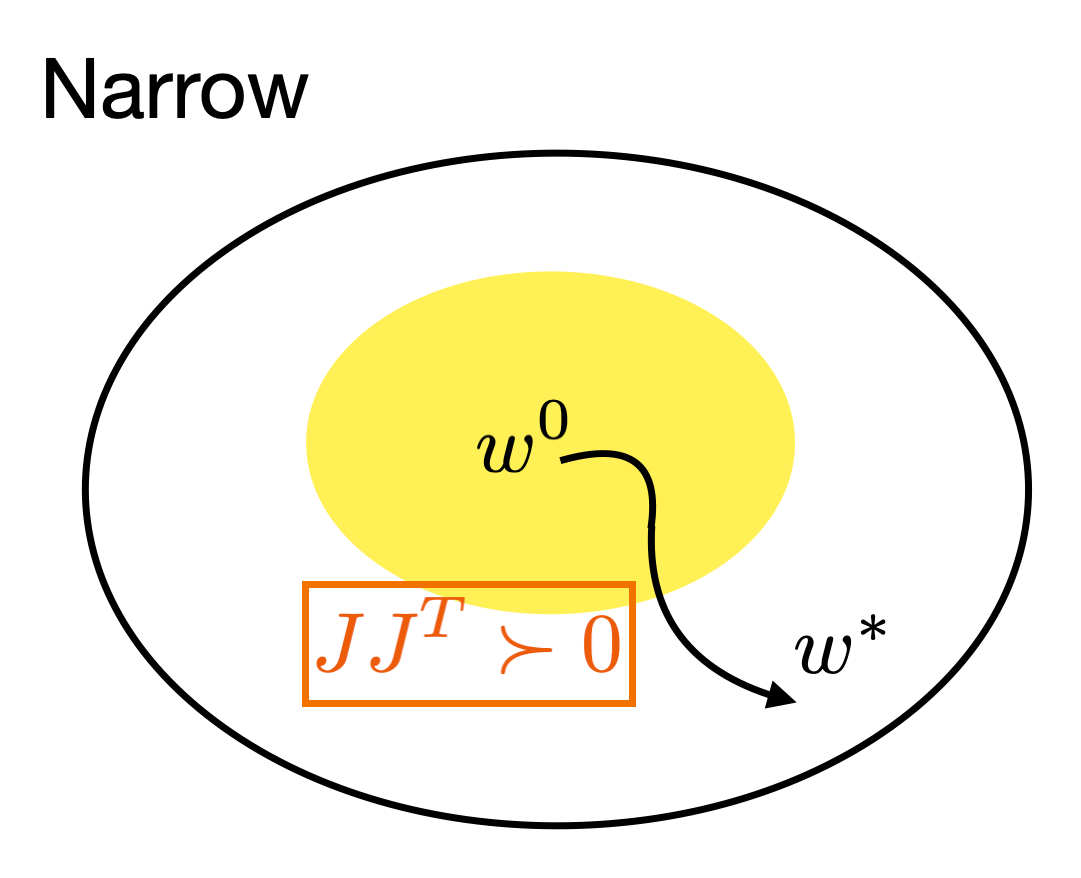}
          \end{minipage}%
          }%
        \subfigure[Narrow networks (our regime)]{
          \begin{minipage}[t]{0.3\linewidth}
          \centering
          \includegraphics[width=1in]{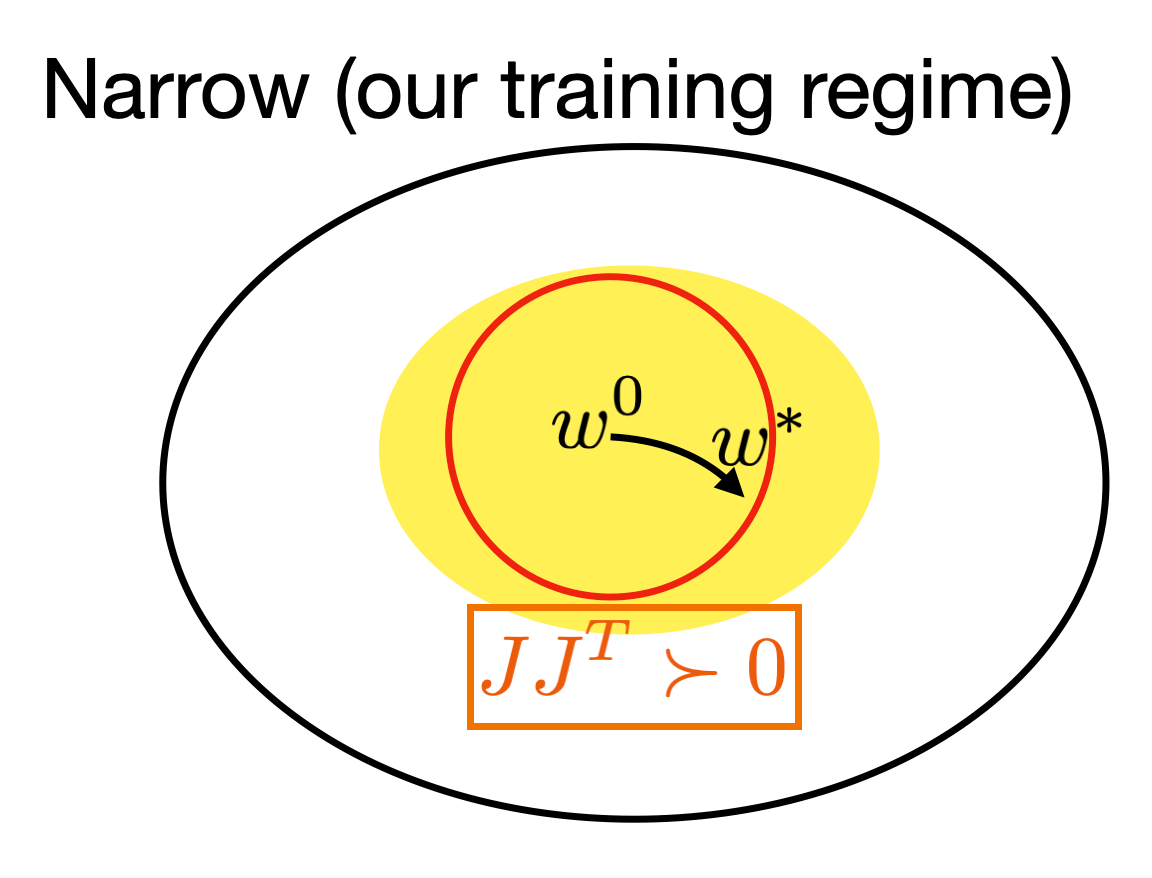}
          \end{minipage}%
          }%
        \centering
        \vspace{-2mm}
        \caption{ The parameter movement under different regimes. The shaded area indicates the region where $J(w;v)$ is non-singular, the black circle denotes the region that GD iterates will explore, and the red circle is the constraint designed in our training regime, it will be discussed in Section \ref{section:algorithm}.}
        \label{figure:idea}
        \vspace{-2mm}
      \end{figure}      

The second reason, and also an easily ignored one, 
is the expressivity (a.k.a. the representation power, see Definition \ref{def:expressivity} for a formal statement).
In above discussion, we implicitly assumed that
there exists a zero-loss global minimizer, which
is equivalent to ``there exists a network configuration such that the network can memorize the data''.
For networks with width at least $n,$ this assumption can be justified in the following way. 
The feature matrix 
{\small
\begin{align} \label{featurematrix}
    \Phi(w):=\left[\begin{array}{c} \sigma\left(w_{1}^{T} x_1\right),\dots,\sigma\left(w_{m}^{T} x_1\right)\\ \vdots \\  \sigma\left(w_{1}^{T} x_n\right),\dots,\sigma\left(w_{m}^{T} x_n\right) \end{array}\right]\in \mathbb{R}^{n \times m}
\end{align}
 }%
can span the whole space $\mathbb{R}^n$ 
when it is full rank and $ m \geq n$,
thus the network can perfectly fit any label $y\in \mathbb{R}^n$ even without training the hidden layer.
It is important to note that when the width $m$ is below the sample size $n$, full row-rankness
does not ensure that the row space of the feature
matrix is the whole space $\mathbb{R}^n$. 
In other words, it is not clear whether a global-min with zero loss exists.

In the next section, we will describe
how we obtain strong expressive power with $\Theta(n/d)$
 neurons, and how to avoid sub-optimal stationary points. 

\vspace{-2mm}
\section{Main Results}
\vspace{-2mm}
\subsection{Problem Settings and Preliminaries}
\vspace{-2mm} 
\label{section:setting}
We denote $\left\{\left(x_{i}, y_{i}\right)\right\}_{i=1}^{n} \subset \mathbb{R}^d \times \mathbb{R}$ as the  training samples, where $x_i \in \mathbb{R}^d$, $y_i \in \mathbb{R}$. 
For theoretical analysis, 
we focus on 1-hidden-layer neural networks  $f(x ; \theta)=\sum_{j=1}^{m} v_{j} \sigma\left(w_{j}^{T} x\right) \in \mathbb{R
}$, where $\sigma(\cdot)$ is the activation function, $w_{j} \in \mathbb{R}^{d}$ and $v_{j} \in \mathbb{R}$ are the parameters to be trained.  Note that we only consider the case where $f(x ; \theta) \in \mathbb{R}$ has 1-dimensional output for notation simplicity.

To learn such a neural network, we search for the optimal parameter $\theta=(w, v)$ by minimizing the empirical loss (\ref{unconstrained}),
and sometimes we also use $\ell(w;v)$ or $f(w;x,v)$ to emphasize the role of $w$. 
We use the following shorthanded notations: $x:=(x_1^T;\dots;x_n^T)\in \mathbb{R}^{n\times d}$, $y:=(y_1,\dots,y_n)^T \in \mathbb{R}^{n}$, $w:=(w_{1},\dots, w_m)_{j=1}^m \in \mathbb{R}^{d\times m}$, $v:=(v_{1}, \dots, v_m)_{j=1}^m \in \mathbb{R}^{m}$,  and $f(w;v):=(f(x_1;w,v),f(x_2;w,v),\dots, f(x_n;w,v))^T \in \mathbb{R}^{n}$. We denote the Jacobian matrix of $f(w;v)$ w.r.t $w$ as $J(w;v)$, \BLACK{ which can be seen in \eqref{jacobian}.}
We define the feature matrix \BLACK{$\Phi(w)$ as in \eqref{featurematrix}. We denote the operator $\nabla_w$ as ``taking the gradient w.r.t. $w$'',
and the same goes for $\nabla_v$.}
Throughout the paper, `w.p.1' is the abbreviation for `with probability one'; %
when we say `in the neighborhood of initialization', it means `$w$ is in the neighborhood of the initialization $w^0$'. 

\BLACK{Now, we formally define the term ``expressivity''.  As discussed in Section \ref{section:relatedwork}, we focus on the finite-sample (as opposed to infinite-sample) expressivity, which is relevant in practical training.

\begin{myDef} \label{def:expressivity}
(Expressivity) We say a  neural net function class $\mathcal{F}=\{f(x;\theta); \theta\in \Theta\}$ has strong ($n$-sample) expressivity if for any $n$ input-output pairs $D=\left\{\left(x_{i}, y_{i}\right)\right\}_{i=1}^{n} \subset \mathbb{R}^d \times \mathbb{R}$ where $x_i$'s are distinct,
there exists a $\hat{\theta} (D) \in \Theta$ such that $f(x_i;\hat{\theta}(D))=y_i$, $i=1, \cdots, n$. Or equivalently, the optimal value of empirical loss \eqref{unconstrained} equals 0 for any $D$. Sometimes we may drop the word ``strong'' for brevity.
\end{myDef}
}
\begin{wrapfigure}{r}{0.4\textwidth}
 \vspace{-3mm}
  \centering 
  \includegraphics[width=0.8in]{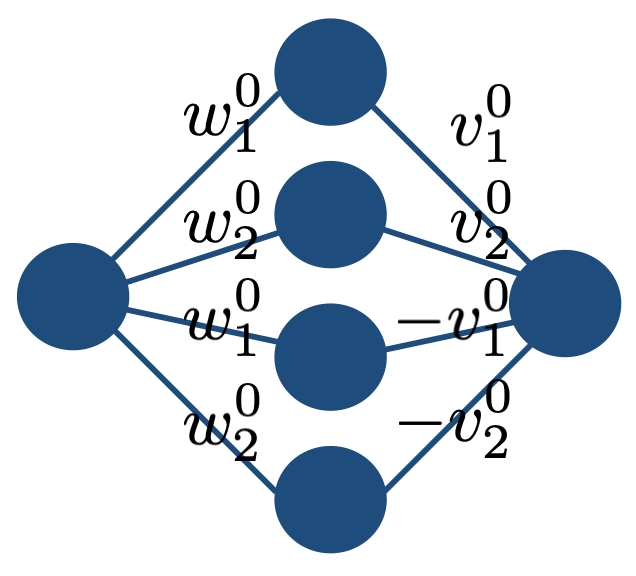}
  \vspace{-3mm}
  \caption{ A simple example of the mirrored LeCun's initialization.}\label{initialexample}
   \vspace{-3mm}
\end{wrapfigure}
Next, let us describe the {\it mirrored LeCun's} initialization in Algorithm \ref{initial}.
The idea is that through this initialization, the hidden outputs will cancel out with the outer weights, so that we get zero initial output for any input $x$; see Figure \ref{initialexample} for a simple illustration. 
Note that 
similar symmetric initialization strategies are also proposed in some recent works such as \cite{chizat2018lazy}  and \cite{daniely2019neural}. However, our purpose is different. More explanation can be seen in the final paragraph
of Section \ref{section:express}.

\begin{algorithm}
  \caption{The mirrored LeCun's initialization}
  \begin{algorithmic}[1] \label{initial}
    \STATE Initialize all the weights using LeCun's initialization: $w_{i,j}^0\sim N(0,\frac{1}{d})$,  $v_i^0\sim N(0,\frac{1}{m})$, for $i=1,\dots,m/2$, $j=1,\dots,d$.
    \STATE Set $(w^0_{\frac{m}{2}+1},\dots,w^0_{m})\leftarrow(w^0_{1},\dots,w^0_{\frac{m}{2}})$, and set $(v^0_{\frac{m}{2}+1},\dots,v^0_{m})\leftarrow(-v^0_{1},\dots,-v^0_{\frac{m}{2}})$
  \end{algorithmic}
  \end{algorithm}
Throughout the paper, we will make the following assumptions. %

  \begin{myassum}\label{assum1}
  For $f(x ; \theta) $ in (\ref{eq1:simplenn}), we assume its width $m$ is an even number, and $m\geq  \frac{2n}{d}$. 
  \end{myassum}
  \begin{myassum}\label{assum2}
    We assume the activation function $\sigma(\cdot): \mathbb{R} \rightarrow \mathbb{R}$ is analytic and L-lipschitz continuous, its zero set only contains 0: $\{z| \sigma(z)=0 \}=\{0\}$. \BLACK{In addition, there are infinitely many non-zero coefficients in the Taylor expansion of $\sigma (\cdot)$.}
  \end{myassum}
  \begin{myassum}\label{assum3}
  $ x_1,  \cdots, x_n $ are independently sampled from a continuous distribution in  $ \mathbb{R}^{d}$. 
  \end{myassum}
    
  When $d>2$, Assumption \ref{assum1} can be applied to narrow networks \footnote{When $d=1,2$, all our results still hold; nevertheless, the required width $m \geq n, 2 n$, thus it does not belong to the ``narrow'' setting we defined earlier in Section \ref{section:relatedwork} (which requires $ m \in [n/d, n ) $).} with $m<n$.  Assumption \ref{assum2} covers many commonly used activation functions such as sigmoid, softplus, and Tanh, but it does not cover ReLU since it is nonsmooth.

  \vspace{-2mm}
  \subsection{Expressivity Analysis}
  \label{section:express}
  \BLACK{In this section, we prove that narrow neural networks (which are still over-parameterized) have strong expressivity.
  Further, the zero-training-error solution is surrounded by a good landscape with no local-min or saddles, which motivates our trainability analysis in the following sections.}
  \begin{myTheo}\label{thm1}
    Suppose Assumption \ref{assum1}, \ref{assum2}, and \ref{assum3} holds. If the neural network $f(x;\theta)=\sum_{j=1}^{m} v_{j} \sigma\left(w_{j}^{T} x\right)$ is initialized at the mirrored LeCun's initialization given in Algorithm \ref{initial}, with $\theta^0=(w^0,v^0)$, then 
    there exists $\epsilon_0>0$ such that for any $\epsilon \leq \epsilon_0$, 
    there exists a $w \in B_\epsilon(w^0) =\{w \mid \|w-w^0\|_F \leq\epsilon\}
    $ and a entry-wise non-zero $v$, 
    such that with probabilty 1 of choosing $\{(x_i,y_i)\}_{i=1}^n$ and $\theta_0$, the output of $f$ will be exactly the groundtruth label:
    \begin{equation}
      f(x_i;\theta)=\sum_{j=1}^m v_j\sigma(w_j^Tx_i)=y_i, i=1,\cdots,n.
    \end{equation}
    In addition, every stationary point 
    $\theta^*=(w^*,v^*)$ (i.e., the gradient is zero) is a global-min of \eqref{unconstrained} with zero loss if it satisfies 
     $w^* \in B_\epsilon(w^0)$ and $v^*$ is entry-wise non-zero.
  \end{myTheo}

 \paragraph{Remark 1.}
    Theorem \ref{thm1} emphasizes the role of hidden weights of a neural network: it is a key ingredient for strong expressivity. When $m <n$, if we fix all the $w=w^0$, the range space of the feature matrix $\Phi(w^0)$ 
    does not cover the whole $\mathbb{R}^n$ space, so there always exists a label $y$, such that no $v^*$ can be found that perfectly maps the input to $y$. However, a small tolerance of the movement of $w$ will let $f(x;\theta)$ perfectly fit any input-label pair, so the movement of $w$ is vitally important. 
    The free perturbation of $w$ serves as an effective remedy against the limited expressivity.

\BLACK{
\paragraph{Remark 2.}
  We emphasize that Theorem \ref{thm1} holds for ``any small enough $\epsilon$'' instead of ``any $\epsilon$''. Therefore, Theorem \ref{thm1} only states ``there is no spurious local-min'' {\it locally}.
  It is still possible that on the {\it global} landscape results there ``exists bad local-min" (e.g. Ding et al. \cite{ding2019sub}). 
  
  We comment a bit more on the maximum required size of $\epsilon$. In our proof in Appendix \ref{appendix:thm1}, it  should not exceed the the radius of the region where the Jacobian $J(w;v)$ stays full-rank (the yellow-shaded area in Figure \ref{figure:idea}). To briefly summarize, the maximum radius is (linearly) proportional to the minimum singular value of the initial Jacobian $J(w^0;v^0)$.  Technical details on the size of this radius can be seen in \cite[Remark 4.1]{IFT}. 
}
  
  \begin{proof}[Proof sketch]
    Theorem \ref{thm1} consists of two arguments: (i) there exists a global-min with zero loss, (ii) in the neighborhood of initialization, every stationary point is a global-min. A detailed proof is relegated to Appendix \ref{appendix:thm1thm2}. We outline the main idea below.

    To argue (i), 
      the key idea is to use the Inverse Function Theorem (IFT), which is stated in
      Appendix \ref{appendix:thm1}.
      According to IFT, as long as an $n\times n$ submatrix of $J(w^0;v^0)$ is invertible, then for any $y \in \mathbb{R}^n$ \BLACK{and any small enough $\epsilon$}, there
   exists a $w^* \in B_\epsilon(w^0)$ whose prediction output $f(w^*;v^0) \propto y-f(w^0;v^0)$. Additionally, since $f(w^0;v^0)=0$, we have $f(w^*;v^0) \propto y$. Once this is shown, then we just need to scale all the outer weight $v_j$ uniformly and the output will be exactly $y$ since $f(w^*;v)$ is linear in $v$.

To argue (ii), recall that all the stationary points $\theta^*=(w^*,v^*)$ satisfy $\nabla_w\ell(\theta^*)=J(w^*;v^*)^T(f(w^*;v^*)-y)=0.$   
Therefore, if $J^T(w^*;v^*) \in \mathbb{R}^{md \times n}$ is of full column rank, the stationary point  $\theta^*=(w^*,v^*)$ is a global minimizer with $\ell(\theta^*)=0$. The desired full-rankness condition is true because: (i) $J(w^0;v^0)$ is full rank w.p.1 at initialization, (ii) $w^*$ will not leave the small neighborhood $ B_\epsilon(w^0)$, so the dynamics of $w$ stays inside the manifold of full-rank Jacobian. 

  \end{proof}
    \vspace{-1mm}
    The proof of Theorem \ref{thm1} relies on the full-rankness of the Jacobian matrix. We note that such  full-rankness holds for {\it both} the mirrored and the regular LeCun's initialization (the case for the regular one can be proved using the same technique). %
    So why do we insist on shifting the initial output to 0? 
    Simply put, the ``randomness of weights'' contributes to the ``full-rankness'', while the `shifting' allows us to gain \BLACK{local} representation power by applying IFT properly.
    To be more specific, 
      $f(w^0,v^0)=0 \in \mathbb{R}^{n}$ is important because it is surrounded by all possible directions pointed from $0 \in \mathbb{R}^{n}$, so for any $y\in \mathbb{R}^{n}$, IFT claims that there exists at least one  $w^* \in B_\epsilon(w^0)$, s.t. $f(w^*;v^0) \propto y-f(w^0;v^0) =y$, therefore, $f(w^*;v^*)$ can perfectly match $y$ by scaling $v^0$ with some constant (Figure \ref{shiftvsregular}(a) illustrates this case when $n=2$).  
       
       \BLACK{In contrast, if we use IFT around {\it regular} LeCun's initialization, the existence of $f(w^0;v^0)$ on the right hand side resists us from scaling $v^0$ like before (Figure \ref{shiftvsregular}(b) illustrates this case). As such, Theorem \ref{thm1} does not hold for {\it any} small $\epsilon$ around the regular initialization.
      This is also revealed in our experiments in Section \ref{section:synthetic}: training fails if we only search around the regular LeCun's initialization.}
      
      \begin{figure}[htbp]
      \vspace{-2mm}
        \centering
        \subfigure[The mirrored LeCun's initialization]{
        \begin{minipage}[t]{0.5\linewidth}
        \centering
        \includegraphics[width=1.3in]{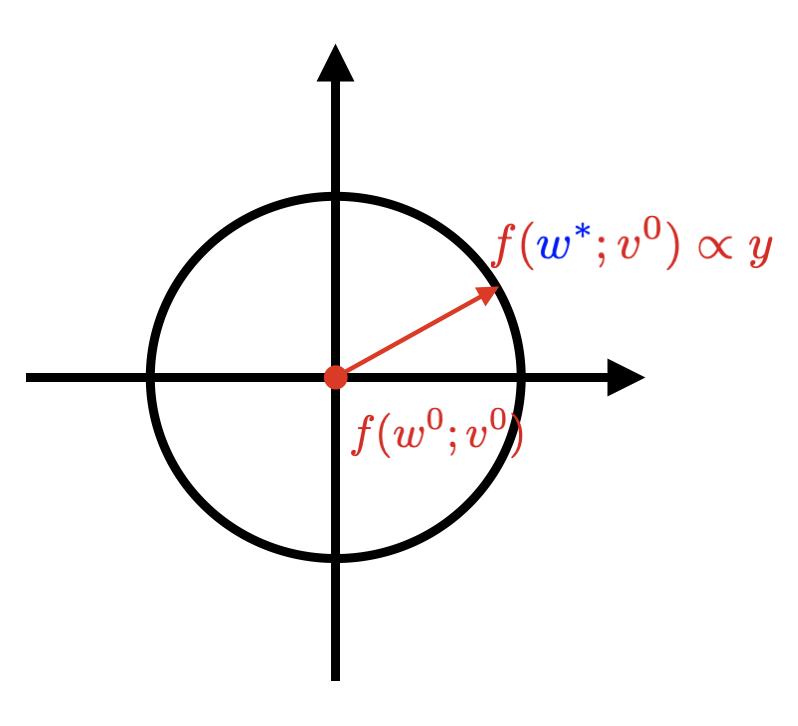}
        \end{minipage}%
        }%
        \subfigure[Regular LeCun's initialization]{
          \begin{minipage}[t]{0.5\linewidth}
          \centering
          \includegraphics[width=1.3in]{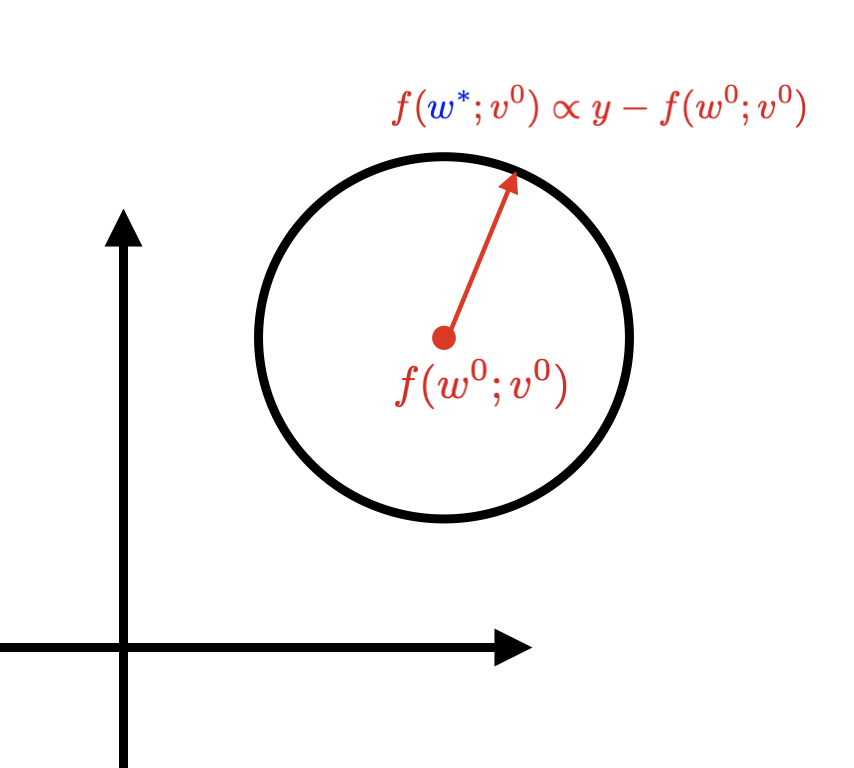}
          \end{minipage}%
          }%
        \centering
        \vspace{-3mm}
        \caption{ Examples of using Inverse Function Theorem (IFT) under different initialization strategies. The illustration here is for $n=2$. For (a), scaling $v^0$ will directly lead to zero loss, while it is not true for (b) due to the non-zero $f(w^0;v^0)$. }
        \vspace{-2mm}
        \label{shiftvsregular}
      \end{figure}
     The idea of zero initial output is also used in other recent works in Table \ref{table:initial}.  
    Despite the similar design, they use such an initialization for different purposes.   In NTK regime, zero initial output helps eliminate the bias term and simplify the proof (\cite{lee2019wide,hu2020surprising,chizat2018lazy,hu2019simple,bai2019beyond} and \cite{daniely2019neural}). 
     Nguyen \cite{nguyen2021proof} also uses zero initial output, but their initialization is very different in that all the hidden layers will have high values while the last layer is assigned to 0.
     In this way, they manage to limit the movement of hidden layers without increasing the width. \BLACK{To our knowledge, this is the first time that the zero initial output has been linked to Inverse Function Theorem, by which the strong expressivity of a narrow neural network can be identified. }

\vspace{-1mm}

\begin{table}[ht]
  \caption{Comparison of recent works considering zero initial output.}
  \vspace{-2mm}
  \label{table:initial} 
 \centering %
 \resizebox{\textwidth}{!}{%
 \begin{tabular}{lll} 
 \toprule %
   Work & Width &  Motivation  \\ \midrule %
  \cite{lee2019wide,hu2020surprising,chizat2018lazy,hu2019simple, bai2019beyond}    & $m_{L-1}\rightarrow \infty$   &  To avoid handling the bias term in the NTK regime    \\
 \cite{daniely2019neural}   & $m=\tilde{O}(n/d)$   &  To avoid handling the bias term in the NTK regime    \\
  \cite{nguyen2021proof}    & $m_{L-1}=O(n)$   &  To ensure linear convergence via imbalanced weight   \\
  {\bf Ours}  & $m=O \left(\frac{n}{d}\right)$   &  To 
  achieve strong expressivity via  Inverse Function Theorem\\
 \bottomrule
 \vspace{-2mm}
\end{tabular}
}
\end{table}

\vspace{-2mm}
\subsection{Trainability Analysis}
\label{section:algorithm}
\vspace{-2mm}

Despite the expressivity and good local properties stated in Theorem \ref{thm1}, in practical training, the weights can easily escape the nice neighborhood, especially when the width is not sufficiently large. To keep the hidden weights inside this nice region, an intuitive idea is to impose an explicit constraint, but it may suffer from bad local-min on the boundary with a very large loss \cite{bertsekas1997nonlinear}. This is supported by our experiments in Section \ref{section:experiment}, Figure \ref{synthetic:loss}, (b): when we add the hidden-weight constraint directly to the regular training regime (i.e., let $\|w-w^0\|_F \leq \epsilon$), it fails to find a low-cost solution when $\epsilon$ and width become small.

In this sense, we need to design a constrained problem, such that all the KKT points will have small training loss, including those on the boundary. Fortunately, Theorem \ref{thm1} suggests one such formulation.  Recall in the proof of Theorem \ref{thm1}, we construct a zero-loss global-min by scaling $v^0$
 , so $v^*$ still follows the pairwise-opposite pattern. Inspired by this, 
we consider the following neural network
(we abuse the notation of $f(x;\theta)$, $f(x;w,v)$ and $v=(v_{1},\dots,v_{\frac{m}{2}})$ here):
\begin{equation}\label{pairNN}
 f(x;w,v)=\sum_{j=1}^{\frac{m}{2}} v_j\left(\sigma(w_j^Tx)-\sigma(w_{j+\frac{m}{2}}^Tx) \right).
\end{equation}
Note that the optimization variable for the outer layer is only $v=(v_{1},\dots,v_{\frac{m}{2}})$, and the rest of the outer weights are automatically set to be $-v$. 
Despite the change of $v$, %
Theorem \ref{thm1} still applies since it does not have any specific requirement on $v$.
We then use (\ref{pairNN}) to formulate the following problem (\ref{constrained}):
\begin{equation}
  \begin{aligned}
    \underset{\theta}{\min} \ \ell(\theta)=&\frac{1}{2}\sum_{i=1}^{n} (y_{i}-f(x_i;\theta))^{2},  \quad 
    \text{s.t.} \quad&  w \in B_\epsilon(w^0), \ v \in B_{\zeta, \kappa}(v) \\
  \end{aligned}
  \label{constrained}
\end{equation}
where $f(x_i;\theta)$ is in the form of \eqref{pairNN},
\begin{align*}
    B_\epsilon(w^0)& :=\{w\mid \|w-w^0\|_F \leq\epsilon\},\\
    B_{\zeta, \kappa}(v)& :=\{v\mid  \; v \geq \zeta {\bf 1}~\text{and}~  v_j/v_{j^{\prime}} \leq \kappa, \forall~(j, j')\in\{1,\cdots, m\}, \text{where}~ \zeta>0, \  \kappa < \infty. \}.
\end{align*}
 Here, $\zeta >0$ is a small constant that keeps
 the entries of $v$ away from zero, which is an essential requirement of Theorem \ref{thm1}.  The requirement of $v_j/v_{j^{\prime}} \leq \kappa < \infty$ allows all entries of $v$ to be uniformly large, but it rules out the case when some entries are much larger than others. 
Instead of regarding all these requirements of $B(v)$ as prior assumptions, we formulate them into the constraints in the problem, so all the iterates in the practical training algorithm will strictly follow these requirements.  
 In Theorem \ref{thm3}, we show that every KKT point of problem (\ref{constrained}) implies the near-global optimality for the unconstrained training problem \eqref{unconstrained}.

\begin{myTheo}\label{thm3}
Suppose Assumption \ref{assum1}, \ref{assum2}, and \ref{assum3} hold and assume $\theta^0=(w^0,v^0)$
  as given in Algorithm \ref{initial}.
  Then every KKT point $\theta^*=(w^*,v^*)$  of (\ref{constrained}) is an approximate global-min w.p.1., that is:   
  \begin{equation}\label{approxglobalmin}
    \ell(w^*,v^*) = O(\epsilon^2).
 \end{equation} 
\end{myTheo}

\BLACK{The proof of Theorem \ref{thm3} is based on  the special structure of neural network $f(x ; \theta)$, including the linear dependence of $v$ and the mirrored pattern of parameters. To better illustrate our proof idea, we provide a user-friendly proof sketch in Appendix \ref{appendix:thm3sketch}. Detailed proof can be seen in Appendix \ref{appendix:thm3detail}.

Theorem \ref{thm3} motivates a training method to reach small loss. We highlight three new ingredients that is not used in regular neural net training: the mirrored initialization, the pairwise structure of $v$ in \eqref{constrained}, and the constrained parameter movement. Combining these elements with Projected Gradient Descent (PGD), we propose a new training regime in Algorithm \ref{algo:ourmethod} in Appendix \ref{appendix:pytorch}. Thorough numerical results are provided in the following sections to demonstrate the efficacy of Algorithm \ref{algo:ourmethod}.  }

\vspace{-3mm}
\subsection{Discussion: Extension to Deep Networks}
\vspace{-2mm}
\label{section:limitation}
In the previous sections, we analyze the trainability and expressivity of narrow 1-hidden-layer networks. We find it possible to extend the previous analysis to deep nets, and we already have some preliminary results. Due to space constraints, more relevant discussions are deferred to Appendix \ref{appendix:lemmadeep}.

\vspace{-2mm}
\section{Experiments}
\label{section:experiment}
\vspace{-2mm}
In this section, we provide empirical validation for our theory. Specifically, we compare the performance of two training regimes \footnote{We call it ``training regime''
instead of ``training method'' since we use a different formulation as well a different algorithm
compared to standard SGD.}: 

\noindent{\bf (1)} {\bf Our training regime:} we optimize a constrained problem \eqref{constrained} 
by using PGD (projected gradient descent), starting from
the mirrored LeCun's initialization. 
(See Algorithm \ref{algo:ourmethod} in Appendix \ref{appendix:pytorch}.)

\noindent{\bf (2)} {\bf Regular training regime:} optimize an unconstrained problem \eqref{unconstrained}
by using GD-based methods, starting from LeCun's initialization.

\BLACK{
\textbf{Main ingredients of our algorithm.}
As shown in Theorem \ref{thm3}, all KKT points in our training regime have small empirical loss. To reach such KKT points, we use Projected Gradient Descent (PGD) (see Bertsekas \cite{bertsekas1997nonlinear}). Even though we do not provide convergence analysis,  PGD can, empirically,  converge to a KKT point with proper choice of stepsize. We outline the proposed training regime in Algorithm \ref{algo:ourmethod} in Appendix \ref{appendix:pytorch}.
To briefly summarize, there are three key ingredients in Algorithm \ref{algo:ourmethod}: the mirrored initialization; the pairwise structure of $v$ in \eqref{constrained}; and the PGD algorithm.  Each of these changes only involves a few lines of code changes based on the regular training.  We demonstrate the \texttt{PyTorch} implementation of these changes in Appendix \ref{appendix:pytorch}. 

\textbf{Better training and test error.}
To evaluate our theory in terms of training error, we conduct experiments on synthetic dataset (shown in Section \ref{section:synthetic}) and random-labeled CIFAR-10 \cite{krizhevsky2009learning} (shown in appendix \ref{appendix:randomlabel}). We further observe the strong generalization power of Algorithm \ref{algo:ourmethod}, even though it is not yet revealed in our theory. Our training regime brings higher or competitive test accuracy on (Restricted) ImageNet \cite{russakovsky2015imagenet} (shown in Section \ref{section:rimagenet}), MNIST \cite{lecun2010mnist}, CIFAR-10, CIFAR-100 \cite{krizhevsky2009learning} (shown in Appendix \ref{appendix:experiment}). Detailed experimental setup are explained in Appendix \ref{appendix:experimentsetup}.

As a side note, for all the experiments in our training regime, we observe that $v$ never touches the boundary of $B_{\zeta, \kappa}(v)$ when $\kappa=1, \zeta=0.001$. So we can regard problem \eqref{constrained} as an unconstrained problem for $v$, and PGD only projects the hidden weights $w$ into $B_\epsilon(w^0)$. When $\epsilon=1000$, problem \eqref{constrained} degenerates into an unconstrained problem (but still different from the regular training due to the changes in the structure of $v$ and initialization).

}

\vspace{-2mm}
\subsection{Training Error on The Synthetic Dataset}
\vspace{-2mm}
\label{section:synthetic}

\BLACK{
On the synthetic regression dataset, we train 1-hidden-layer networks under different widths and different training settings, the final training errors are shown in Figure \ref{synthetic:loss}. We explain as follows.

\begin{figure}[h]
\vspace{-2mm}
  \subfigure[Our training regime]{
    \begin{minipage}[t]{0.5\linewidth}
    \centering
    \includegraphics[width=2in]{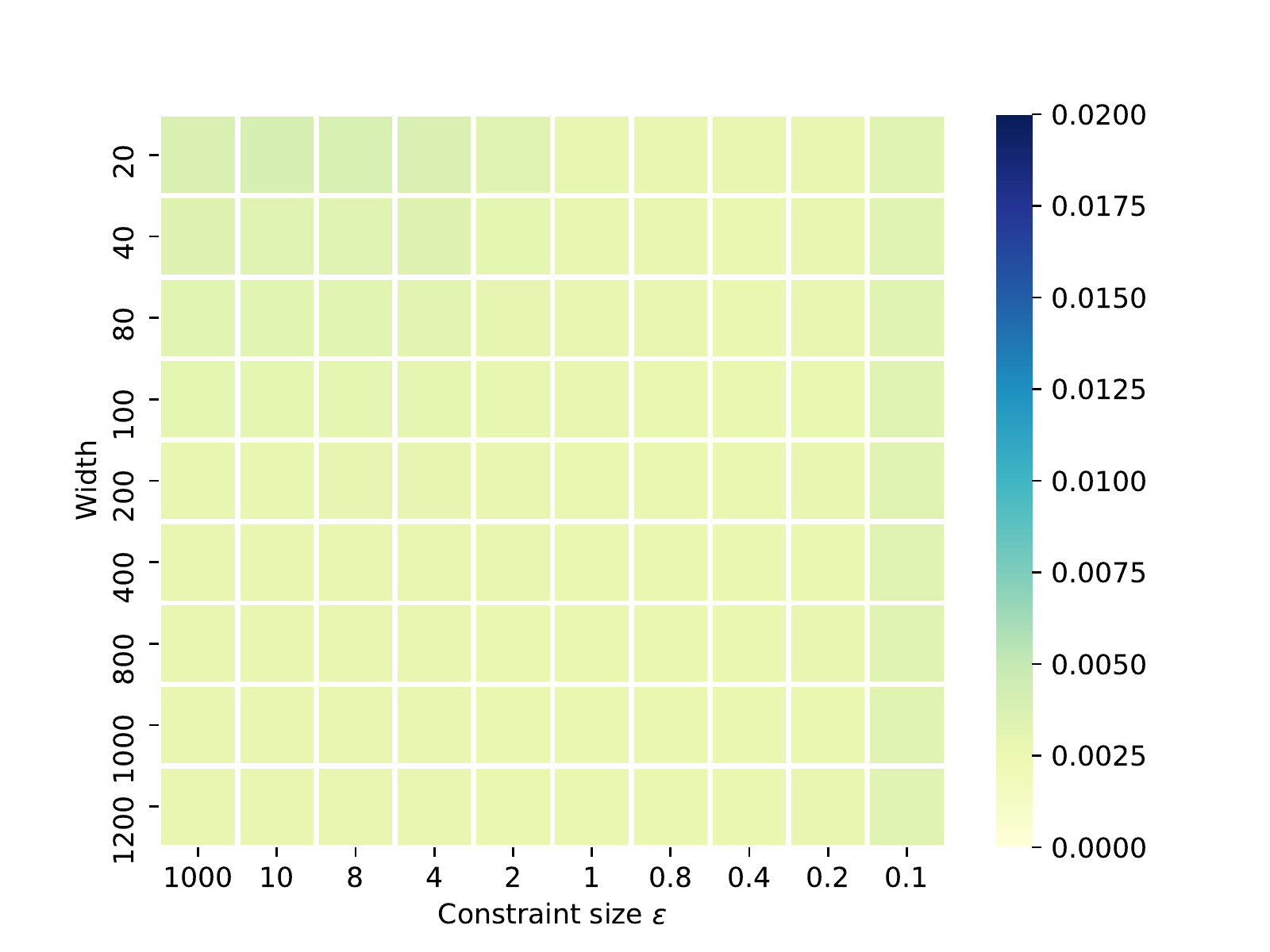}
    \end{minipage}%
    }%
    \subfigure[Regular training]{
      \begin{minipage}[t]{0.5\linewidth}
      \centering
      \includegraphics[width=2in]{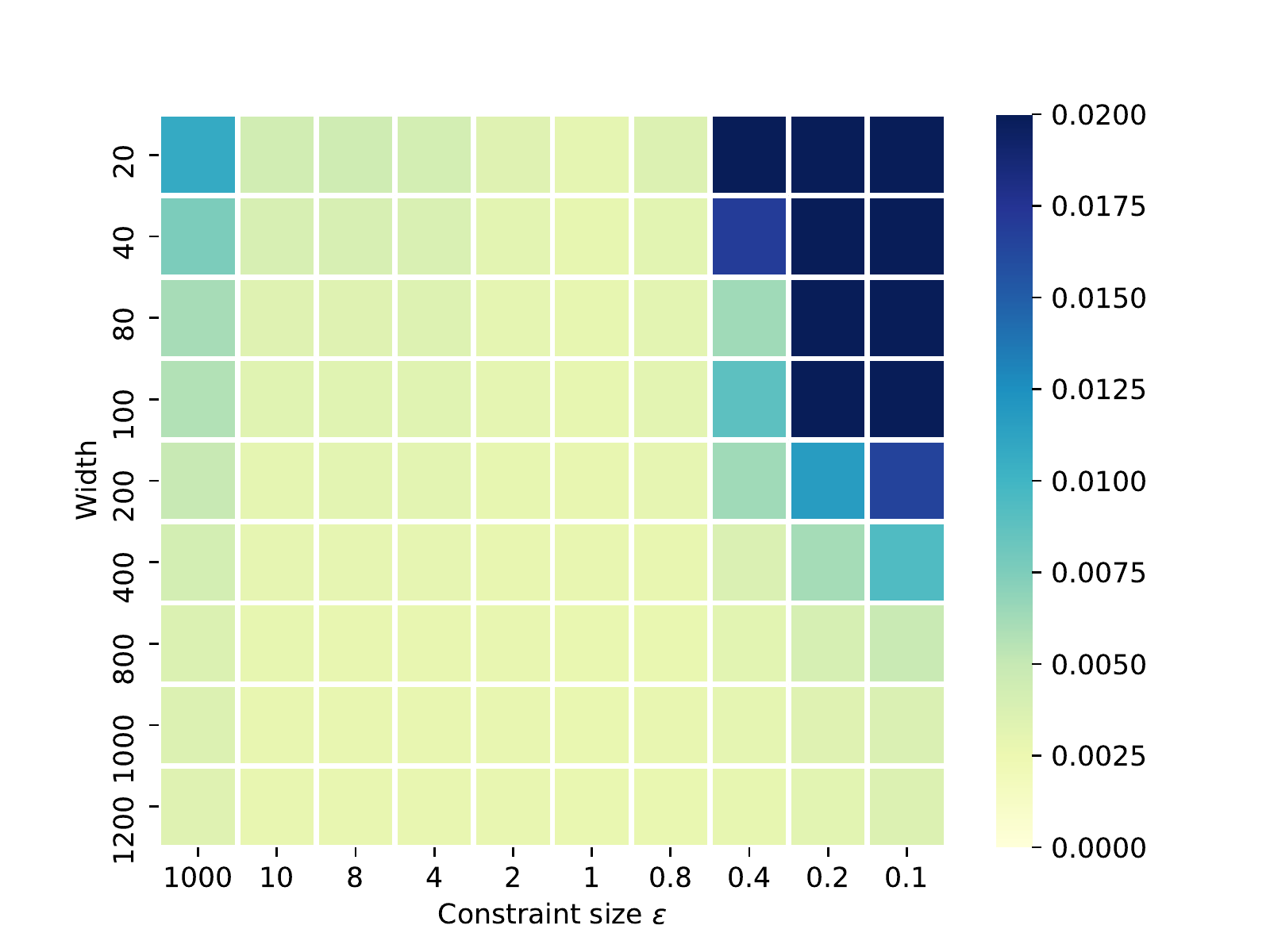}
      \end{minipage}%
      }\\%
      \subfigure[$\lambda_{\text{min}}(J(w^*;v^*))$\\ in our training regime]{
      \begin{minipage}[t]{0.5\linewidth}
        \vspace{-3mm}
      \centering
      \includegraphics[width=2in]{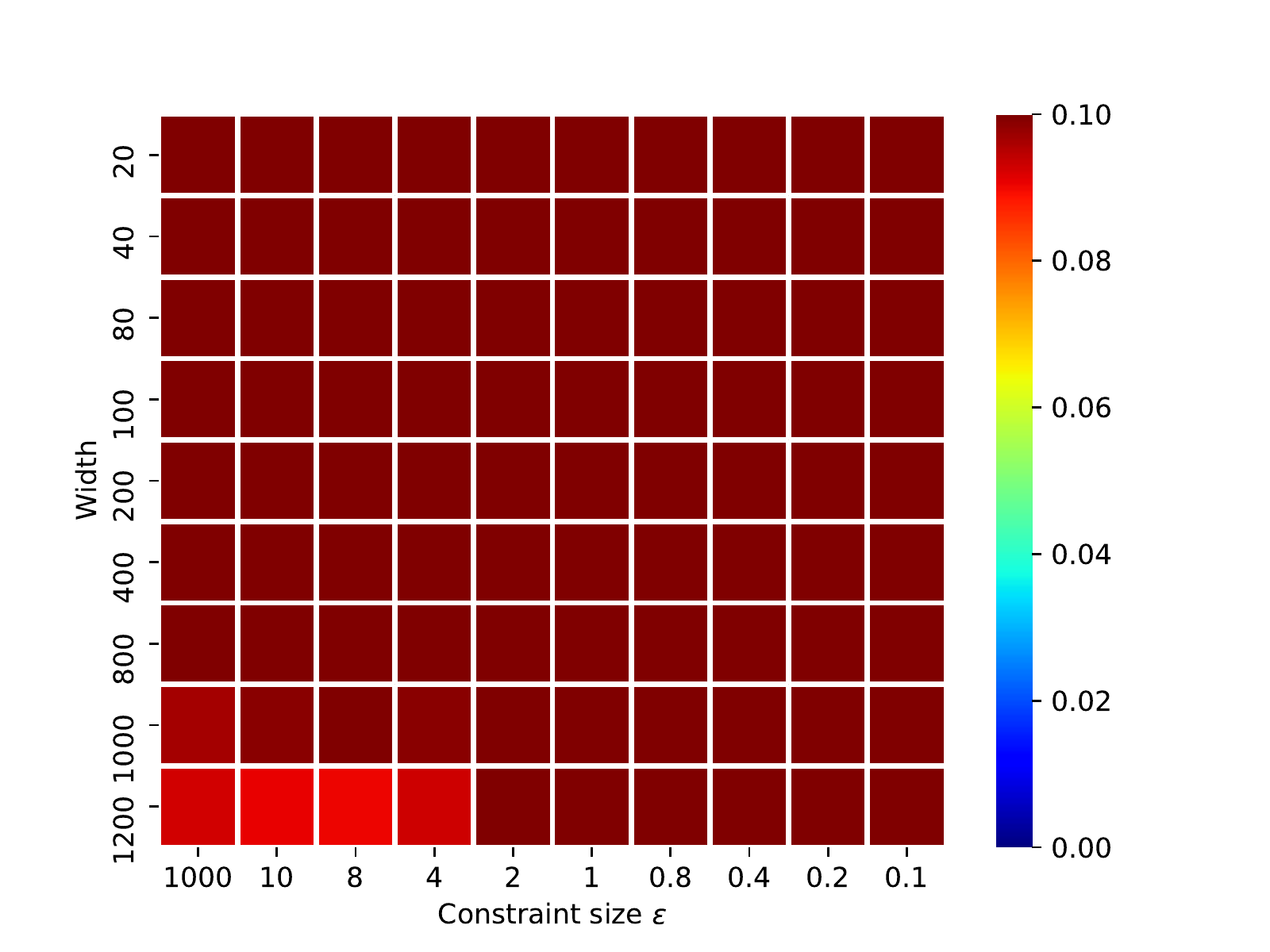}
      \end{minipage}%
      }%
      \subfigure[$\lambda_{\text{min}}(J(w^*;v^*))$ \\in regular training regime]{
        \begin{minipage}[t]{0.5\linewidth}
        \vspace{-3mm}
        \centering
        \includegraphics[width=2in]{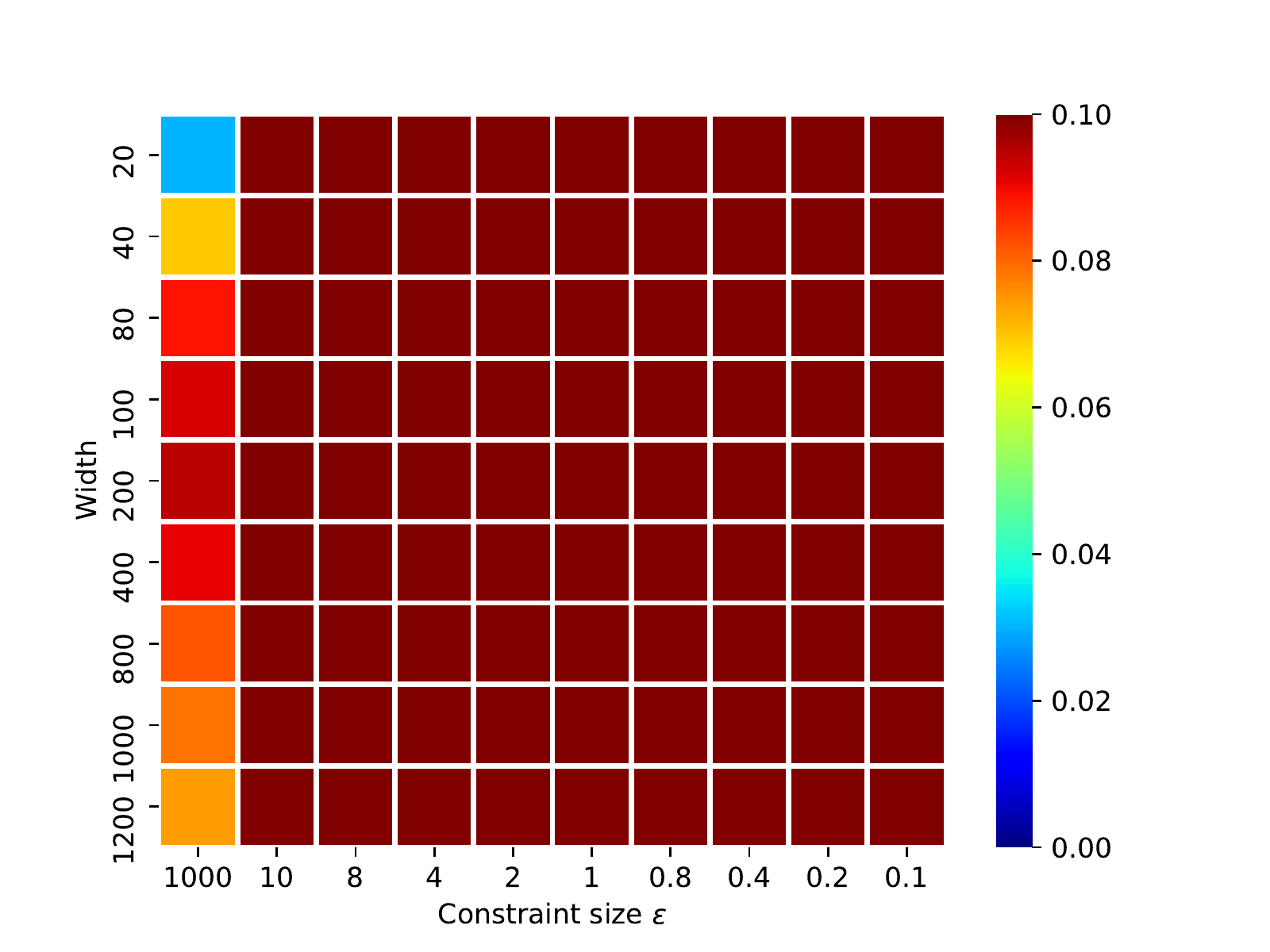}
        \end{minipage}%
        }%
  \vspace{-3mm}
  \caption{Synthetic data: training losses and $\lambda_{\text{min}}(J(w^*;v^*))$ in different width \& hidden-weight constraint size $\epsilon$. }
  \label{synthetic:loss}
  \vspace{-1mm}
\end{figure}

As for our training regime, we try different hidden-weight constraint size $\epsilon$, these results can be seen in Figure \ref{synthetic:loss}, (a). Accordingly, the performance of the regular training regime can be seen in the 1st column in Figure \ref{synthetic:loss}, (b). As argued above, $\epsilon=1000$ degenerates PGD into unconstrained GD, so this column shows the results for regular training regime.
As for the rest of the columns in the Figure \ref{synthetic:loss}, (b), we try to investigate an ablation study: ``What will happen if we directly add constraint $B_\epsilon (w^0)$ on $w$, and use PGD without any modification of the initialization \& network structure?'' In each block in Figure \ref{synthetic:loss}, a grid search of step-size is performed to ensure the convergence of algorithms.
The key messages from this set of experiments are summarized as follows.

First, our training regime performs well regardless of width, yet regular unconstrained training fails when the network is narrow (1st column in Figure \ref{synthetic:loss} (b)). 

Second, it does not work when we directly impose the hidden-weight constraint of \eqref{constrained} on
the regular training regime (and PGD is used accordingly), as it fails to find a low-cost solution when $\epsilon$ and width become small. There are two possible causes: perhaps there is no global-min inside the ball, or it converges to bad local-min on the boundary. 
In contrast, our training regime always finds a low-cost solution with any choice of $\epsilon$.
Therefore, we suggest  {\it not} to directly add constraint and use PGD. Instead,  when using PGD, it is better to utilize the mirrored initialization \&  the pairwise structure of $v$ in \eqref{constrained} (see Figure \ref{synthetic:loss} (a)).
}

Furthermore, Figure \ref{synthetic:loss}, (c) \& (d) depict $\lambda_{\text{min}}(J(w^*;v^*))$, i.e. the minimum singular value of the Jacobian at the stationary (or KKT) points. As expected in Section \ref{section:intro}, when the width is small, regular training regime (when $\epsilon=1,000$) has trouble controlling the parameter movement, and it is likely to get trapped at a stationary point with a singular Jacobian matrix, leading to a large loss despite the convergence. However, this is not an issue in our training regime.

 \vspace{-2mm}
\subsection{Test Accuracy on R-ImageNet} \label{section:rimagenet}
\vspace{-2mm}
\begin{wrapfigure}{r}{0.5\textwidth}
\vspace{-25pt}
 \includegraphics[width=0.5\textwidth]{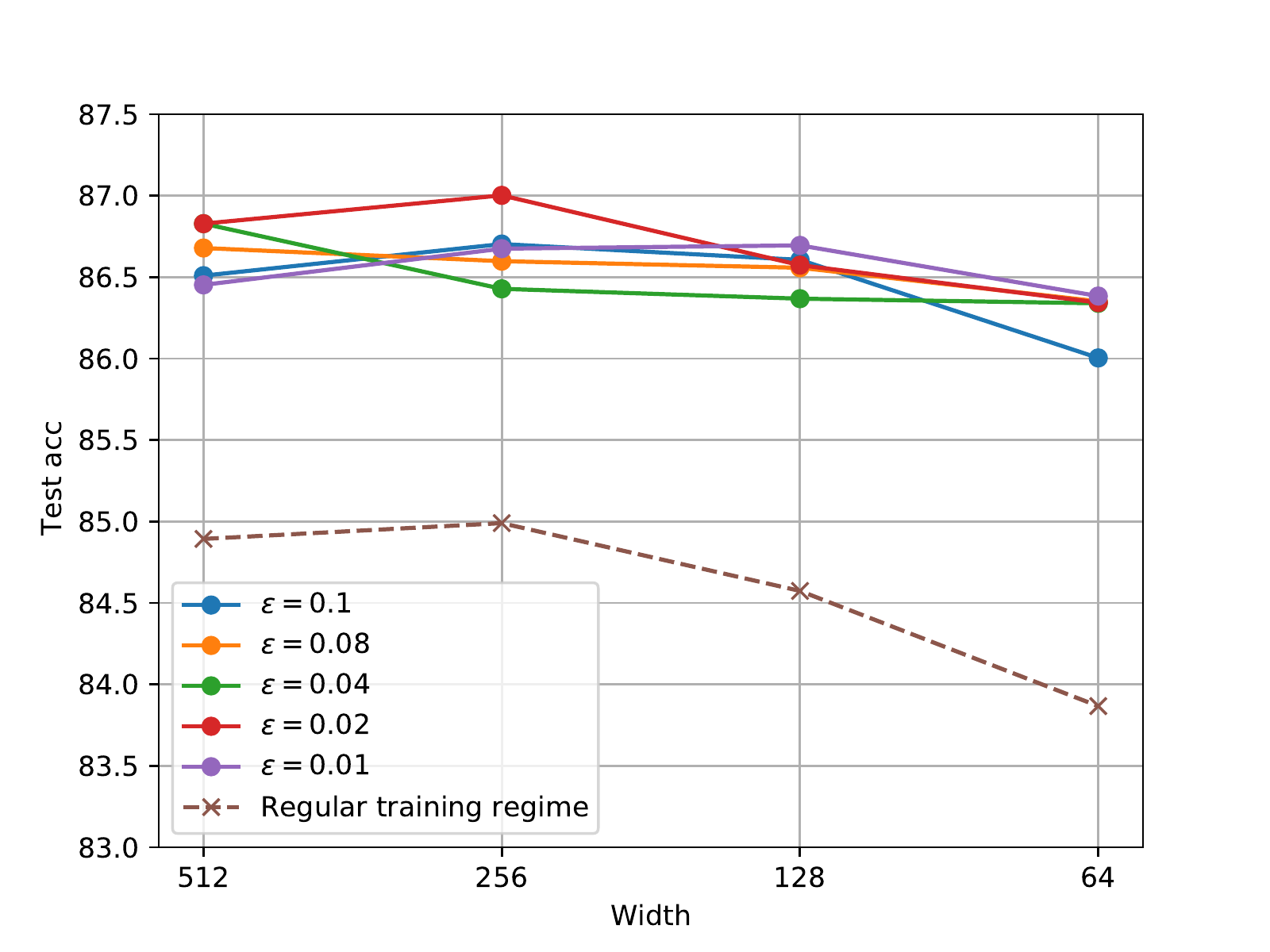} 
\vspace{-15pt} 
 \caption{R-ImageNet: test accuracy under our training regime with different $\epsilon$ vs regular training regime. \BLACK{In x-axis, width stands for the number of channels in the final CNN block of ResNet-18. These results are averaged over 5 seeds. }}
 \label{rimagenet:testacc} 
\end{wrapfigure}
We check the test accuracy (not just the training accuracy) of the proposed method 
on R(Restricted)-ImageNet (\cite{russakovsky2015imagenet}).
R-ImageNet is a subset of ImageNet with resolution 224$\times$224, and it is widely used in  various papers (e.g. \cite{tsipras2018robustness}, \cite{inkawhich2020perturbing}).
Detailed description can be found in
Appendix \ref{appendix:experimentsetup}.
In both training regimes, experiments are conducted on ResNet-18 \cite{he2016deep}, which contains 4 CNN blocks and 1 fully connected output layer. \BLACK{To compare the performance under different widths, we shrink the number of channels in the final CNN block gradually from 512 to 64 (note that we will call ``the number of channels'' as the ``width'' in CNN, this is a straightforward extension of  the width in FCN). 
In our training regime, we apply all the constraints in problem \eqref{constrained} to the final CNN block \& the output layer.}

\BLACK{There are two messages shown in Figure \ref{rimagenet:testacc}. First, our training regime outperforms regular training when using the standard ResNet-18 (i.e., width 512 in the final CNN block). Second, our training regime in the most narrow setting (width 64) performs quite close to the standard case
(width 512).
In comparison, regular SGD does not perform well in the narrow setting.
More theoretical analysis on the generalization power will be considered as our future work.}

\vspace{-2mm}
\section{Conclusion}
\vspace{-2mm}

  In this work, we shed new light on both the expressivity and trainability of narrow networks. Despite the limited number of neurons, we prove that the network can memorize $n$ samples, and it can be provably trained to approximately zero loss in our training regime. 
  We notice some interesting questions by reviewers and colleagues. We provide further discussion on these questions in Appendix \ref{appendix:QandA}, they may be intriguing for general readers.
  
  \BLACK{Finally, there are several important future directions. First, we empirically observe that our training regime brings strong generalization power, more theoretical analysis will be interesting. Second, our current analysis is still limited to 1-hidden-layer networks,
  we are trying to extend it to deep ones. Third, the algorithmic convergence analysis to reach the KKT point is imperative.  } 
  \clearpage
\begin{ack}
We would like to thank  Dawei Li, Zeyu Qin, Jiancong Xiao and Congliang Chen  for valuable and productive discussions. We want to thank the anonymous reviewers for their valuable suggestions and comments.
M. Hong is partially supported by an NSF grant CMMI-1727757, and an IBM Faculty Research Award. The work of Z.-Q. Luo is supported by the National Natural Science Foundation of China (No. 61731018) and the Guangdong Provincial Key Laboratory of Big Data Computation Theories and Methods.

\end{ack}

\bibliographystyle{plain}  

\bibliography{references}

\clearpage

\appendix
\numberwithin{equation}{section}

\section*{Appendix}

\section*{Potential Negative Societal Impacts}
In this paper, we discuss the expressivity and trainability of narrow neural networks. This paper provides a new understanding from the theoretical study, and such new insights will inspire a better training approach for neural networks with a small number of parameters.
In industrial applications, several aspects of impact can be expected: 
training a huge neural network is at the expense of heavy computational burdens and large
power consumption, which is a big challenge for  embedded
systems and small portable devices. In this sense, our work sheds new light on training narrow networks with much fewer parameters, so it will save energy for AI industries and companies. On the other hand, if everyone can afford to train powerful neural networks on their cell phone, then it will be a potential threat to most famous companies \& institutes who are boast of their exclusive computational advantages. Additionally, there are chances that neural networks will be used for illegal usage. 
\BLACK{
\section*{Appendix Organization}

The Appendix is organized as follows.

\begin{itemize}
    \item Appendix \ref{appendix:QandA} provides some interesting questions by reviewers and colleagues. We provide further discussion on these questions, they may be intriguing for general readers.
    
    \item Appendix \ref{appendix:morerelatedwork} provides more discussions on the literature.
    
    \item Appendix \ref{appendix:relatedwork} provides more discussions on the related work Daniely \cite{daniely2019neural}.
    \item Appendix \ref{appendix:notation} introduces all the notations that will occur in the proof.
    \item Appendix \ref{appendix:thm1thm2} and \ref{appendix:thm3} provide detailed proof for Theorem \ref{thm1} and Theorem \ref{thm3}, respectively.
    \item Appendix \ref{appendix:lemmadeep} discusses extending the current analysis to deep neural networks.
    \item Appendix \ref{appendix:experiment} introduces the following contents. (i) the formal statement of our training regime, i.e., Algorithm \ref{algo:ourmethod}. (ii) The \texttt{Pytorch} implementation for Algorithm \ref{algo:ourmethod}. (iii) Experimental details and settings for all the experiments appear in the paper. (iv) More experiments.
\end{itemize}

}
\BLACK{
\section{Some More Discussions}
\label{appendix:QandA}
In this section, we organize some frequently asked questions by reviewers and colleagues. Many of these questions may also be intriguing for general readers. We are thankful for the valuable discussion and we would like to share these questions. Here are our answers from the authors' perspective. 

\textbf{Q1:} {\it The paper merely focuses on the optimization aspect, that is, minimizing the training loss, and ignores the more important problem (from ML perspective) of generalization. It would have been helpful if the authors include a discussion on the implications of these results on the generalization error as well.}

\textbf{A1:} We agree that generalization is a very important issue for deep learning (and still largely mysterious). 
Much of recent effort is spent on explaining why a huge number of parameters can still lead to a small generalization gap. Despite this interesting line of research, for narrow networks, the risk of overfitting is much smaller (due to traditional wisdom that fewer parameters lead to a smaller generalization gap), thus generalization of narrow networks is probably less mysterious than wide networks. According to our experiments on various real datasets ( in Section \ref{section:rimagenet} \& Appendix \ref{appendix:experiment}), our training regime provides competitive or even better generalization performance than the regular training. 

We think for the current stage, it may be more imperative to resolve expressiveness and optimization issues so as to improve practical performance. That being said, we agree that the theoretical study of the generalization gap is an interesting next step for research, and we will study it in future works.

\textbf{Q2:} {\it  To which extent the result can be extended to non-smooth activations, such as ReLU or powers of ReLU?}
 
\textbf{A2:} We think it is possible, but definitely not easy. A  main reason that we analyze smooth activation is that we need to prove the full-rankness of Jacobian $J(w;v)$ in Theorem \ref{thm1}. We believe that, for narrow nets, the full-rankness of Jacobian $J(w;v)$ with both smooth and non-smooth activation can be proved, but at the current stage, we lack suitable techniques for non-smooth ones (at least it is hard to extend the current technique to ReLU). We briefly summarize the technical difference below.

1. For ReLU, every entry of the NTK matrix ($J(w;v)J(w;v)^T$) has a closed-form solution when the width $m= \infty$, and the corresponding  NTK matrix is full rank under mild data assumption. This nice property allows us to utilize concentration inequalities to keep the full-rankness of $J(w;v)J(w;v)^T$ under finite (but large enough) width. This idea is used in  Du et al, \cite{du2019gradient} for neural nets with width $m=\text{Poly} (n)$. However, this ``concentration-based'' approach is sensitive to width, it may be difficult to extend to narrow cases.

2. As for smooth activation, we utilize an important property of analytic function (shown in Lemma \ref{jacobian:analytic}). 
Lemma \ref{jacobian:analytic} is based the intrinsic property of {\it any} analytic function, instead of a ``infinite-width'' argument. Therefore,  it casts a higher possibility to use it for narrow nets. This property is also used in Li et al. \cite{li2018benefit} to prove the full-rankness of $J(w;v)$ for neural nets with width $m=O(n)$. Further, we successfully extend this result to width $m<n$ (with more sophisticated analysis).

In summary, analyzing ReLU requires new techniques. It is hard to extend the current analysis to narrow nets with general non-smooth activation.

\textbf{Q3:} {\it The “trainability” results only allow the small movement of hidden weights. This, in essence, is not desirable as it does not allow learning representations which is crucial in deep learning. Also, this is conceptually similar to the requirements in the NTK / lazy training regime, with the difference that those results offer precise convergence rates.}

\textbf{A3:} We believe our analysis is different from ``lazy training", based on the following reasons.

i)  We would like to point out that ``small movement'' does not imply our method is similar to ``lazy training".  Chizat et al.\cite{chizat2018lazy} described``lazy training" as the situation where "these two paths remain close until the algorithm is stopped". Here, the two paths correspond to the trajectory of training the original model and the linearized model respectively. ``Training in a neighborhood of initialization" is neither a sufficient nor a necessary condition of ``lazy training”. For wide nets, ``moving in a neighborhood" and "lazy training" are also co-existent, not causal. Logically speaking, for narrow nets, the two paths can be rather different while the training appears in a neighborhood.

ii) We then argue for narrow nets, linearized trajectory and neural net trajectory have to be quite different. Note that the linearized model of narrow nets cannot fit arbitrary data. In fact, when width $m<n$, the feature matrix does not span the whole $\mathbb{R}^{n}$ space if we fix first-layer hidden weights $w=w^{0} .$ Thus our training trajectory has to be rather different from the linearized model trajectory, so as to fit data. In contrast, for wide networks, random fixed features suffice to represent data, so staying close to the linearized trajectory (or even coincide) can lead to zero training error.

iii) In our setting, the movement of first-layer hidden weights is {\it necessary} (no matter small or big). So ``feature learning" is critical for our algorithm to achieve small training error. For wide net analysis, the movement of first-layer hidden weights is NOT necessary for zero training error, so there is no need for ``feature learning".

In summary, narrow networks are out of the scope of lazy training analysis and have extra difficulties, making our analysis rather different from lazy training analysis (despite some similarities such as using the full-rankness of Jacobian).

For completeness, we further explain the main differences between our work and the previous papers on wide nets, most of the following opinions are also expressed in Section \ref{section:challange}.

\begin{itemize}
    \item[1.] expressivity is rarely an issue for wide-net papers. A wide enough network (width $m>n$ ), even linear, can always fit $n$ arbitrary data samples (can be proved using simple linear algebra, which is shown in Section \ref{section:challange}). However, when $m<n$, the expressivity is questionable. Fortunately, our Theorem \ref{thm1} provides a clean positive answer.
    \item[2.] When $m<n$, we are not clear whether the claim ``GD converges to global-min in narrow nets'' is true or not. In practice, GD easily got stuck at a large-loss stationary point in narrow-net training, but the wide nets are much easier to reach near-0 loss. So what we did is NOT proving GD works in narrow nets, but designing an algorithm and showing it works (at least under certain conditions).
    \item[3.] The empirical motivation is different. Existing NTK papers tried to ``explain" why wide networks work well. We aim to ``design" methods for training narrow networks, a topic of great interest for practitioners with limited computation resources and on-device AI. 
\end{itemize}

}

\section{Additional Related Work}\label{appendix:morerelatedwork}

We provide discussions on other related work (in addition
to those closely related ones mentioned in the main body).

\textbf{Memorization of small-width networks.}
 \BLACK{ Yun et al. \cite{yun2018small2} studies how many neurons are required for a multi-layer ReLU network to memorize $n$ samples. 
 In particular, they proved that if there are at least three layers, then the number of neurons per layer needed
 to memorize the data can be $O(\sqrt{n})$. 
They also show that if initialized near a  global-min of the empirical loss, then SGD quickly finds a nearby point with much smaller loss. However, they did not mention how to find such an initialization strategy.
}

\textbf{Convergence analysis of $O(n)$-width networks.}
Zhou et al. \cite{zhou2021local} studies the local convergence theory of mildly over-parameterized 1-hidden-layer networks. They show that, as long as the initial loss is low, GD converges to a zero-loss solution.
They further propose an initialization strategy that provably returns a small loss under a mild assumption on the width. However, their proposed initialization may be costly to find since it requires solving an additional optimization problem.

\textbf{Convergence analysis of narrow networks.}
A few recent works 
\cite{nitanda2019gradient,ji2019,chen2019much} showed that under a ``$\gamma$-margin neural tangent separability'' condition, GD converges to global minimizers 
for training 2-layer ReLU net with width \footnote{Their width also depends  on desired accuracy 
$\epsilon$, and we skip the dependence here.}
 $\text{poly}( \log n,  \gamma  ) $.
 For certain special data distributions where $\gamma 
 = \text{poly}(\log n)$, their width is $\text{poly}(\log n, \log \frac{1}{ \epsilon })$.
 Nevertheless, for general data distribution,
 their width can be larger than $ n $. 

\textbf{Global landscape analysis.}
There are many works on the global landscape analysis; see, e.g.,
\cite{kawaguchi2016deep,lu2017depth, laurent2018deep, nouiehed2018learning, zhang2019depth, nguyen2018loss,li2018over, ding2019spurious,liang2018understanding,liang2018adding,liang2019revisiting,liang2021ReLU,venturi2018spurious, safran2017spurious,soltanolkotabi2019theoretical} and the surveys \cite{sun2020global,sun2020optimization}.
For networks with width less than $n$,
the  positive result on the landscape either
requires special activation like quadratic activation
\cite{soltanolkotabi2019theoretical}, or special data distribution like linearly separable or two-subspace data \cite{liang2018understanding}. 
For certain non-quadratic activations, it was shown that
sub-optimal strict local minima can exist
for networks with width less than $n $ \cite{li2018over}.
These results are different from ours in the scope,
since they discuss \textit{global landscape} of unconstrained problems,
while we discuss local landscape.

\section{More Discussion on the Related Work: Daniely \cite{daniely2019neural}}
\label{appendix:relatedwork}

Daniely \cite{daniely2019neural} studies the expressivity of 1-hidden-layer networks. They provide two results under different scenarios: to memorize  $n(1-\epsilon)$ random input-binary-label pairs via SGD, the required width is either
\begin{itemize}
    \item[(i)]  $\tilde{O}\left(\frac{n}{\epsilon^{2}}\right)$ (\cite[Theorem 5]{daniely2019neural}); or 
    \item[(ii)] $\tilde{O}( n / d)$ (\cite[Theorem 7]{daniely2019neural}),
    where $\tilde{O}$ hides a factor
    that is dependent on $ n $ and $ d $. 
\end{itemize}

In this section, we will explain our claim in Section \ref{section:relatedwork} that
 their required width can be much
 larger than ours. 

\paragraph{Major differences} In our work,  our required width is smaller than $n$ (when $d > 2$). However, in either result (i) or (ii), their required width is often much larger than $n$. In fact, their width is actually at least $O\left(n^{2}\right)$ or larger for fixed $d$,
if we consider the effect of $\epsilon$ (for
their first bound) or
the hidden factor
in $\tilde{O}$ (for their second bound). 

We will elaborate below.

For their result (i), similar to Bubeck et al. \cite{bubeck2020network}, their  required width grows with $\epsilon$. As the authors pointed out before their Theorem 7 on page 8, the required width is $O(n^3)$ if $\epsilon \leq 1/ d $. This is why they add result (ii) in \cite[Theorem  7]{daniely2019neural}.

For their result (ii), i.e.,
\cite[Theorem  7]{daniely2019neural} 
their required width
 is roughly $ O( n / d  \left( \log (d \log n)\right)^{\log n}) $. 
 Compared to the desired
 bound $O(n/d)$, their actual bound
 contains an additional factor of roughly
 $\left( \log (d \log n)\right)^{\log n}$.
We would like to stress that the additional factor is {\it not} a constant factor or a log-factor, but more like a ``super-polynomial-factor''.
As a result, the required width is actually at least $O(n^2/d)$ or larger. %
The detailed explanations
are provided next, and briefly summarized below.

\begin{itemize}
    \item  In %
    Appendix \ref{subsec: appen A1 width bound},
    we will explain why there exists an extra ``$\left(\log (d \log n)\right)^{\log n}$'' term in their required width bound. In \cite[Theorem  7]{daniely2019neural}, their statement only mentioned $\tilde{O}(n / d)$ but not the exact expression. 
    
    \item In Appendix \ref{subsec: appen A2 superpolynomial},
    we will explain why their bound is at least $O(n^2)$ or larger for fixed $d$.
    \item In 
    Appendix \ref{subsec: appen A3 other diff}, we will summarize some other differences.
\end{itemize}

\subsection{Identifying a More Precise Width Bound}\label{subsec: appen A1 width bound}
To find a more precise width bound of \cite[Theorem  7]{daniely2019neural}, we tracked the proof
 of \cite{{daniely2019neural}} as follows (note that their width $q$ is our $m$ and their sample size $m$  is our $n$, and we will use our notation below).

  \textbf{The desired result.} Their goal is to memorize the $n(1-\epsilon)$ random input-label pairs via SGD.
  
\textbf{Notation.}
 They consider binary labels, i.e., $y_i= \{+1, -1\}$ for $i=1, \cdots, n$.  
 They consider a feature vector $\Psi_{\boldsymbol{\omega}}\left(\mathbf{x}_{i}\right) \in \mathbb{R}^{dm}$, where $\boldsymbol{\omega} \in \mathbb{R}^{d \times m} $. The predictor (classifier) is in the form of $f_{\boldsymbol{\omega},\mathbf{v}}(\mathbf{x}):=\mathbf{v}^T \Psi_{\boldsymbol{\omega}}\left(\mathbf{x}_{i}\right)$, where $\mathbf{v}\in \mathbb{R}^{dm} $. 
 
 \textbf{Step 1} 
 \textbf{The desired result
 holds if there exists $ \mathbf{v}$ such that
 a certain condition holds. 
  }

 More specifically, they
 have shown that the desired
 result holds if the following claim holds:
there exists certain $\mathbf{v}$, such that
  w.h.p. over the dataset and $\boldsymbol{\omega} \in \mathbb{R}^{d \times m} $, 
\begin{equation} \label{daniely goal}
    \left\langle \mathbf{v}, \Psi_{\boldsymbol{\omega} }\left(\mathbf{x}_{i}\right)\right\rangle=y_{i}+o(1), \quad \forall i =1,\cdots, n.
\end{equation}
Thus, the goal becomes to find $\mathbf{v}$ such
that \eqref{daniely goal} holds. 
This goal is stated in the first paragraph of Section 4.3, page 15.  Though they did not define it  explicitly, $o(1)$ in \eqref{daniely goal}  means `` a constant smaller than 1''. 
 
The relation between  \eqref{daniely goal} and
the desired result is independent of the calculation of the width bound. Anyhow, for completeness, we briefly explain why \eqref{daniely goal} leads to the desired result (most readers can skip the paragraph).
 First, \eqref{daniely goal} implies
 that $f $ memorized the dataset for this
 $ ( \mathbf{v}, \Psi_{\boldsymbol{\omega}} ) $.
In fact, if \eqref{daniely goal} holds, 
then $f_{\boldsymbol{\omega},\mathbf{v}}(\mathbf{x}_i)$ has the same sign as $y_i+o(1)$.
Note that $y_{i}+o(1)$ shares the same sign as $y_{i}$ because  $y_i$ is assumed to be either $+1$ or $-1$ and $o(1)$ is a constant smaller than 1.
Together we conclude that $f_{\boldsymbol{\omega},\mathbf{v}}(\mathbf{x}_i)$ shares the same sign as  $y_i$ for any $i$, which means
the predictions 
$ \text{sign} 
(f_{\boldsymbol{\omega},\mathbf{v}}(\mathbf{x}_i) )
 = y_i$.
Second, such a $\mathbf{v}$ can be reached approximately by SGD (see Algorithm 2 on page 10
of \cite{daniely2019neural}). This requires
an argument that we skip here.

 \textbf{Step 2:} 
 \textbf{There exists $\mathbf{v}$ which satisfies
 a relation specified below. }
 
 More specifically, they proved the following result. 
\begin{myTheo}[Theorem 16, page 16, Daniely \cite{daniely2019neural}]
\label{thm_daniely16}
W.p. $1-\delta-2^{-\Omega(d)}$ over the choice of dataset and $\boldsymbol{\omega}$, there exists a $\mathbf{v}$ such that the following for all $i$:
 \begin{equation} \label{eq_daniely16}
      \left\langle\mathbf{v}, \Psi_{\boldsymbol{\omega}}\left(\mathbf{x}_{i}\right)\right\rangle=f_{\boldsymbol{\omega}}\left(\mathbf{x}_{i}\right)=y_{i}+O\left(\frac{  \left(\log (d / \delta) \right)^{\frac{c^{\prime}}{2}}}{d}\right)+O\left(\sqrt{\frac{d^{c-1} \left(\log (d / \delta)\right)^{c^{\prime}+2}}{m}}\right).
  \end{equation}

\end{myTheo}
 
 Note that in Theorem \ref{thm_daniely16}, they did not explicitly write out the expression of $c$ and $c^\prime$. The definition of these two constants can be seen in Section 3.2 and Section 4.3
 of  \ref{thm_daniely16} respectively. 
 We will discuss them in Step 3.

\textbf{Step 3}: \textbf{Identifying a condition on $m$
so that the obtained bound \eqref{eq_daniely16} in Step 2 implies the desired bound \eqref{daniely goal} in Step 1. }
   
   We demonstrate the detailed derivation next. 
\begin{itemize}
  
  \item[Step 3.0:] 
  To achieve the goal of \eqref{daniely goal},  the 3rd term of \eqref{eq_daniely16} needs to be no more than 1. This term is the crucial part to derive the width bound on $m$. We provide detailed analysis as follows.

    \item[Step 3.1:] 
  Now we need to make sure the 3rd term of \eqref{eq_daniely16} is no more than 1, which
  is equivalent to (ignoring constant-factors)
    \begin{equation}\label{expression of m zero}
        m \geq d^{c-1}(\log (d / \delta))^{(c'+2)},
    \end{equation}
    where $c^{\prime} \geq 4 c+2$. This definition of $c^\prime$ can be seen in the first paragraph of Section 4.3 on page 15;
    $c, \delta $ are specified in the next two steps. 
    
    \item[Step 3.2:] Identify 
    $ \delta = 1 / \log n$. This is specified in the paragraph below Theorem 16, page 16. Plugging it into \eqref{expression of m zero}, we have
    
        \begin{equation}\label{expression of m zero zero}
        m \geq d^{c-1}(\log (d \log n))^{(c'+2)}.
    \end{equation}
    
     \item[Step 3.3:]  Identify $c=\log n / \log d$.
    The definition of $c$ first appeared in the first paragraph of Section 3.2, where they assume the number of samples is $n=d^{c}$.  This definition is used in the first paragraph of Section 4.3, with
    a slightly different form $n/d=d^{c-1}$.
    This definition implies $c=\log n / \log d$. Further, we have
    \begin{equation}\label{expression of cprime}
        c^\prime \geq 4 c+2 \geq 4 c +2 =4 \log n / \log d +2  
    \end{equation}
    Plugging  \eqref{expression of cprime} and $d^{c-1}=n/d$  into \eqref{expression of m zero zero},
    we have:
    \begin{equation}\label{expression of m first}
    m \geq \frac{n}{d}(\log (d \log n))^{ 4 \log n / \log d +2  }.
    \end{equation}

\end{itemize}

Finally, ignoring the numerical constants,
the bound becomes 
 \begin{equation}\label{expression of m second}
    m \geq 
O( \frac{n}{d}(\log (d \log n ))^{ \log n/\log d }).
    \end{equation}

\subsection{Why The Bound is ``Super-Polynomial''}\label{subsec: appen A2 superpolynomial}

Now we explain why their bound \eqref{expression of m second} is at least $O(n^2)$ or larger for fixed $d$. The bound \eqref{expression of m second} is
a bit complicated as it depends on both
$d $ and $ n $.
We are more interested in its dependence
on $ n $, thus we fix $d $ and 
analyze how it scales with $ n $.
With fixed $d$, the exponent $ \log n / \log d$
can be simplified to $ \log n $,
thus the bound \eqref{expression of m second}
can be simplified to 
\begin{equation}\label{expression of m third}
    m \geq 
O( \frac{n}{d} (\log (d \log n ))^{ \log n }).
    \end{equation}
 We will show that this bound \eqref{expression of m second} is at least $O(n^2)$ or larger for fixed $d$. 

Define 
$ B_1 = ( \log d)^{\log n}$
and $B_2 = [\log (\log n)]^{\log n} $.
From  \eqref{expression of m third} we obtain
\begin{equation}\label{expression of m fourth}
 m \geq  O( \frac{n}{d}
  \max \left\{B_{1}, B_{2}\right\} ). 
 \end{equation}

The extra factor 
$ \max \left\{B_{1}, B_{2}\right\}  $
is  {\it not} a constant factor or a log-factor, but more like a ``super-polynomial'' factor on $n$, as explained below.

\begin{itemize}
    \item  For $B_{1}$, consider the fixed input dimension $d$ for two cases.
    \begin{itemize}
        \item For $d$ satisfying $\log d>2.7$ (i.e. $d>15$), their required width is actually $O\left(\frac{n^{2}}{d}\right)$. This is an order of magnitude larger than our bound $O(n / d)$.
        \item For $d$ satisfying $\log d>2.7^{2}$ (i.e. $d>1395$), the required width is actually $O\left(\frac{n^{3}}{d}\right)$, two orders of magnitude larger than $O(n / d)$.
    \end{itemize}
    \item For $B_2$,  when $n>2.7^{2.7^{27}} \approx 2.2 \times 10^{6}$, we have $\left(\log (\log (n))^{\log n}>n\right.$. As a result, the required width is at least $O\left(\frac{n^{2}}{d}\right)$.
\end{itemize}

Theoretically speaking, 
 it is not hard to prove that their required width can be larger than $n^{k}$ for any fixed integer $k$ (which is why we say their bound
 is ``super-polynomial''). 
Empirically speaking, 
a calculation of their bound for a real
dataset can reveal how large it is. 
On CIFAR-10 dataset \cite{krizhevsky2009learning}, $n=50000, d=3072$.  Plugging these numbers
into \eqref{expression of m second} (ignore $O(\cdot)$)
we obtain $m \geq 58290499136 >>n $.
In comparison, our required width is only $m \geq 2 n / d \approx 33 $, which is much smaller than $n=50000$.
 
In summary, rigorously speaking, their required width $\tilde{O}(n / d)$ is $\operatorname{not} O(n / d)$, but can be larger than $O\left(\frac{n^{2}}{d}\right)$ or even $O\left(\frac{n^{3}}{d}\right)$.

\subsection{Other Differences}\label{subsec: appen A3 other diff}

Besides the above discussion, there are some other differences between Daniely \cite{daniely2019neural} and our work.

First, they analyze SGD, and we analyze a constrained optimization problem and projected SGD. This may be the reason why we can get a stronger bound on width. In the experiments in Section \ref{section:experiment}, we observe that SGD performs badly when the width is small (see the first left column in (b), Figure \ref{synthetic:loss}). Therefore, we suspect an algorithmic change is needed to train narrow nets with such width (due to the training difficulty), and we indeed propose a new method to train narrow nets.

Second, they consider binary $\{+1,-1\}$ dataset, while our results apply to arbitrary labels. 
In addition, their proof seems to be highly dependent on the fact that the labels are $\{+1,-1\}$,
and seems hard to generalize to general
labels.

\section{Definition and Notations} \label{appendix:notation}
Before going through the proof details, we restate some of the important notations that will repeatedly appear in the proof, the following notations are also introduced in Section \ref{section:setting}.

We denote $\left\{\left(x_{i}, y_{i}\right)\right\}_{i=1}^{n} \subset \mathbb{R}^d \times \mathbb{R}$ as the  training samples, where $x_i \in \mathbb{R}^d$, $y_i \in \mathbb{R}$. 
For theoretical analysis, 
we focus on 1-hidden-layer neural networks  $f(x ; \theta)=\sum_{j=1}^{m} v_{j} \sigma\left(w_{j}^{T} x\right) \in \mathbb{R
}$, where $\sigma(\cdot)$ is the activation function, $w_{j} \in \mathbb{R}^{d}$ and $v_{j} \in \mathbb{R}$ are the parameters to be trained. 
To learn such a neural network, we search for the optimal parameter $\theta=(w, v)$ by minimizing the empirical loss:
\begin{equation*}
  \min _{\theta} \ell(\theta)=\frac{1}{2} \sum_{i=1}^{n}\left(y_{i}-f\left(x_{i} ; \theta\right)\right)^{2}.
\end{equation*}
Sometimes we also use $\ell(w;v)$ or $f(w;x,v)$ to emphasize the role of $w$. 

We use the following shorthanded notations:
\begin{itemize}
    \item $x:=(x_1^T;\dots;x_n^T)\in \mathbb{R}^{n\times d}$, $y:=(y_1,\dots,y_n)^T \in \mathbb{R}^{n}$;
    \item $w:=(w_{1},\dots, w_m)_{j=1}^m \in \mathbb{R}^{d\times m}$, $v:=(v_{1}, \dots, v_m)_{j=1}^m \in \mathbb{R}^{m}$;
    \item $w_{a,b}$ indicates the $b$-th component of $w_a \in \mathbb{R}^d$; 
    \item $\theta^0=(w^0, v^0)$ indicates the initial parameters. Unless otherwise stated, it means the parameters at {\it the mirrored} LeCun's initialization given in Algorithm \ref{initial};
    \item $f(w;v):=(f(x_1;w,v),f(x_2;w,v),\dots, f(x_n;w,v))^T \in \mathbb{R}^{n}$, indicating the neural network output on the whole dataset $x$.
    \item  Define $\ell\circ f = \frac{1}{2}\|y-f\|_2^2$.
\end{itemize}
  We denote the Jacobian matrix of $f(w;v)$ w.r.t $w$ as {\tiny
$$J(w;v):=\left[\begin{array}{c}\nabla_w f\left(w ; x_{1}, v\right)^{T} \\ \vdots \\  \nabla_w f\left(w ; x_{n}, v\right)^{T} \end{array}\right]  =\left[\begin{array}{ccc}v_{1} \sigma^{\prime}\left(w_{1}^{T} x_{1}\right) x_{1}^T & \cdots & v_{m} \sigma^{\prime}\left(w_{m}^{T} x_{1}\right) x_{1}^T \\ &\vdots & \\ v_{1} \sigma^{\prime}\left(w_{1}^{T} x_{n}\right) x_{n}^T & \cdots & v_{m} \sigma^{\prime}\left(w_{m}^{T} x_{n}\right) x_{n}^T\end{array}\right]\in \mathbb{R}^{n \times md}.$$}%
We define the feature matrix 
{\tiny
\begin{align*}
    \Phi(w):=\left[\begin{array}{c} \sigma\left(w_{1}^{T} x_1\right),\dots,\sigma\left(w_{m}^{T} x_1\right)\\ \vdots \\  \sigma\left(w_{1}^{T} x_n\right),\dots,\sigma\left(w_{m}^{T} x_n\right) \end{array}\right]\in \mathbb{R}^{n \times m}.
\end{align*}
 }%
 
\section{Proof of Theorem \ref{thm1} } \label{appendix:thm1thm2}

The proof of Theorem \ref{thm1} consists of two parts: when the hidden weights $w$ is in the neighborhood of the mirrored LeCun's initialization, we have (i) There exists a global-min with 0 loss, (ii) every stationary point is a global-min. We prove these two arguments respectively. The first part (argument (i)) can be seen in Appendix \ref{appendix:thm1}, the second part (argument (ii)) can be seen in Appendix \ref{appendix:thm2}.

\subsection{Proof of The First Part of Theorem \ref{thm1}}
\label{appendix:thm1}

To prove the first part of Theorem \ref{thm1}, we use the Inverse Function Theorem (IFT) \cite{IFT} at the mirrored LeCun's initialization. IFT is stated below. 

\begin{myTheo}[Inverse function theorem (IFT)]\label{inv}
  Let $\psi: U \rightarrow \mathbb{R}^{n}$ be a $C^{1}$ -map where $U$ is open in $\mathbb{R}^{n}$ and $w \in U.$ Suppose that the Jacobian $J\left(w\right)$ is invertible. There exist open sets $W$ and $F$ containing $w$ and $\psi \left(w\right)$ respectively, such that the restriction of $\psi$ on $W$ is a bijection onto $F$ with a $C^{1}$ -inverse.
\end{myTheo}

Our overall proof idea is as follows:
 we will use IFT to show that:
     then for any $y \in \mathbb{R}^n$ and any small enough $\epsilon$, there
   exists a $w^* \in B_\epsilon(w^0)$ whose prediction output $f(w^*;v^0) \propto y-f(w^0;v^0)$. Additionally, since $f(w^0;v^0)=0$, we have $f(w^*;v^0) \propto y$. Once this is shown, then we just need to scale all the outer weight $v_j$ uniformly and the output will be exactly $y$ since $f(w^*;v)$ is linear in $v$. 
More details can be seen as follows.

In our case, let $\psi=f(w;v^0)$ be the function of $w$, mapping from $\mathbb{R}^{md}$ to  $\mathbb{R}^{n}$. It may appears that IFT cannot be directly applied since $md\geq 2n$ (cf. Assumption \ref{assum1}), so $f(w;v^0)$ is not dimension-preserved mapping. %
However, this issue can be alleviated by applying the IFT to {\it a subvector} of $w$, while fixing the rest of the variables.

More specifically, we denote $n=k_1 d+k_2$ with $k_1, k_2 \in \mathbb{N}$, and $w=(\tilde{w}^T,\tilde{w}^{\prime T})^T$, where  
$\tilde{w}=(w_1^T,\cdots, w_{k_1}^T,w_{k_1+1,1},\cdots, w_{k_1+1,k_2})^T \in \mathbb{R}^{n}$ and $\tilde{w}^{\prime}=(w_{k_1+1,k_2+1},\cdots,w_{k_1+1,d},w_{k_1+2}^T,\cdots, w_{m}^T )^T\in \mathbb{R}^{md-n}.$ 
Here, $w_{a,b}$ indicates the $b$-th component of $w_a \in \mathbb{R}^d$.

We now apply IFT  to $f(\tilde{w};v^0,\tilde{w}^{\prime 0}) \in \mathbb{R}^n$ (this notation views $\tilde{w}$ as the variable and $v^0,\tilde{w}^{\prime 0}$ are treated as parameters).  
Firstly, in Lemma \ref{fullrank}, we prove that w.p.1, the corresponding Jacobian matrix $J(\tilde{w}^0;v^0,\tilde{w}^{\prime 0}) \in \mathbb{R}^{n\times n}$ is of full rank at the mirrored LeCun's initialization, so that the condition for IFT holds.
Then, by IFT, there exist open sets $W$ and $F$ containing $\tilde{w}^0$ and $f(\tilde{w};v^0,\tilde{w}^{\prime 0})$ respectively, such that the restriction of $f(\tilde{w};v^0,\tilde{w}^{\prime 0})$ on $W$ is a bijection onto $F$. Here, we denote $\epsilon$ and $\delta$ as the radius of $W$ and $F$, respectively.

Now, since $f(\tilde{w}^0;v^0,\tilde{w}^{\prime 0})=f(w^0;v^0)=0 \in \mathbb{R}^n$, set $F$ contains all possible directions pointed from the origin.   That is to say, for any label vector $y \in \mathbb{R}^n$, we can always scale it using $\delta$, such that  $\delta \frac{y}{\|y\|} \in F$, and then, by IFT ,
there exists a $\tilde{w}^* \in B_\epsilon(\tilde{w}^0) \subset W$ satisfying 

  \begin{equation} \label{thm1:mapping}
    f(\tilde{w}^*;v^0,\tilde{w}^{\prime 0})=\delta \frac{y}{\|y\|}.
  \end{equation}

  Since $\tilde{w}$ is just the truncated version of $w$, (\ref{thm1:mapping}) implies: there exists a $w^* \in B_\epsilon(w^0)$, s.t. 

\begin{equation} \label{thm1:mapping2}
  f(w^*;v^0)=\delta \frac{y}{\|y\|}.
\end{equation}
  
 Now, we scale the outer weight to $v^*=\frac{\|y\|}{\delta}v^0$ and the output will be exactly $y$, i.e.:

  $$f(w^*;v^*)=y.$$

  Therefore, the proof is concluded.

 \BLACK{
 \paragraph{Remark: ``there exists an $\epsilon$'' or ``any small $\epsilon$''?}
 Readers may mention that IFT states ``there exists a neighborhood with size $\epsilon$'', however, Theorem \ref{thm1} claims for ``any small enough $\epsilon$''.  We would like to clarify that the statement of Theorem \ref{thm1} is {\it not} a typo. Here is the reason: in the statement of IFT, ``existence of a small neighborhood'' will imply ``IFT holds for any subset of this neighborhood'', so actually, Theorem \ref{thm1} holds for any small (enough) neighborhood with size $\epsilon$.}

  \begin{lemma}\label{fullrank}
    Under Assumption \ref{assum1}, \ref{assum2} and \ref{assum3}, as a function of $w$, $v$ and $x$, $J(\tilde{w}^0;v^0,\tilde{w}^{\prime 0}) \in \mathbb{R}^{n\times n}$ is of full rank at the mirrored LeCun's initialization, w.p.1..
  \end{lemma}

  \begin{proof}[Proof of Lemma \ref{fullrank}]
    Recall the Jacobian matrix of $f(\tilde{w};v^0,\tilde{w}^{\prime 0})$ w.r.t. $\tilde{w}$:
    $$J(\tilde{w}^0;v^0,\tilde{w}^{\prime 0})
    :=\left[\begin{array}{c}\nabla_{\tilde{w}} f\left(\tilde{w}^0;x_1,v^0,\tilde{w}^{\prime 0}\right)^{T} \\ \vdots \\  \nabla_{\tilde{w}} f\left(\tilde{w}^0;x_n,v^0,\tilde{w}^{\prime 0}\right)^{T} \end{array}\right] \in \mathbb{R}^{n \times n}. $$

    We first consider a general case where $k_2\neq 0$. Since  $f\left(\tilde{w};x,v,\tilde{w}^{\prime}\right)=f(w;x,v)=\sum_{j=1}^{m} v_{j} \sigma\left(w_{j}^{T} x\right)$, taking derivative w.r.t. $\tilde{w}$ yields $J(\tilde{w}^0;v^0,\tilde{w}^{\prime 0})$ equals to: 

    {\tiny
    \begin{equation} \label{jacobian:squarematrix}
    \left(\begin{array}{cccccc}v_{1}^0 \sigma^{\prime}\left(w_{1}^{0T} x_{1}\right) x_{1}^T & \cdots & v_{k_1}^0 \sigma^{\prime}\left(w_{k_1}^{0T} x_{1}\right) x_{1}^T &v_{k_1+1}^0 \sigma^{\prime}\left(w_{k_1+1}^{0T} x_{1}\right)x_{1,1}& \cdots & v_{k_1+1}^0 \sigma^{\prime}\left(w_{k_1+1}^{0T} x_{1}\right)x_{1,k_2} \\ & &\vdots& &  &  \\  v_{1}^0 \sigma^{\prime}\left(w_{1}^{0T} x_{n}\right) x_{n}^T & \cdots & v_{k_1}^0 \sigma^{\prime}\left(w_{k_1}^{0T} x_{n}\right) x_{n}^T  &v_{k_1+1}^0 \sigma^{\prime}\left(w_{k_1+1}^{0T} x_{n}\right)x_{n,1}& \cdots & v_{k_1+1}^0 \sigma^{\prime}\left(w_{k_1+1}^{0T} x_{n}\right)x_{n,k_2}\end{array}\right).
    \end{equation}

    }

  To prove the full-rankness of $J(\tilde{w}^0;v^0,\tilde{w}^{\prime 0})$, we need to show that w.p.1, $\det (J(\tilde{w}^0;v^0,\tilde{w}^{\prime 0})) \neq 0$. Here, $\det (J(\tilde{w}^0;v^0,\tilde{w}^{\prime 0}))$ is an analytic function since the activation function $\sigma(\cdot)$ is analytic (see Assumption \ref{assum2}). Therefore, we borrow an important result of \cite[Proposition 0]{mityagin2015zero} which states that the zero set of an analytic function is either the whole domain or zero-measure. The result is formally stated as the following lemma under our notation.

  \begin{lemma}\label{jacobian:analytic}
    Suppose that: as a function of $\tilde{w}$, $\tilde{w}^{\prime}$, $v$ and $x$, $\det (J(\tilde{w}^0;v^0,\tilde{w}^{\prime 0})): \mathbb{R}^{md+m+nd} \rightarrow \mathbb{R}$ is a real analytic function on $\mathbb{R}^{md+m+nd}$. If $\det (J(\tilde{w}^0;v^0,\tilde{w}^{\prime 0}))$ is not identically zero, then its zero set $\Omega=\left\{\tilde{w}^0,v^0,\tilde{w}^{\prime 0},x  \mid \det (J(\tilde{w}^0;v^0,\tilde{w}^{\prime 0}))=0\right\}$ has zero measure.
  \end{lemma}

  Based on Lemma \ref{jacobian:analytic}, in order to prove $\det (J(\tilde{w}^0;v^0,\tilde{w}^{\prime 0})) \neq 0$ w.p.1, we only need to prove it is not identically zero. To do so, we first transform $\det (J(\tilde{w}^0;v^0,\tilde{w}^{\prime 0}))$ into its equivalent form:

  \begin{eqnarray}
    (\ref{jacobian:squarematrix})&=&  \det\left([B_1,\cdots, B_{k_2}, C_{k_2+1}, \cdots, C_{d}]\right),
  \end{eqnarray}
  where
  $$ 
  B_j=\left(\begin{array}{cccccc}v_{1}^0 \sigma^{\prime}\left(w_{1}^{0T} x_{1}\right) x_{1,j} & \cdots & v_{k_1+1}^0 \sigma^{\prime}\left(w_{k_1+1}^{0T} x_{1}\right) x_{1,j}  \\ 
    &\vdots&   \\ 
    v_{1}^0 \sigma^{\prime}\left(w_{1}^{0T} x_{n}\right) x_{n,j} & \cdots & v_{k_1+1}^0 \sigma^{\prime}\left(w_{k_1+1}^{0T} x_{n}\right) x_{n,j} \end{array}\right)  \in \mathbb{R}^{n\times (k_1+1)}, \quad j=1,\cdots, k_2,
  $$

  $$
  C_j=\left(\begin{array}{cccccc}v_{1}^0 \sigma^{\prime}\left(w_{1}^{0T} x_{1}\right) x_{1,j} & \cdots & v_{k_1}^0 \sigma^{\prime}\left(w_{k_1}^{0T} x_{1}\right) x_{1,j}  \\ 
    &\vdots&   \\ 
    v_{1}^0 \sigma^{\prime}\left(w_{1}^{0T} x_{n}\right) x_{n,j} & \cdots & v_{k_1}^0 \sigma^{\prime}\left(w_{k_1}^{0T} x_{n}\right) x_{n,j} \end{array}\right)\in \mathbb{R}^{n\times k_1}, \quad j=k_2+1,\cdots, d.
  $$

  In addition, $B_j$ can be further rewritten as:
  {\small
  $$
  B_j=\left(\begin{array}{c}  B_{1,j}\\ 
    \vdots \\
    B_{k_2,j}\\
    D_{j} \end{array}\right) \in \mathbb{R}^{n\times (k_1+1)},
  $$}
  where for $i=1, \cdots k_2$:
  
  {\tiny
  \begin{eqnarray*}
    B_{i,j}&=& \!\!\!\!\!\!\!\!\left(\begin{array}{cccccc}v_{1}^0 \sigma^{\prime}\left(w_{1}^{0T} x_{(k_1+1)(i-1)+1}\right) x_{(k_1+1)(i-1)+1,j} & \cdots & v_{k_1+1}^0 \sigma^{\prime}\left(w_{k_1+1}^{0T} x_{(k_1+1)(i-1)+1}\right) x_{(k_1+1)(i-1)+1,j}  \\ 
      &\vdots&   \\ 
      v_{1}^0 \sigma^{\prime}\left(w_{1}^{0T} x_{(k_1+1)i}\right) x_{(k_1+1)i,j} & \cdots & v_{k_1+1}^0 \sigma^{\prime}\left(w_{k_1+1}^{0T} x_{(k_1+1)i}\right) x_{(k_1+1)i,j} \end{array}\right) \\
      &\in& \mathbb{R}^{(k_1+1)\times (k_1+1)},
  \end{eqnarray*}
  }
 and 
 {\tiny
 \begin{eqnarray*}
  D_{j}&=&\left(\begin{array}{cccccc}v_{1}^0 \sigma^{\prime}\left(w_{1}^{0T} x_{(k_1+1)k_2+1}\right) x_{(k_1+1)k_2+1,j} & \cdots & v_{k_1+1}^0 \sigma^{\prime}\left(w_{k_1+1}^{0T} x_{(k_1+1)k_2+1}\right) x_{(k_1+1)k_2+1,j}  \\ 
    &\vdots&   \\ 
    v_{1}^0 \sigma^{\prime}\left(w_{1}^{0T} x_{n}\right) x_{n,j} & \cdots & v_{k_1+1}^0 \sigma^{\prime}\left(w_{k_1+1}^{0T} x_{n}\right) x_{n,j} \end{array}\right) \\
    &\in& \mathbb{R}^{(n-(k_1+1)k_2)\times(k_1+1)}.
 \end{eqnarray*}
 }

  Similarly, $C_j$ can be further rewritten as:
{\small
  $$
  C_j=\left(\begin{array}{c}  C_{1,j}\\ 
    \vdots \\
    C_{k_2,j}\\
    E_{j} \end{array}\right) \in \mathbb{R}^{n\times (k_1)},
  $$
}
  where for $i=1, \cdots k_2$:

  {\tiny
  \begin{eqnarray*}
    C_{i,j}&=&\!\!\!\!\!\!\!\!\left(\begin{array}{cccccc}v_{1}^0 \sigma^{\prime}\left(w_{1}^{0T} x_{(k_1+1)(i-1)+1}\right) x_{(k_1+1)(i-1)+1,j} & \cdots & v_{k_1}^0 \sigma^{\prime}\left(w_{k_1}^{0T} x_{(k_1+1)(i-1)+1}\right) x_{(k_1+1)(i-1)+1,j}  \\ 
      &\vdots&   \\ 
      v_{1}^0 \sigma^{\prime}\left(w_{1}^{0T} x_{(k_1+1)i}\right) x_{(k_1+1)i,j} & \cdots & v_{k_1}^0 \sigma^{\prime}\left(w_{k_1}^{0T} x_{(k_1+1)i}\right) x_{(k_1+1)i,j} \end{array}\right) \\
      & \in& \mathbb{R}^{(k_1+1)\times k_1} 
  \end{eqnarray*}
  }

 and 

 {\tiny

 \begin{eqnarray*}
  E_{j}&=&\left(\begin{array}{cccccc}v_{1}^0 \sigma^{\prime}\left(w_{1}^{0T} x_{(k_1+1)k_2+1}\right) x_{(k_1+1)k_2+1,j} & \cdots & v_{k_1}^0 \sigma^{\prime}\left(w_{k_1}^{0T} x_{(k_1+1)k_2+1}\right) x_{(k_1+1)k_2+1,j}  \\ 
    &\vdots&   \\ 
    v_{1}^0 \sigma^{\prime}\left(w_{1}^{0T} x_{n}\right) x_{n,j} & \cdots & v_{k_1}^0 \sigma^{\prime}\left(w_{k_1}^{0T} x_{n}\right) x_{n,j} \end{array}\right) \\
    & \in& \mathbb{R}^{(n-(k_1+1)k_2)\times k_1}.
 \end{eqnarray*}

 }

  Therefore, we can rewrite $\det (J(\tilde{w}^0;v^0,\tilde{w}^{\prime 0}))$ as the following form:

 \begin{equation}\label{jacobian:blockmatrix}
  \det (J(\tilde{w}^0;v^0,\tilde{w}^{\prime 0})) = \left(\begin{array}{cccccc} B_{1,1} & \cdots & B_{1,k_2} & C_{1,k_2+1}&\cdots & C_{1,d} \\ 
   & &\vdots&  & &  \\ 
   B_{k_2,1} & \cdots & B_{k_2,k_2} & C_{k_2,k_2+1}&\cdots & C_{k_2,d} \\ 
   D_1 & \cdots & D_{k_2} & E_{k_2+1}&\cdots & E_{d} \\ 
  \end{array}\right). 
 \end{equation}

 Based on Lemma \ref{jacobian:analytic}, in order to prove $\det (J(\tilde{w}^0;v^0,\tilde{w}^{\prime 0})) \neq 0$ w.p.1., we only need to prove that, as a function of  $\tilde{w}$, $\tilde{w}^{\prime}$, $v$ and $x$,  $\det (J(\tilde{w}^0;v^0,\tilde{w}^{\prime 0}))$ is not identically zero. To do so, we just need to construct a dataset $x$ such that $(\ref{jacobian:blockmatrix}) \neq 0$. We construct such $x:=(x_1,\dots,x_n)\in \mathbb{R}^{n\times d}$ in the following way: 

 \begin{itemize}
   \item[(1)] For $i=1,\cdots, (k_1+1)$:  $x_{i,j}=\{\begin{array}{lll} \\\delta_i, &&  j=1 \\ 0,  & & \text{otherwise} \end{array}$, 
   where $\delta_i\neq\delta_{i^{\prime}} \neq 0$, $\forall i, i^{\prime}$.
   \item[(2)] For $i=(k_1+1)+1,\cdots, 2(k_1+1)$:  $x_{i,j}=\{\begin{array}{lll} \delta_i, &&  j=2 \\ 0,  & & \text{otherwise} \end{array}$,
   where $\delta_i\neq\delta_{i^{\prime}} \neq 0$, $\forall i, i^{\prime}$.
   \item[(3)] $\cdots$ 
   \item[(4)] For $i=(k_2-1)(k_1+1)+1,\cdots, k_2(k_1+1)$:  $x_{i,j}=\{\begin{array}{lll} \delta_i, &&  j=k_2 \\ 0,  & & \text{otherwise} \end{array}$,
   where $\delta_i\neq\delta_{i^{\prime}} \neq 0$, $\forall i, i^{\prime}$.
   \item[(5)] For $i=k_2(k_1+1)+1,\cdots, n$:  $x_{i,j}=\{\begin{array}{lll} 1, &&  j=i-k_1k_2 \\ 0,  & & \text{otherwise} \end{array}$.
 \end{itemize}

 Under such a construction, (\ref{jacobian:blockmatrix}) becomes the determinant of a block-diagonal matrix:

 \begin{equation}\label{jacobian:blockmatrix2}
  \det (J(\tilde{w}^0;v^0,\tilde{w}^{\prime 0})) = \det \left(\begin{array}{cccccc} B_{1,1} & \cdots &0  & 0\\ 
   & &\vdots&  \\ 
   0 & \cdots & B_{k_2,k_2} & 0\\ 
  0 & \cdots & 0 & E\\ 
  \end{array}\right),
 \end{equation}

 where $E=[ E_{k_2+1}\cdots  E_{d}] $ is a square matrix in $\mathbb{R}^{\left(n-(k_1+1)k_2 \right) \times \left(n-(k_1+1)k_2\right)}$. To prove $(\ref{jacobian:blockmatrix2})\neq 0$, we need to prove $B_{1,1}, \cdots, B_{k_2,k_2}$ and $E$ are all full rank matrices.

 As for the full-rankness of $E$, thanks to the construction (5), $E=[ E_{k_2+1}\cdots  E_{d}]$ now becomes: 

 $$
 E_{k_2+1}=\left(\begin{array}{cccccc}v_{1}^0 \sigma^{\prime}\left(w_{1}^{0T} x_{(k_1+1)k_2+1}\right)  & \cdots & v_{k_1}^0 \sigma^{\prime}\left(w_{k_1}^{0T} x_{(k_1+1)k_2+1}\right)  \\ 
  0& \cdots & 0 \\
   &\vdots&   \\ 
   0& \cdots & 0  \end{array}\right) \in \mathbb{R}^{(n-(k_1+1)k_2)\times k_1},
 $$

 $$
 E_{k_2+2}=\left(\begin{array}{cccccc}
  0& \cdots & 0 \\
  v_{1}^0 \sigma^{\prime}\left(w_{1}^{0T} x_{(k_1+1)k_2+2}\right)
   & \cdots & v_{k_1}^0 \sigma^{\prime}\left(w_{k_1}^{0T} x_{(k_1+1)k_2+2}\right)  \\ 
   &\vdots&   \\ 
   0& \cdots & 0  \end{array}\right) \in \mathbb{R}^{(n-(k_1+1)k_2)\times k_1},
 $$

 and so on so for: 
 $$
 E_{d}=\left(\begin{array}{cccccc}
  0& \cdots & 0 \\
   &\vdots&   \\ 
   0& \cdots & 0  \\
   v_{1}^0 \sigma^{\prime}\left(w_{1}^{0T} x_{n}\right)
    & \cdots & v_{k_1}^0 \sigma^{\prime}\left(w_{k_1}^{0T} x_{n}\right)  
  \end{array}\right) \in \mathbb{R}^{(n-(k_1+1)k_2)\times k_1}.
 $$

  Since $md \geq 2n$ and $J(\tilde{w}^0;v^0,\tilde{w}^{\prime 0}) \in \mathbb{R}^{n \times n}$  is the jacobian w.r.t. the first $n$ components of $w \in \mathbb{R}^{md}$, it only involves $w_1, w_2, \cdots w_{k_1+1}$ and it will not reach beyond $w_{\frac{m}{2}} \in \mathbb{R}^d$. Recall in the mirrored LeCun's initialization, we only copy the 2nd half of $w$: $(w^{0}_{\frac{m}{2}+1},\dots,w^{0}_{m})\leftarrow(-w^{0}_{1},\dots,-w^{0}_{\frac{m}{2}})$, that is to say, $w_1, w_2, \cdots w_{k_1+1}$ are independent Gaussian random variables, unaffected by the copying phase, similarly for $v_1,\cdots, v_{k_1+1}$. 

  In short, since Gaussian random variables take value 0 on a zero probability measure,  and $\sigma(z)=0$ only happens when $z=0$ (see Assumption \ref{assum2}), we have $E=[ E_{k_2+1}\cdots  E_{d}]$ is full rank w.p.1.

  As for the full-rankness of $B_{kk}$, $k=1,\cdots, k_2$, we only need to prove $B_{11}$ is invertible, the proof of the rest of $B_{kk}$ are the same. 

  Under construction (1), we have:
  $$  
  B_{1,1}=\left(\begin{array}{cccccc}v_{1}^0 \sigma^{\prime}\left(w_{1}^{0T} x_{1}\right)\delta_1  & \cdots & v_{k_1+1}^0 \sigma^{\prime}\left(w_{k_1+1}^{0T} x_{1}\right)\delta_1  \\ 
   &\vdots&   \\ 
   v_{1}^0 \sigma^{\prime}\left(w_{1}^{0T} x_{(k_1+1)}\right)\delta_{k_1+1}  & \cdots & v_{k_1+1}^0 \sigma^{\prime}\left(w_{k_1+1}^{0T} x_{(k_1+1)}\right)\delta_{k_1+1}  \end{array}\right) \in \mathbb{R}^{(k_1+1)\times (k_1+1)} .
 $$

 Again, since Gaussian random variables take value 0 on a zero probability measure, and $\delta_i\neq 0$ for $i=1,\cdots, k_1+1$, we only need to prove the following $\tilde{B}_{1,1}$ is full rank: 
$$
\tilde{B}_{1,1}=\left(\begin{array}{cccccc} \sigma^{\prime}\left(w_{1}^{0T} x_{1}\right)  & \cdots &  \sigma^{\prime}\left(w_{k_1+1}^{0T} x_{1}\right)  \\ 
  &\vdots&   \\ 
  \sigma^{\prime}\left(w_{1}^{0T} x_{(k_1+1)}\right)  & \cdots &  \sigma^{\prime}\left(w_{k_1+1}^{0T} x_{(k_1+1)}\right)  \end{array}\right) \in \mathbb{R}^{(k_1+1)\times (k_1+1)} .
$$

Next, we borrow the following lemma from \cite[Proposition 1]{li2018benefit}, which is the restatement under our notation (their original statement applies for deep neural network, here, we restate it for 1-hidden-layer case in Lemma \ref{jacobian:fullrank}). 

\begin{lemma}\label{jacobian:fullrank}
  Under Assumption \ref{assum2}, given an 1-hidden-layer neural network with width $m \geq n$, Let $\Omega=\left\{w \mid \operatorname{rank}\left(\Phi(w)\right)<\min \left\{m, n\right\}\right\}$, where $\Phi(w)$ is the hidden feature matrix
  
  $$\Phi(w):=\left[\begin{array}{c} \sigma\left(w_{1}^{T} x_1\right),\dots,\sigma\left(w_{m}^{T} x_1\right)\\ \vdots \\  \sigma\left(w_{1}^{T} x_n\right),\dots,\sigma\left(w_{m}^{T} x_n\right) \end{array}\right]\in \mathbb{R}^{n \times m}.$$

  Suppose there exists a dimension $k$ such that $x_{i,k} \neq x_{i^{\prime},k}, \forall i \neq i^{\prime}$, then $\Omega$ is a zero-measure set.
\end{lemma}

To prove the full-rankness of $\tilde{B}_{1,1}$, we regard it as the hidden feature matrix of an 1-hidden-layer neural network equipped with width $m^\prime=k_1+1$ and activation function $\sigma^{\prime}(z)$, which satisfies Assumption \ref{assum2}. 
In addition,  recall $\delta_i\neq\delta_{i^{\prime}} \neq 0$ for $\forall i, i^{\prime}$, so $(x_1,\cdots, x_{k_1+1})$  satisfies the condition of Lemma \ref{jacobian:fullrank} w.p.1, and the sample size equals to the width $m^\prime=k_1+1$. 
Therefore, all the assumptions are satisfied and Lemma \ref{jacobian:fullrank} directly shows that $\tilde{B}_{1,1}$ is invertible w.p.1.. 

Similarly, with the same proof technique, it can be shown that the rest of ${B}_{k,k}$ are also invertible w.p.1.. 
In conclusion, we have constructed a dataset $x$, such that  (\ref{jacobian:blockmatrix2}) is non-zero w.p.1., which implies $\Omega=\left\{\tilde{w}^0,v^0,\tilde{w}^{\prime 0},x  \mid \det (J(\tilde{w}^0;v^0,\tilde{w}^{\prime 0}))=0\right\}$  has zero measure by Lemma \ref{jacobian:analytic}. In other words, under the joint distribution of $\tilde{w}^0$, $v^0$, $\tilde{w}^{\prime 0}$ and $x$, $ J(\tilde{w}^0;v^0,\tilde{w}^{\prime 0})$ is invertible w.p.1.. Recall $\tilde{w}^0$, $v^0$, and $\tilde{w}^{\prime 0}$ all follow continuous distribution, furthermore, $x$ also follows a continuous distribution (Assumption \ref{assum3}), so $J(\tilde{w}^0;v^0,\tilde{w}^{\prime 0})$ is still invertible w.p.1. under the distribution of $x$,
so the whole proof is completed.

When $n=k_1d$, or equivalently, $k_2=0$, things become easier and we just need to change the size of $B_{kk}$ to $\mathbb{R}^{k_1\times k_1}$, and there is no need to consider $C_{i,j}$, $D_j$  and $E_j$, the rest of the proof is the same, we omit it for brevity.
\end{proof}

\subsection{Proof of The Second Part of Theorem \ref{thm1}}
\label{appendix:thm2}

Suppose we have a 1-hidden-layer neural network $f(x;\theta)=\sum_{j=1}^{m} v_{j} \sigma\left(w_{j}^{T} x\right) \in \mathbb{R}$, and let $f(\theta)\in \mathbb{R}^n$ be output of $(x;\theta)$ on the dataset $x=(x_1,\cdots, x_n)$, let $J(w^{*};v^*)\in \mathbb{R}^{n \times md}$ be its Jacobian matrix w.r.t. $w$ at the  stationary point $\theta^*=(w^*,v^*)$, we have:
  
\begin{equation}
  \nabla_{w}\ell(\theta^*)=J(w^{*};v^*)^T(f(\theta^*)-y)=0.
\end{equation}

Therefore, as long as we can prove that  $J(w^{*};v^*)^T \in \mathbb{R}^{md \times n}$ is full column rank, the stationary point  $\theta^*$ will become a global minimizer with $\ell(\theta^*)=0$, so the proof is completed.

Now we prove the full-rankness of $J(w^{*};v^*)$. Recall in Lemma \ref{fullrank}, we have proved that, as a function of 
$x$, $\det (J(\tilde{w}^0;v^0,\tilde{w}^{\prime 0}))\neq 0$ w.p.1., where $ J(\tilde{w}^0;v^0,\tilde{w}^{\prime 0})$ equals to

{\tiny
\begin{equation} \label{jacobian:squarematrix2}
\left(\begin{array}{cccccc}v_{1}^0 \sigma^{\prime}\left(w_{1}^{0T} x_{1}\right) x_{1}^T & \cdots & v_{k_1}^0 \sigma^{\prime}\left(w_{k_1}^{0T} x_{1}\right) x_{1}^T &v_{k_1+1}^0 \sigma^{\prime}\left(w_{k_1+1}^{0T} x_{1}\right)x_{1,1}& \cdots & v_{k_1+1}^0 \sigma^{\prime}\left(w_{k_1+1}^{0T} x_{1}\right)x_{1,k_2} \\ & &\vdots& &  &  \\  v_{1}^0 \sigma^{\prime}\left(w_{1}^{0T} x_{n}\right) x_{n}^T & \cdots & v_{k_1}^0 \sigma^{\prime}\left(w_{k_1}^{0T} x_{n}\right) x_{n}^T  &v_{k_1+1}^0 \sigma^{\prime}\left(w_{k_1+1}^{0T} x_{n}\right)x_{n,1}& \cdots & v_{k_1+1}^0 \sigma^{\prime}\left(w_{k_1+1}^{0T} x_{n}\right)x_{n,k_2}\end{array}\right),
\end{equation}
}

Since  $\|w^*-w^0\|_F \leq \epsilon$ and the determinant is a continuous function of $w$,  we have $\det (J(\tilde{w}^*;v^0,\tilde{w}^{\prime *}))\neq 0$ (w.p.1.) when $\epsilon$ is small. In addition, as we can see from (\ref{jacobian:squarematrix2}), $(v_1^0,\cdots, v_{k_1+1}^0)$ are just constant terms in the corresponding columns, so the full-rankness of (\ref{jacobian:squarematrix2}) still holds if we change $v_j^0$ to $v_j^* \neq 0$. In summary, $\det (J(\tilde{w}^*;v^*,\tilde{w}^{\prime *}))\neq 0$ when $v^*$ is entry-wise non-zero, where $J(\tilde{w}^*;v^*,\tilde{w}^{\prime *})$ equals to

{\tiny
\begin{equation} \label{jacobian:squarematrix3}
\left(\begin{array}{cccccc}v_{1}^* \sigma^{\prime}\left(w_{1}^{*T} x_{1}\right) x_{1}^T & \cdots & v_{k_1}^* \sigma^{\prime}\left(w_{k_1}^{*T} x_{1}\right) x_{1}^T &v_{k_1+1}^* \sigma^{\prime}\left(w_{k_1+1}^{*T} x_{1}\right)x_{1,1}& \cdots & v_{k_1+1}^* \sigma^{\prime}\left(w_{k_1+1}^{*T} x_{1}\right)x_{1,k_2} \\ & &\vdots& &  &  \\  v_{1}^* \sigma^{\prime}\left(w_{1}^{*T} x_{n}\right) x_{n}^T & \cdots & v_{k_1}^* \sigma^{\prime}\left(w_{k_1}^{*T} x_{n}\right) x_{n}^T  &v_{k_1+1}^* \sigma^{\prime}\left(w_{k_1+1}^{*T} x_{n}\right)x_{n,1}& \cdots & v_{k_1+1}^* \sigma^{\prime}\left(w_{k_1+1}^{*T} x_{n}\right)x_{n,k_2}\end{array}\right),
\end{equation}
}

Furthermore, $J(\tilde{w}^*;v^*,\tilde{w}^{\prime *})\in \mathbb{R}^{n \times n}$ is nothing but 
a $n\times n$ submatrix of $J(w^*;v^*)\in \mathbb{R}^{n \times md}$. 
Now that $md \geq 2n >n$, $\det (J(\tilde{w}^*;v^*,\tilde{w}^{\prime *}))\neq 0$ implies the full-row-rankness of $J(w^*;v^*)$ (w.p.1.). Thus the whole proof is completed.

\section{Proof of Theorem \ref{thm3}}
\label{appendix:thm3}

\BLACK{In this section, we provide both proof sketch and the detailed proof of Theorem \ref{thm3}. They can be seen in Appendix \ref{appendix:thm3sketch} and Appendix \ref{appendix:thm3detail}, respectively.  For general readers, reading proof sketch in Appendix \ref{appendix:thm3sketch} will help grasp our main idea.}

\subsection{Proof Sketch of Theorem \ref{thm3}}
\label{appendix:thm3sketch}
\begin{proof}[Proof sketch] 
  The proof is built on the special structure of neural network $f(x;\theta)$, including the linear dependence of $v$ and the mirrored pattern of parameters. 
  Here, we describe our high level idea and the analysis roadmap. 
  The proof  consists of proving the following claims: 
  \begin{itemize} \vspace{-2mm}
    \item[(I)] every KKT point $\theta^*$ satisfies $\|\nabla_w \ell(w^*;v^*)\|=O(\epsilon)$; \vspace{-1mm}
    \item[(II)] the gradient of $w$ always dominates the error term, i.e. $\|\nabla_w \ell (w;v)\|_2^2 =\Omega \left(\ell(w^*;v^*)\right) $, so we have $\ell(w^*;v^*) = O(\epsilon^2)$.   \vspace{-2mm}
  \end{itemize}
  To prove claim (I), 
  we only need to consider the case where $w^*$ is on the boundary and $\|\nabla_w\ell(w^*;v^*)\|_2\neq 0$ (otherwise Theorem \ref{thm3} automatically holds based on Theorem \ref{thm1}). In this case, by the optimality condition, taking $\eta \in \mathbb{R}$ as a small step size, we have 
  \begin{equation} 
  \label{thm3:gradientdirection}
    -\eta\nabla_w\ell(w^*;v^*)= \eta 
  \frac{\|\nabla_w\ell(w^*;v^*)\|}{\|w^*-w^0\|}(w^*-w^0)=\tilde{\eta}(w^*-w^0),
  \end{equation}
  where $\tilde{\eta}=\eta 
  \frac{\|\nabla_w\ell(w^*;v^*)\|}{\|w^*-w^0\|}$.  Now, our key observation is that, after moving along (\ref{thm3:gradientdirection}) from $w^*$,  the change of the loss is not significant due to the special local structure of $f(w;v^*)$, therefore, $\|\nabla_w\ell(w^*;v^*)\|_2$ can be bounded.
  To be more specific, 
  we denote $f^*$ as the neural network output at the KKT point $\theta^*$; denote $\bar{f}$ as the neural network output after taking a small step $\eta$ along the negative partial gradient direction of $w$; and denote $f^{\prime}$ as an rough estimate of $\bar{f}$:
    \begin{eqnarray}
    f^{*} &:=& f\left(w^*;v^*\right)=J(w^0;v^*)w^*+R^*, \label{thm3:f*} \\
    \bar{f}&:=&f\left(w^*-\eta \nabla_w \ell(w^*;v^*);v^*\right)=(1+\tilde{\eta}) f^{*}-(1+\tilde{\eta}) R^{*}+\bar{R}, \label{thm3:fnew} \\
    f^{\prime} &:=& \bar{f}_{\text{lin}} +(1+\tilde{\eta})R^* =(1+\tilde{\eta}) f^*, \label{thm3:fprime}
  \end{eqnarray}

  where $\bar{f}_{\text{lin}}$ is the first-order Taylor approximation of $\bar{f}$, $R^*$ is the second-order Taylor residue of $f^*$ (similarly for $\bar{R}$). Additionally, \eqref{thm3:f*}, \eqref{thm3:fnew}, and \eqref{thm3:fprime} are due to the the symmetric property of $w^0$, $v^0$ and $v^*$, so the bias terms in the Taylor expansion will vanish.  Now, we compare the value of the loss on each of $f^*$, $\bar{f}$ and $f^\prime$. Define $\ell\circ f = \frac{1}{2}\|y-f\|_2^2$, 
  we prove the following crucial relationship:
  \begin{equation} \label{thm3:relationship}
    \eta \|\nabla_w \ell(w^*;v^*)\|_2^2 \overset{(a)}{\leq} \ell \circ f^* -\ell \circ \bar{f}  \overset{(b)}{\leq} \ell \circ f^{\prime}  - \ell \circ \bar{f}  = \frac{1}{2}\left\|f^{\prime}-\bar{f}\right\|_{2}\left\|f^{\prime}+\bar{f}-2 y\right\|_{2} \overset{(c)}{=} O( \tilde{\eta} \epsilon^2)
  \end{equation}  
  Eq. (\ref{thm3:relationship}) plays a key role in our analysis. 
  Here, $(a)$ can be easily shown by applying Descent lemma in this local region, yet $(b)$ and $(c)$ are not that obvious. 
  Recall in \eqref{thm3:f*}, \eqref{thm3:fnew}, and \eqref{thm3:fprime}, we know that: (i) although the location of $\bar{f}$ is unclear, $f^{\prime}$ points at the same direction as $f^*$. (ii) As an estimator of $\bar{f}$, $f^\prime$ is not far away from it, i.e. $\|\bar{f}-f^\prime\|_2=\|\bar{R}-(1+\tilde{\eta}) R^{*}\|_2 $ only involves the second-order Taylor residue terms. With this observation, 
 $(b)$ is proved in  Lemma \ref{lemma:increaseloss} (stated below) by geometric properties, and $(c)$ is calculated in Lemma \ref{lemma:estimateerror} (stated below), so the relationship (\ref{thm3:relationship}) is proved. Therefore, we have $\theta^*$ satisfies $\|\nabla_w \ell(w^*;v^*)\|=O(\epsilon)$ by plugging in $\eta=\frac{\epsilon \tilde{\eta}}{\|\nabla_w \ell(w^*;v^*)\|_2}$, and claim (I) is proved.
  \begin{lemma}\label{lemma:increaseloss}
    Under the conditions of Theorem \ref{thm3}, we have $\ell \circ f^{\prime}\geq \ell \circ f^{*}$, i.e., $\left\|f^{\prime}-y\right\|_{2}^{2} \geq\left\|f^{*}-y\right\|_{2}^{2}$.
  \end{lemma}

  \begin{lemma}\label{lemma:estimateerror}
    Under the conditions of Theorem \ref{thm3}, we have:
      $\left\|f^{\prime}-\bar{f}\right\|_2\left\|f^{\prime}+\bar{f}-2y\right\|_2 = O(\epsilon^2)$.
  \end{lemma}
  As for claim (II), it is true as long as $J(w;v)$ is of full row rank, which has been shown in Theorem \ref{thm1}. A more detailed proof of Lemma \ref{lemma:increaseloss} and \ref{lemma:estimateerror}, as well as the proof of the whole Theorem \ref{thm3} are in Appendix \ref{appendix:thm3detail}.

  \begin{figure}[htbp]
    \centering 
    \includegraphics[width=4in]{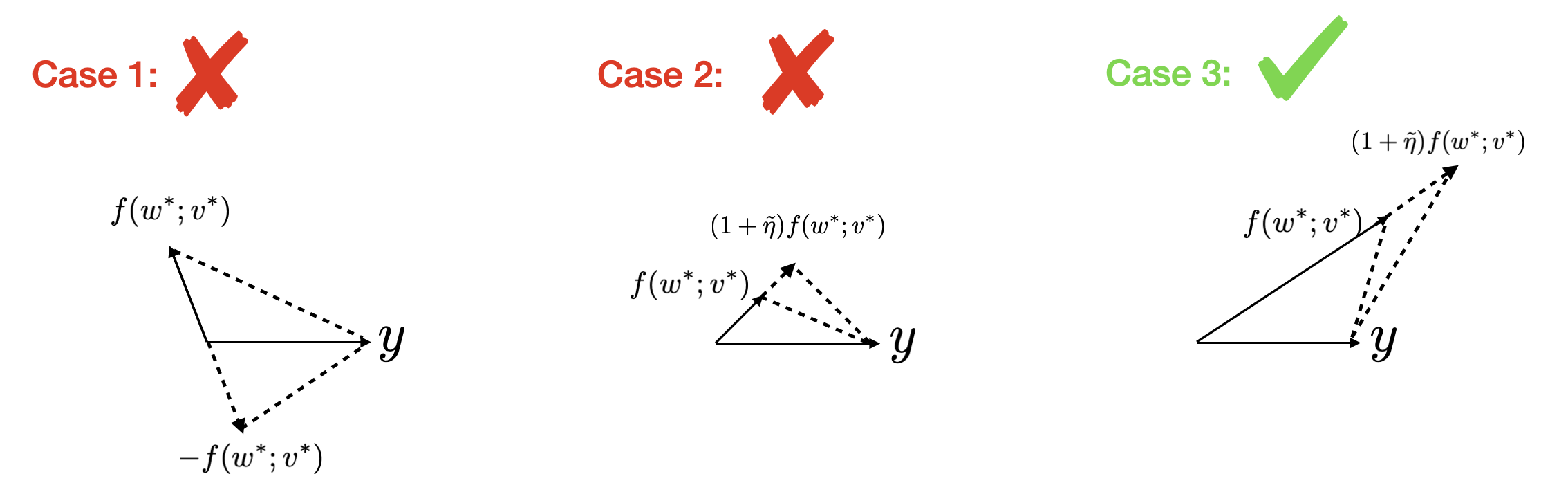}
    \caption{ Geometrical illustration of three cases in Lemma \ref{lemma:increaseloss}: comparing with $f^*$, $f^{\prime}= (1+\tilde{\eta})f^*$ will not further reduce the distance to $y$.}\label{thm3figure}
  \end{figure}

  {\it Proof sketch of Lemma \ref{lemma:increaseloss}.} Here, we provide a proof sketch of Lemma \ref{lemma:increaseloss}, we need to discuss the following cases:
  \begin{itemize} %
    \item[(1)] 
    When $f^{*T}y\geq 0$: we prove by contradiction. Since $f^{*}$ is linear in $v^*$, 
     $\| (1+\tilde{\eta})f^{*} -y\|_2^2 < \| f^{*} -y\|_2^2 $ implies $\| f(w^*;(1+\tilde{\eta})v^*) -y\|_2^2 < \| f^{*} -y\|_2^2$,   
    which means we can further reduce the loss by %
    changing
    $v^*$ to $(1+\tilde{\eta})v^*$, which is still feasible, we have a contradiction to the assumption that $(w,^*v^*$) is a KKT point (see Figure \ref{thm3figure}, Middle). 

    \item[(2)] When $f^{*T}y< 0$: changing $v^*$ to $-v^*$ will further reduce the distance to $y$ (see Figure \ref{thm3figure}, Left), this is a contradiction to the fact that $v^*$ is a KKT point, so case (2) will not happen.
    \end{itemize} 
    In conclusion, we always have $\ell \circ f^{\prime}\geq \ell \circ f^*$ (see Figure \ref{thm3figure}, Right).
\end{proof}

\subsection{Detailed Proof of Theorem \ref{thm3}}
\label{appendix:thm3detail}
The proof of Theorem \ref{thm3} consists of proving the following claims: under the setting of Theorem \ref{thm3},

  \begin{itemize}
    \item[(I)] every KKT point $\theta^*$ satisfies $\|\nabla_w \ell(w^*;v^*)\|=O(\epsilon)$.
    \item[(II)] the gradient of $w$ always dominates the error term, i.e. $\|\nabla_w \ell (w;v)\|_2^2 =\Omega \left(l(w^*;v^*)\right) $, so we have $\ell(w^*;v^*) = O(\epsilon^2)$.   
  \end{itemize}
  
  Note that the statement of claim (I) is not precise, there are chances that the constant terms will exponentially grow (will be discussed later). Nevertheless, it is just an intermediate result that helps provide a clearer big picture. our final result in claim (II) will be precise.
  
To prove claim (I), 
we only need to consider the case when $w^*$ is on the boundary of the constraint $B_{\epsilon}(w^0)$ (Theorem \ref{thm3} automatically holds when $w^*$ is in the interior of $B_{\epsilon}(w^0)$).

 Now, suppose $w^*$ is a non-zero-gradient KKT point on the boundary, by the optimality condition, its negative gradient direction should be along the same direction as $w^*-w^0$, therefore, if we further take a small step $\eta$ along the negative gradient direction, we have:

\begin{equation}\label{gradientdirection}
 -\eta\nabla_w \ell (w^*;v^*)= \eta\frac{\|\nabla_w \ell (w^*;v^*)\|_2}{\|w^*-w^0\|_2 }(w^*-w^0)  =\tilde{\eta}(w^*-w^0), 
\end{equation}

where $\tilde{\eta}:= \eta\frac{\|\nabla_w \ell (w^*;v^*)\|_2}{\|w^*-w^0\|_2 } = \eta \frac{\|\nabla_w \ell (w^*;v^*)\|_2}{\epsilon }$.
In this case, the loss function will decrease when we further move $w$ along $-\nabla_w \ell (w^*;v^*)$ with a sufficiently small stepsize $\eta$. In other words, we can apply Descent lemma (details can be seen in Bertsekas et al. \cite{bertsekas1997nonlinear}) in this local region, i.e.

\begin{eqnarray}\label{descentproperty}
  \ell \left(w^*-\eta \nabla_w \ell(w^*;v^*);v^*\right)-\ell (w^*;v^*) &\leq & 
  -\eta\|\nabla_w  \ell(w^*;v^*)\|_2^2.
\end{eqnarray} 

As a matter of fact, after further taking a small GD step at $w^*$, $f\left(w^*-\eta \nabla_w \ell(w^*;v^*);v^*\right)$ will be closer to the groundtruth $y$, we will come back to this fact later, it will be used to bound $\|\nabla_w  \ell(w^*;v^*)\|_2$. 

Now, for any $w \in B_{\epsilon}(w^0)$, we take the Taylor expansion of $f(w;v^*)$ at $w^0$:
\begin{eqnarray}
  f(w;v^*)&=&f(w^0;v^*)+J(w^0;v^*)(w-w^0)+R(w) 
  \label{taylor} \\ 
  &\overset{\text{(a) \& (b)}}{=}&J(w^0;v^*)w+R(w),   \label{taylor2}\\
  f_{\text{lin}}(w;v^*)&:=&f(w^0;v^*)+J(w^0;v^*)(w-w^0) \label{fhat}  \\
  &\overset{\text{(a) \& (b)}}{=}&  J(w^0;v^*)w,     \label{fhat2} 
\end{eqnarray}

where $R(w) \in \mathbb{R}^n$ is the residue term of the Taylor expansion:
\begin{equation}\label{residue}
  [R(w)]_i=\int_{0}^1 (w-w^0)^TH_i(w^0+t(w-w^0))(w-w^0)(1-t)d t,  \quad i=1,\cdots, n,
\end{equation}

where $H_i(w)$ is the Hessian matrix of $f(w;v^*,x_i)$ at $w$ (for simplicity, we drop the dependence of $v^*$ and $x_i$ in the notation), and $f_{\text{lin}}(w;v^*)$ is a linear approximation of $f(w;v^*)$. In addition, 
(a) \& (b) is due to the fact that $w_{j}^{0}$ follows the mirrored LeCun's initialization with  $(w^0_{\frac{m}{2}+1},\dots,w^0_{m})=(w^0_{1},\dots,w^0_{\frac{m}{2}})$, recall the construction of $f(w;v)$ in (\ref{pairNN}), the hidden output will cancel out with the outer weight, so we have (a): 

\begin{equation} \label{fw0v}
  f(w^0;v^*)\overset{(\ref{pairNN})}{=}\sum_{j=1}^{\frac{m}{2}} v^*_{j} \left( \sigma(w_{j}^{0T} x)-\sigma(w_{j+\frac{m}{2}}^{0T} x)  \right)=0.
\end{equation}

Similarly, we have (b): 

\begin{equation}\label{Jw0w0}
  J(w^0;v^*)w^0=\nabla_w f\left(w^{0} ; x, v^*\right)^{T}w^0=\sum_{j=1}^{\frac{m}{2}} v_j^*\left(\sigma^{\prime}(w_j^{0T}x)w_j^{0T}x - \sigma^{\prime}(w_{j+\frac{m}{2}}^{0T}x)w_{j+\frac{m}{2}}^{0T}x\right)=0.
\end{equation}

Based on (\ref{taylor2}), we define the following quantities: 

\begin{eqnarray}
  \bar{f}&:=&f\left(w^*-\eta \nabla_w \ell(w^*;v^*);v^*\right)\overset{(\ref{taylor2})}{=} J(w^0;v^*)(w^*-\eta \nabla_w \ell(w^*;v^*))+\bar{R}, \label{fnew} \\
  f^{*} &:=& f\left(w^*;v^*\right)\overset{(\ref{taylor2})}{=} J(w^0;v^*)w^*+R^* \label{f*},
\end{eqnarray}

where $\bar{R}=R\left(w^*-\eta \nabla_w \ell(w^*;v^*)\right)$, $R^*=R\left(w^*\right)$ as it is introduced in (\ref{residue}). 
Additionally,  $\bar{f}$ can be re-written in the form of $f^{*}$: 
\begin{eqnarray}
  \bar{f}&\overset{(\ref{fnew})}{=}&J(w^0;v^*)(w^*-\eta \nabla_w \ell(w^*;v^*))+\bar{R} \\ 
  &= & J(w^0;v^*) w^* - \eta J(w^0;v^*)\nabla_w\ell(w^*;v^*)  +\bar{R} \\
  &\overset{(\ref{f*})}{=}& f^{*} -R^* -\eta J(w^0;v^*)\nabla_w\ell(w^*;v^*) +\bar{R} \\
  &\overset{(\ref{gradientdirection})}{=}& f^{*} -R^* +\tilde{\eta}J(w^0;v^*)(w^*-w^0) +\bar{R}\\
  &\overset{(\ref{Jw0w0})}{=}&f^{*} -R^* +\tilde{\eta}J(w^0;v^*)(w^*) +\bar{R} \\
  &\overset{(\ref{f*})}{=}&f^* -R^* +\tilde{\eta}(f^{*}- R^*) +\bar{R} \\
  &=& (1+\tilde{\eta}) f^{*} -(1+\tilde{\eta})R^* +\bar{R}.
\end{eqnarray}

Now, we construct a rough estimator of $\bar{f}$ by merely adding a residue term $(1+\tilde{\eta}) R^*$ on $\bar{f}_{\text{lin}}$, which is $f^{\prime}$ defined as follows: 

\begin{eqnarray}
  f^{\prime} &:=& \bar{f}_{\text{lin}}+(1+\tilde{\eta}) R^* \\
  &=&f_{\text{lin}}\left(w^*-\eta \nabla_w \ell(w^*;v^*);v^*\right)  +(1+\tilde{\eta})R^*\\
  &\overset{\eqref{fhat2}}{=}& J(w^0;v^*)(w^*-\eta\nabla_w\ell(w^*;v^*)) +(1+\tilde{\eta})R^*\\
  &=& J(w^0;v^*) w^* - \eta J(w^0;v^*)\nabla_w\ell(w^*;v^*) +(1+\tilde{\eta})R^*\\
  &\overset{(\ref{f*})}{=}& f^{*} -R^* +\eta J(w^0;v^*)\nabla_w\ell(w^*;v^*) +(1+\tilde{\eta})R^* \\
  &\overset{(\ref{gradientdirection})}{=}& f^{*} -R^* +\tilde{\eta}J(w^0;v^*)(w^*-w^0) +(1+\tilde{\eta})R^* \\
  &\overset{(\ref{Jw0w0})}{=}&f^{*} -R^* +\tilde{\eta}J(w^0;v^*)(w^*) +(1+\tilde{\eta})R^* \\
  &\overset{(\ref{f*})}{=}&f(w^*;v^*) -R^* +\tilde{\eta}(f^{*}- R^*) +(1+\tilde{\eta})R^* \\
  &=& (1+\tilde{\eta}) f^{*}. \label{fprime}
\end{eqnarray}

According to the descent property in (\ref{descentproperty}), after taking a very small GD step at $w^*$ , the new loss function $\ell\circ \bar{f}$ will be smaller than the old one $\ell \circ f^*$, where $\ell\circ f = \frac{1}{2}\|y-f\|_2^2$. In contrast, we discuss the change of the loss function from $f^*$ to that of the rough estimator $f^{\prime}$. The following Lemma \ref{lemma:increaseloss} shows that $\ell \circ f^{\prime} \geq \ell \circ f^*$, different from the fact that $\ell \circ \bar{f} \leq \ell \circ f^*$.

\setcounter{lemma}{0}
\begin{lemma}\label{lemma2:increaseloss}[Corresponding to Lemma \ref{lemma:increaseloss} in the proof sketch.]
  Under the background of Theorem \ref{thm3} and the definition of  $f^{\prime}$ \& $f^*$ in (\ref{fprime}) \& (\ref{f*}), we have $\ell \circ f^{\prime} \geq \ell \circ f^*$, i.e., $\left\|f^{\prime}-y\right\|_2^2\geq \left\|f^*-y\right\|_2^2$.
\end{lemma}

The proof of Lemma \ref{lemma2:increaseloss} can be seen in Appendix \ref{appendix:lemma:increaseloss}. 

Now, we have 
\begin{eqnarray}
  \eta \left\|\nabla_w \ell(w^*;v^*)\right\|_2^2 &\overset{(\ref{descentproperty})}{\leq}& \ell(w^*;v^*) - \ell \left(w^*-\eta\nabla_w \ell(w^*;v^*);v^*\right) \\
  & =&  \frac{1}{2}\left\|f^*-y\right\|_2^2- \frac{1}{2}\left\|\bar{f}-y\right\|_2^2  \\
  & \overset{\text{Lemma}\ref{lemma:increaseloss}}{\leq} & \frac{1}{2} \left\|f^{\prime}-y\right\|_2^2 - \frac{1}{2}\left\|\bar{f}-y\right\|_2^2   \\ 
  & =& \frac{1}{2} \left\|f^{\prime}-\bar{f}\right\|_2\left\|f^{\prime}+\bar{f}-2y\right\|_2.  \label{diffofsquare}
\end{eqnarray}

Next, we bound $\left\|f^{\prime}-\bar{f}\right\|_2\left\|f^{\prime}+\bar{f}-2y\right\|_2$ using the following Lemma \ref{lemma2:estimateerror}. %

\begin{lemma}\label{lemma2:estimateerror}[Corresponding to Lemma \ref{lemma:estimateerror} in the proof sketch.]
  Under the background of Theorem \ref{thm3} and the definition of  $f^{\prime}$ \& $\bar{f}$ in (\ref{fprime}) \& (\ref{fnew}), we have:
  \begin{equation} \label{diffofsquare3}
    \left\|f^{\prime}-\bar{f}\right\|_2\left\|f^{\prime}+\bar{f}-2y\right\|_2 \leq \left(\tilde{\eta} \epsilon^2   \lambda_{[0:2]} \sqrt{20n}\right) \left(\sqrt{n} C_y+3 \left(n \left(\frac{m}{2}L \zeta\epsilon \right)^2 + n C_y^2 \right)^{\frac{1}{2}} \right),
  \end{equation}
  where $\zeta$ is the constraint for $v$ required in $B(v)$: $v \geq \zeta {\bf 1}$; $\lambda_{[a:b]}:=\underset{i}{\max} \{\lambda_{i[a:b]} | i=1,\cdots,n \}$, and each $\lambda_{i[a:b]}$ is the maximum eigenvalue of the Hessian matrices $H_i(w^0+t(w^*-w^0)):=\nabla_w^2 f(w^0+t(w^*-w^0);v^*,x_i)$ in the interval $t\in [a,b]$.
\end{lemma}
The proof of Lemma \ref{lemma2:estimateerror} can be seen in Appendix \ref{appendix:lemma:estimateerror}.
Now, with the help of Lemma \ref{lemma2:estimateerror}, we have:
{\small
\begin{eqnarray}
  \eta \left\|\nabla_w \ell(w^*;v^*)\right\|_2^2 
  & \overset{(\ref{diffofsquare})}{\leq}& \frac{1}{2} \left\|f^{\prime}-\bar{f}\right\|_2\left\|f^{\prime}+\bar{f}-2y\right\|_2  \\
  & \overset{(\ref{diffofsquare3})}{\leq} &\frac{1}{2}\left(\tilde{\eta} \epsilon^2   \lambda_{[0:2]} \sqrt{20n}\right) \left(\sqrt{n} C_y+3 \left(n \left(\frac{m}{2}L \zeta\epsilon \right)^2 + n C_y^2 \right)^{\frac{1}{2}} \right).
\end{eqnarray} }
Recall $\eta=\frac{\epsilon}{\|\nabla_w \ell(w^*;v^*)\|_2}\tilde{\eta}$ and $\tilde{\eta}$ is sufficiently small, we have:
\begin{equation}  \label{epsilongradient}
  \left\|\nabla_w \ell(w^*;v^*)\right\|_2 \leq   \frac{1}{2} \epsilon\left(   \lambda_{[0:2]} \sqrt{20n}\right) \left(\sqrt{n} C_y+3 \left(n \left(\frac{m}{2}L \zeta\epsilon \right)^2 + n C_y^2 \right)^{\frac{1}{2}} \right). 
\end{equation}
Since $\zeta$ can be chosen arbitrarily small, we can choose it to be smaller than $\frac{1}{\epsilon}$. That is to say, $\|\nabla_w\ell (w^*;v^*)\|_2 =O(\epsilon)$, so the claim (I) is proved. Note that to be precise, the constant on the right hand side of the above inequality  depends on $\lambda_{[0:2]}$, which is a linear function of $v$, and the latter vector may not be upper bounded. %
Nevertheless, claim (I) is just an intermediate result, and our subsequent derivation will directly use the right hand side of  \eqref{epsilongradient}, which is precise.

Now, we build the relationship between $\left\|\nabla_w \ell(w^*;v^*)\right\|_2$ and the loss function. Here, we need to eliminate the dependence of $v$ in the final result: since there is no uniform upper bound for $v$, it can potentially make  $\lambda_{[0:2]}$ grow exponentially.
To alleviate this issue, we utilize the fact that $f(x;\theta)$ is linear in $v$, so  $\lambda_{[0:2]}$ is also linear in $v$, and then we manage to remove the dependence of $v$ in our final result.
Specifically, let us define \[v_{\text{min}}^*:=\underset{j}{\min}\{v_j^* \mid j=1,\cdots, m \}, \quad  v_{\text{max}}^*:=\underset{j}{\max}\{v_j^* \mid j=1,\cdots, m \}.\] 

So we have
\begin{eqnarray}
  \left\|\nabla_w \ell(w^*;v^*)\right\|_2^2 &=& \|J(w^*;v^*)^T(y-f^*)\|_2^2  \overset{(c)}{\geq} v_{\text{min}}^{*2} \cdot \| \tilde{J}(w^*;v^*)^T(y-f^*)\|_2^2 \\
  &\geq& v_{\text{min}}^{*2} \cdot \tilde{\lambda}_{\text{min}}^{*2} \cdot \|y-f^*\|^2_2,
\end{eqnarray}
where  $\tilde{\lambda}_{\text{min}}^*$ is the smallest singular value of $\tilde{J}(w^*;v^*)$, and $\tilde{J}(w^*;v^*)$ is: 
$$\tilde{J}(w^*;v^*):=\left(\begin{array}{ccc} \sigma^{\prime}\left(w_{1}^{*T} x_{1}\right) x_{1}^T & \cdots &  \sigma^{\prime}\left(w_{m}^{*T} x_{1}\right) x_{1}^T \\ &\vdots & \\ \sigma^{\prime}\left(w_{1}^{*T} x_{n}\right) x_{n}^T & \cdots & \sigma^{\prime}\left(w_{m}^{*T} x_{n}\right) x_{n}^T\end{array}\right)\in \mathbb{R}^{n \times md}.$$
Here, $(c)$ is straightforward because $\tilde{J}(w^*;v^*)$ is just the simplified version of $J(w^*;v^*)$ by removing all the coefficient $v_j^*$;  furthermore, $\tilde{J}(w^*;v^*)$ is full row rank because (i) $J(w^*;v^*)$ is proved to be full rank in the second part of Theorem \ref{thm1} in Appendix \ref{appendix:thm2}, (ii) $v^*$ is entry-wise non-zero, so $\tilde{\lambda}_{\text{min}}^{*}$ is strictly positive.

Now, we need to remove the dependence of $v$ in the right hand side of (\ref{epsilongradient}). Similarly as before,  $\lambda_{[a:b]}$ can also be bounded by $v_{max}^*\tilde{\lambda}_{[a:b]}$, where $\tilde{\lambda}_{[a:b]}$ is equal to $\underset{i}{\max} \{\tilde{\lambda}_{i[a:b]} | i=1,\cdots,n \}$, and each $\tilde{\lambda}_{i[a:b]}$ is the maximum eigenvalue of the Hessian matrices $\tilde{H}_i(w^0+t(w^*-w^0)):=\nabla_w^2 \tilde{f}(w^0+t(w^*-w^0);,x_i)$ in the interval $t\in [a,b]$, and $\tilde{f}(w;,x_i):=\sum_{j=1}^m \sigma(x_i^T w_j)$ is the 
simplified version of $f(w;v,x_i)$ by removing all the coefficient $v_j$. In conclusion, we have
{\small
\begin{eqnarray}
  v_{\text{min}}^{*} \tilde{\lambda}_{\text{min}}^{*} \|y-f^*\|_2 &\leq& \left\|\nabla_w \ell(w^*;v^*)\right\|_2  \\
  &\leq&   
  \frac{1}{2} \epsilon\left(   \lambda_{[0:2]} \sqrt{20n}\right) \left(\sqrt{n} C_y+3 \left(n \left(\frac{m}{2}L \zeta\epsilon \right)^2 + n C_y^2 \right)^{\frac{1}{2}} \right) \\
  &\leq&
  \frac{1}{2} \epsilon\left(  v_{\text{max}}^{*} \tilde{\lambda}_{[0:2]} \sqrt{20n}\right) \left(\sqrt{n} C_y+3 \left(n \left(\frac{m}{2}L \zeta\epsilon \right)^2 + n C_y^2 \right)^{\frac{1}{2}} \right). \label{finalboundinthm3}
\end{eqnarray}
}
Rearrange and take the square on both sides, we get $ \|y-f^*\|_2^2 =  O \left(\kappa^2 \epsilon^2 \right) $, where $\kappa$ is the finite constant in the constraint $B(v)$ of problem \eqref{constrained}. So $\ell (\theta^*) =O(\epsilon^2)$, the proof of Theorem \ref{thm3} is completed.

\subsection{Proof of Lemma \ref{lemma:increaseloss}}
\label{appendix:lemma:increaseloss}

  To prove Lemma \ref{lemma:increaseloss}, we need to discuss the following cases:

  \begin{itemize}
    \item[(i)] When $f^{*T}y\geq 0$,  we prove Lemma \ref{lemma:increaseloss} by contradiction: when $\tilde{\eta}$ is sufficiently small, suppose $\| (1+\tilde{\eta})f^{*} -y\|_2^2 < \| f^{*} -y\|_2^2 $, then $(w^*,v^*)$ is not a KKT point (this case corresponds to the Figure \ref{thm3figure} (Middle)). 
    Note that in problem (\ref{constrained}), $f^{*}=\sum_{j=1}^{\frac{m}{2}} v_j^*\left(\sigma(w_j^{*T}x)-\sigma(w_{j+\frac{m}{2}}^{*T}x) \right)$ is linear in $(v_1^*, \cdots, v_{\frac{m}{2}}^*)$, so $(1+\tilde{\eta})f^{*}=f(w^*;(1+\tilde{\eta})v^*)$. 
    That is to say, $\| (1+\tilde{\eta})f^{*} -y\|_2^2 < \| f^{*} -y\|_2^2 $ implies $\| f(w^*;(1+\tilde{\eta})v^*) -y\|_2^2 < \| f^{*} -y\|_2^2$,   
    which means we can further reduce the loss by changing $v^* \rightarrow (1+\tilde{\eta})v^*$. Since $(1+\tilde{\eta})v^*$ is still feasible if $v^*$ is feasible, we have a contradiction to the assumption that $(w,^*v^*$) is a KKT point. 

   Therefore, in case (i), moving from $f^*$ to $f^{\prime}=(1+\tilde{\eta})f^*$ will not further reduce the loss (this case corresponds to the Figure \ref{thm3figure} (Right)). In other words, we always have $f^{*T}(f^*-y) \geq 0$, this property will also be used in Lemma \ref{lemma2:estimateerror}.
    \item[(ii)] When $f^{*T}y< 0$, we have 
    
     \begin{eqnarray}
       \| f^{*} -y\|_2^2 &=&  \|\Phi(w^*)v^*-y\|_2^2\\
      &=& \|\Phi(w^*)v^*\|_2^2+\| y\|_2^2 -2(\Phi(w^*)v)^{T}y \\
       &>& \|f^{*}\|_2^2+\| y\|_2^2 -2(-\Phi(w^*)v)^{T}y \\
       &=& \| \Phi(w^*)(-v^*)-y\|_2^2  \\
       &=& \|-f^*-y\|_2^2,
     \end{eqnarray}
     where 
     $$
    \Phi(w):=\left[\begin{array}{c} \sigma\left(w_{1}^{T} x_1\right),\dots,\sigma\left(w_{m}^{T} x_1\right)\\ \vdots \\  \sigma\left(w_{1}^{T} x_n\right),\dots,\sigma\left(w_{m}^{T} x_n\right) \end{array}\right]\in \mathbb{R}^{n \times m}.$$

     Therefore, changing $v^*$ to $-v^*$ will further reduce the loss function (see Figure \ref{thm3figure} (Left)). Since $-v^*$ is feasible if $v^*$ is feasible, this is a contradiction to the assumption that $(w^*,v^*)$ is a KKT point. That is to say, we always have case (i): $f^{*T}y\geq 0$.
    \end{itemize}
    In conclusion, we always have $\ell \circ f^{\prime}\geq \ell \circ f^*$, so the proof is completed.

    \subsection{Proof of Lemma \ref{lemma:estimateerror}}
    \label{appendix:lemma:estimateerror}

    Similarly with the residue term (\ref{residue}), we define $R \in \mathbb{R}^n$ with each component satisfying $[R]_i:= \int_{0}^{1+\tilde{\eta}} (w-w^0)^TH_i(w^0+t(w-w^0))(w-w^0)(1-t)d t$, we have:
    {\small
    \begin{eqnarray}
      \left\|f^{\prime}-\bar{f}\right\|_2^2& \overset{(\ref{fprime}) \& (\ref{fnew})}{=} & \left\|-(1+\tilde{\eta})R^* +\bar{R}\right\|_2^2 \\
      & =& \left\|(1+\tilde{\eta})R^* -\bar{R}\right\|_2^2 \\
      & \overset{(*)}{=}& \left\|(1+\tilde{\eta})R^* - (1+\tilde{\eta})^2 R \right\|_2^2 \\
      &=&\left\| (1+\tilde{\eta})(R^*-R)-(1+\tilde{\eta})\tilde{\eta} R \right\|_2^2\\
      &\leq& \!\!\!\!\!\!\!\!\!\!\!\!(1+\tilde{\eta})^2\sum_{i=1}^n\left(\int_{1}^{1+\tilde{\eta}}  (w-w^0)^TH_i(w^0+t(w-w^0))(w-w^0)(1-t)d t \right)^2  \nonumber\\
      &&\!\!\!\!\!\!\!\!\!\!\!\!+ (1+\tilde{\eta})^2\tilde{\eta}^2 \sum_{i=1}^n \left(\int_{0}^{1+\tilde{\eta}} (w-w^0)^TH_i(w^0+t(w-w^0))(w-w^0)(1-t)d t\right)^2  \nonumber\\
      &\leq & \!\!\!\!\!\!\!\!\!\!\!\!(1+\tilde{\eta})^2 \sum_{i=1}^n \left(\lambda_{i[1:1+\tilde{\eta}]}\epsilon^2\int_{1}^{1+\tilde{\eta}}d t \right)^2 + (1+\tilde{\eta})^2\tilde{\eta}^2 \sum_{i=1}^n\left(\lambda_{i[0:1+\tilde{\eta}]} \epsilon^2 \int_{0}^{1+\tilde{\eta}}d t \right)^2 \nonumber\\
      &=&  n(1+\tilde{\eta})^2 \tilde{\eta}^2 \lambda_{[1:1+\tilde{\eta}]}^2\epsilon^4 + n(1+\tilde{\eta})^4\tilde{\eta}^2 \lambda_{[0:1+\tilde{\eta}]}^2\epsilon^4   \\
      &=& n\tilde{\eta}^2 \epsilon^4 \left( (1+\tilde{\eta})^2  \lambda_{[1:1+\tilde{\eta}]}^2 + (1+\tilde{\eta})^4 \lambda_{[0:1+\tilde{\eta}]}^2\right) \\
      &=& n\tilde{\eta}^2 \epsilon^4 \left( 4  \lambda_{[1:1+\tilde{\eta}]}^2 + 16 \lambda_{[0:1+\tilde{\eta}]}^2\right) \\
      &=& n\tilde{\eta}^2 \epsilon^4 \left( 20 \lambda_{[0:2]}^2\right), \label{estimateerrorsquare}
    \end{eqnarray}
    }

    where the last two inequalities is because of the fact that $\tilde{\eta}\leq 1$ is sufficiently small and $\lambda_{[1:1+\tilde{\eta}]} \leq \lambda_{[0:1+\tilde{\eta}]} \leq \lambda_{[0:2]}$. (*) is due to: for $i=1, \cdots, n$: 
    {\tiny
    \begin{eqnarray}
      [\bar{R}]_i&=& \int_{0}^1\! \!\!\!\! \left(w\!\!-\!\!\eta \nabla_w\ell(w^*;v^*)\!\!-\!\!w^0\right)^T \!\!\!\!H_i\!\!\left(w^0\!\!+\!\!t(w\!\!-\!\!\eta \nabla_w\ell(w^*;v^*)\!\!-\!\!w^0)\right)\! \!\! \left(w-\eta \nabla_w\ell(w^*;v^*)-w^0\right)\!\!(1-t)d t \\
      &\overset{(\ref{gradientdirection})}{=}& (1+\tilde{\eta})^2\int_{0}^1 \left(w-w^0\right)^TH_i\left(w^0+(1+\tilde{\eta})t(w-w^0)\right) \left(w-w^0\right)(1-t)d t \\
      &=& (1+\tilde{\eta})^2\int_{0}^{1+\tilde{\eta}} \left(w-w^0\right)^TH_i\left(w^0+t(w-w^0)\right) \left(w-w^0\right)(1-t)d t \\
      &=& (1+\tilde{\eta})^2 [R]_i.
    \end{eqnarray}}
  Now, we bound $\left\|f^{\prime}+\bar{f}-2y\right\|_2$, since $(w^*,v^*)$ is a KKT point, it is proved in Lemma \ref{lemma2:increaseloss} in Appendix \ref{appendix:lemma:increaseloss} that $f^{*T}y \geq 0 $, furthermore, at $w=w^*$, the loss function at $(w^*,v^*)$ should be less or equal to all other feasible points $(w^*,v)$, including $v=\zeta {\bf 1}$, i.e.,
  \begin{eqnarray}
    \|f^*-y\|_2^2&=& \|f(w^*;v^*)-y\|_2^2 \\
    &\leq& \|f(w^*;v^*)\|^2+\|y\|_2^2 \\
    &\leq& \|f(w^*;\zeta {\bf 1})\|^2+\|y\|_2^2 \\
    &\overset{(**)}{\leq}&  n \left(\frac{m}{2}L \zeta\epsilon \right)^2 + \|y\|_2^2 \\
    &\leq&  n \left(\frac{m}{2}L \zeta\epsilon \right)^2 + nC_y^2, 
    \label{f*-y}
  \end{eqnarray}
  where the last inequality is because of Assumption \ref{assum3} (each $y_i \leq C_y$), and (**): is due to
\begin{eqnarray}
  \|f(w;\zeta{\bf 1})\|_2^2 &\overset{(\ref{pairNN})}{=}& \sum_{i=1}^n\left(\sum_{j=1}^{\frac{m}{2}} \zeta\left(\sigma(w_j^Tx_i)-\sigma(w_{j+\frac{m}{2}}^Tx_i) \right) \right)^2 \\
  & =&   \zeta^2  \sum_{i=1}^n\left(\sum_{j=1}^{\frac{m}{2}} \left(\sigma(w_j^Tx_i)-\sigma(w_{j+\frac{m}{2}}^Tx_i) \right) \right)^2 \\
  & \overset{\text{Assumption } \ref{assum2} }{\leq}&   \zeta^2 L^2  \sum_{i=1}^n\left(\sum_{j=1}^{\frac{m}{2}}(w_j-w_{j+\frac{m}{2}})^Tx_i  \right)^2 \\
  & \leq&   \zeta^2 L^2  \sum_{i=1}^n\left(\sum_{j=1}^{\frac{m}{2}}\|w_j-w_{j+\frac{m}{2}}\|_2 \|x_i\|_2  \right)^2\\
  & \overset{(***)}{=}&    \zeta^2 L^2  \sum_{i=1}^n\left(\sum_{j=1}^{\frac{m}{2}}\|w_j-w_{j+\frac{m}{2}}\|_2 \right)^2 \\
  & \leq&     n \left(\frac{m}{2}L \zeta\epsilon \right)^2,
\end{eqnarray}
where $(***):$ we assume $\|x_i\|_2\leq 1$ for $i=1,\cdots, n$. For general $\|x_i\|_2$, the difference is up to a constant.

  Recall $f^{\prime}=(1+\tilde{\eta})f^*$, we have
  \begin{eqnarray}
    \|f^\prime-y\|_2&=& \|(1+\tilde{\eta})f^*-y\|_2\\
    &=& \|f^*-y+\tilde{\eta}(f^*-y)+\tilde{\eta}y\|_2 \\
    &=& \|f^*-y\|_2+\tilde{\eta}\|f^*-y\|_2+\tilde{\eta}\|y\|_2 \\
    &\overset{(\ref{f*-y})}{\leq}& (1+\tilde{\eta}) \left(n \left(\frac{m}{2}L \zeta\epsilon \right)^2 + n C_y^2  \right)^{\frac{1}{2}} +\tilde{\eta}\|y\|_2 \\
    &\leq & (1+\tilde{\eta}) \left(n \left(\frac{m}{2}L \zeta\epsilon \right)^2 + n C_y^2 \right)^{\frac{1}{2}} +\tilde{\eta}\sqrt{n} C_y \\
    &\leq & 2 \left(n \left(\frac{m}{2}L \zeta\epsilon \right)^2 + n C_y^2 \right)^{\frac{1}{2}} +\sqrt{n} C_y,
  \end{eqnarray}

  where the last inequality is because $\tilde{\eta}$ is sufficiently small.
  Now, combining with the descent property $\|\bar{f}-y\|_2 \leq \|f^*-y\|_2$, we have
  \begin{eqnarray}
    \|f^{\prime}+\bar{f}-2y\|_2 &\leq & \|f^\prime -y\|_2+\|\bar{f}-y\|_2  \\
    &\leq &  \|f^\prime -y\|_2+\|f^*-y\|_2 \\
    &\leq &  
    3 \left(n \left(\frac{m}{2}L \zeta\epsilon \right)^2 + n C_y^2 \right)^{\frac{1}{2}} +\sqrt{n} C_y .
    \label{diffofsquare2}
  \end{eqnarray}

  We conclude the proof of Lemma \ref{lemma:estimateerror} by combining (\ref{estimateerrorsquare}) and (\ref{diffofsquare2}).

\section{Extension To Deep Networks}
\label{appendix:lemmadeep}

In this section, we discuss how to extend our analysis to deep networks. To do so, we apply the mirrored LeCun's initialization and the constrained formulation \eqref{constrained} to the last two layers and treat the output of the $(L-2)$-th layer as the input features. In the proof of Theorem  \ref{thm1}, the expressivity is guaranteed if the inputs $\{x_1,\cdots,x_n\}$ follow a {\it continuous joint distribution}, which is true under Assumption \ref{assum3}. Fortunately, for deep neural networks with $m_l\ge \frac{2n}{m_{l-1}}$ (where $m_l$ is the width of the $l$-th layer), the outputs of the $(L-2)$-th layer still follow a {\it continuous joint distribution} under Assumption \ref{assum2} and \ref{assum3}, so the expressivity can be shown using the similar technique as Theorem \ref{thm1}.
  This result is formally stated and proved in the Lemma \ref{lemma:deep} below.
\begin{lemma}
    Given a deep fully-connected neural network with $L$ layers:

  $$f(x;\theta)=w^{(L)} \sigma\left(w^{(L-1)} \ldots \sigma\left(w^{(2)} \sigma\left(w^{(1)} x\right)\right)\right),$$

  where $\sigma(\cdot): \mathbb{R} \rightarrow \mathbb{R}$ is the activation function, $w^{(l)}\in \mathbb{R}^{m_{l} \times m_{l-1}}$ are the weights, $l=1, \ldots, L$.
 Under Assumption \ref{assum2} and \ref{assum3}, suppose $m_l \geq m_{l+1}$, for $l\leq L-3$ and $m_{L-1}m_{L-2}\geq 2n$, then at the initialization $\theta^0$ (for $l \leq L-2$, LeCun's initialization is used; for the last two layers, the mirrored LeCun's initialization is used), for the inputs $\{x_1,\cdots, x_n\}$, the outputs of the $(L-2)$-th layer follow a continuous joint distribution.
 \label{lemma:deep}
\end{lemma}

To prove Lemma \ref{lemma:deep}, we first prove the following lemma:
\begin{lemma}
Suppose that  $\psi: \mathbb{R}^{k_1}\rightarrow \mathbb{R}^{k_2}$ is a analytic mapping and for almost every $u \in \mathbb{R}^{k_1}$, the Jacobian matrix  $J(u)$ of $\psi$ w.r.t. $u$ is of full row rank.
If $u\in \mathbb{R}^{k_1}$ follows a continuous distribution and $k_1\geq k_2$, then $\psi(u)$ also follows a continuous distribution.
\label{lemma:continuous}
\end{lemma}
\begin{proof}
Let $Z_0\subseteq \mathbb{R}^{k_2}$ be a zero measure set in $\mathbb{R}^{k_2}$,
We define $S_1(Z_0)=\{u\in \mathbb{R}^{k_1}\mid \psi(u)\in Z_0, J (u)\ \text{is non-singular}\}$.
By the definition of $S_1(Z_0)$, any $u\in S_1(Z_0)$ can be written as $u=(u_1^T, u_2^T)^T$, where $u_1\in \mathbb{R}^{k_2}$ and $J(u_1; u_2)$ is invertible ( $J(u_1; u_2)$ is the $l \times l$ submatrix of $J(u)$, similarly as in Lemma \ref{fullrank}).
Then by the Inverse Function Theorem, there exists some ball $\mathcal{B}_{\epsilon(u)}(u)\subseteq \mathbb{R}^{k_1}$ (centered at $u$ with radius $\epsilon(u)$) such that for any $u'=((u')^T_1, (u')^T_2)^T \in \mathcal{B}_{\epsilon(u)}(u)\cap S_1(Z_0)$, $u_1'=\tau(u_2', z')$ , where $z'=\psi(u')\in Z_0$ and $\tau $ is a smooth mapping in a neighborhood $\tilde{Z}_0$ of $(u_2, \psi(u))$.
Then for any $u$, there exists a rational point  $\bar{u}\in \mathbb{Q}^{k_1}$ and a rational number  $\bar{\epsilon}(u)\in \mathbb{Q}$ such that $u\in N(u):=\mathcal{B}_{\bar{\epsilon}(u)}(\bar{u})\subseteq \mathcal{B}_{\epsilon(u)}(u)$.
Since the collection of all open balls with a rational center and a rational radius is a countable set, we 
let $N_1, N_2, \cdots, N_n, \cdots$ be different $N(u)$ for $u\in S_1(V_0)$.
Then $S_1(Z_0)=\cup_{i=1}^{\infty}(N_i\cap S_1(Z_0))$.

We then only need to  prove that for any $i$, $N_i\cap S_1(Z_0)$ is of measure zero in $\mathbb{R}^k$.
We define the mapping $\tilde{\tau}: \tilde{Z_0}$ as $\tilde{\tau}(u_2', z')=u'$ if $z'=\psi(u')$.
Since $Z_0$ is of measure zero in $\mathbb{R}^{k_2}$, $\tilde{Z}_0$ is measure zero in $\mathbb{R}^{k_1}$.
Then because $\tilde{\tau}$ is smooth, the image of $\tilde{\tau}$ of the set $\tilde{V}_0$ is of zero measure in $\mathbb{R}^{k_1}$ (The image of a zero mesure set under a  smooth mapping is also measure zero).
Notice that $\mathcal{B}_{\bar{\epsilon}(u)}(\bar{u})\cap S_1(Z_0)$ is contained in the image $\tilde{\tau}(\tilde{Z}_0)$, it is also of zero measure. This finishes the proof.
\end{proof}
Now we prove Lemma \ref{lemma:deep}. When $l \leq L-3$,
let $u^{(l)} \in \mathbb{R}^{m_l}$ be the output vector of the $l$-th layer.
Then $u^{(l+1)}=\psi(w^{(l+1)0}, u^{(l)}),$
where $\psi$ is analytic and $w^{(l+1)0}$ is the initial parameter in the $l$-th layer.
We now prove that $u^{(l)}$ follows a continuous distribution by induction.
When $l=0$, it is true since $u^{(0)}=(x_1,\cdots,x_n)$ is just the input data, which follows a continuous distribution under Assumption \ref{assum3}.
Now suppose $u^{(l)}$ still follows a continuous distribution, we have: (i) since $w^{(l+1)0}$ follows a continuous distribution at the mirrored initialization,  $u^{(l)}$ and $w^{(l+1)0}$ follow a continuous joint distribution; (ii) Now, viewing $(w^{(l+1)0},u^{(l)} ) \in \mathbb{R}^{m_l m_{l+1}+m_l}$ as the input of $\psi$, when $m_l \geq m_{l+1}$, $J(w^{(l+1)0},u^{(l)})$ can be proved to be full row rank w.p.1. using the same technique as Theorem \ref{thm1}. In conclusion, we have $u^{(l+1)}$ follows a continuous distribution by Lemma \ref{lemma:continuous}.
Hence, we finish the proof of Lemma \ref{lemma:deep}.

\BLACK{
We further comment a bit on extending the trainability analysis to deep nets. For this part, it requires more detailed analysis because the input feature of the penultimate layer is changing along the training (which is fixed in the shallow case), this topic will be considered as future work.
Nevertheless, our idea motivates a better training regime for deep networks, and it is numerically verified in our experiments.}

\section{Implementation Details \& More Experiments}
\label{appendix:experiment}

\BLACK{
\subsection{Guidance on \texttt{PyTorch} Implementation}
\label{appendix:pytorch}
In this section, we provide sample code to implement the our proposed method for narrow nets training, which can  achieve small empirical loss as proved in Theorem \ref{thm3}. We formally state our training regime in Algorithm \ref{algo:ourmethod}. 

\begin{algorithm} 
  \caption{Our training regime}
  
\noindent {\bf Set up hyperparameters:}

$~~$ Choose a constraint size $\epsilon$, $\zeta,\kappa$  and a step size $\eta$.

$~~$ Define $B_{\epsilon}\left(w^{0}\right):=\left\{w \mid\left\|w-w^{0}\right\|_{F} \leq \epsilon\right\}$

$~~$ Define $B_{\zeta,\kappa}(v)=\{v| v \geq \zeta {\bf 1}, \text{and for } \forall v_j, v_j^{\prime}, v_j/v_j^{\prime} \leq \kappa   \}$. 

\noindent {\bf Set up the pairwise structure of $v$:}

 $~~$ Consider $f(x;\theta)=\sum_{i=1}^{\frac{m}{2}} v_j(\sigma(w_j^Tx)-\sigma(w_{j+\frac{m}{2}}^Tx) )$.

\noindent {\bf Initialization}: 

$~~$ Initialize $\theta^0=(w^0,v^0)$  by the mirrored LeCun's initialization, as shown in Algorithm \ref{initial}

\noindent {\bf Training:} 

$~~$ Update $v$ via Projected Gradient Descent: $v^{t+1}\leftarrow \mathcal{P}_{B(v)} (v^{t}-\eta \nabla_v \ell(\theta^t) )$.

$~~$ Update $w$ via Projected Gradient Descent: $w^{t+1}\leftarrow \mathcal{P}_{B_\epsilon(w^0)}(w^{t}-\eta \nabla_w \ell(\theta^t))$.

\noindent {\bf Until the final epoch $t=T$.}
\label{algo:ourmethod}
\end{algorithm}

Algorithm \ref{algo:ourmethod} can be adopted to deep nets by viewing $w$ as the  hidden weights in the penultimate layer (or the final block of ResNet \cite{he2016deep} in our computer vision experiments), and view $x$ as the feature outputted by all the previous layers.
As shown in Algorithm \ref{algo:ourmethod}, there are several key ingredients: the pairwise structure of $v$ in \eqref{constrained}; the mirrored initialization; and the PGD algorithm. We now demonstrate their implementation in \texttt{PyTorch}. Each of them only involves several lines of code changes based on the regular training regime.
}

\paragraph{The pairwise structure of $v$ \& The Mirrored LeCun's initialization.}.

\begin{lstlisting}[language=Python, caption= ]  
import torch
import torch.optim as optim
import copy
class ShallowNet(nn.Module):
    def __init__(self, n_input, n_hidden):
        super(ShallowNet, self).__init__()
        
        self.fc1 = nn.Linear(n_input, n_hidden1,bias=False)
        self.tanh=nn.Tanh()
        self.n_hidden1=n_hidden1
        #Cut down half the width of the output layer 
        self.fc2 = nn.Linear(int(n_hidden/2), 1, bias=False)
        
        #The mirrored initialization
        hidden_half = self.fc1.weight[0:int(n_hidden/2)]
        hidden_layer=torch.cat([hidden_half,hidden_half],dim=0)
        self.fc1.weight = torch.nn.Parameter(hidden_layer)
 
    
    def forward(self, x):

       
        x=self.fc1(x)
        h=self.tanh(x)
        
        #Keep the pairwise structure of v
        h1=h[:,0:int(self.n_hidden1/2)]
        h2=h[:,int(self.n_hidden1/2):self.n_hidden1]

        x_pred1=self.fc2(h1)
        x_pred2=self.fc2(h2)
        x_pred=x_pred1-x_pred2
    
        return x_pred

\end{lstlisting}

To extend the mirrored initialization to deeper nets such as ResNet, we just need to repeat the code line [12 - 14] for every hidden layer, including the BatchNorm layer and the CNNs in the shortcut layers in the Residue block.

\paragraph{Projected Gradient Descent.} We now demonstrate how to implement PGD. First, we need to copy the parameters at the initialization, will be used for projection.

\begin{lstlisting}[language=Python, caption=] 
    # Copy the parameters at the initialization, will be used for projection
    model_initial = copy.deepcopy(model)
\end{lstlisting}

Then we do the projection after each gradient update.

\begin{lstlisting}[language=Python, caption=] 
def train(model, model_initial, epoch, x,y, optimizer):

    #standard code in regular training
    clf_criterion=nn.MSELoss()
    model.train()
    for i in range(epoch):
        optimizer.zero_grad()
        pred=model(x=x)
        loss = clf_criterion(pred,y)  # calculate current loss
        loss.backward() # calculate gradient
        optimizer.step() # update parameters
        
        # Projection
        for para,para0 in zip(model.parameters(), model_initial.parameters()):
            #project the hidden layer
            if para.data.size()[0]==model.n_hidden1:
                if torch.norm(para.data - para0.data) > eps:
                    para.data = para0.data + eps * (para.data - para0.data) / torch.norm(para.data - para0.data)
            
            #project the output layer
            if para.data.size()[0]==1:
                para.data=projectv(para.data)
        
def projectv(v):
    vmax=torch.max(v)
    argmax=torch.argmax(v)
    vmin=torch.min(v)
    argmin=torch.argmin(v)
    #print(vmax/vmin)
    #print('vmax',vmax)
    #print('vmin',vmin)
    if vmin <0.001:
        #print('projectv1')
        v[argmin]=0.001
    if vmax/vmin>1:
        v[argmax]=1*vmin
        #print('projectv2')
    return v

\end{lstlisting}

\subsection{Details on Experimental Setup}
\label{appendix:experimentsetup}

Our empirical studies are based on the synthetic dataset, MNIST, CIFAR-10, CIFAR-100 and the R-ImageNet datasets. MNIST, CIFAR-10 and CIFAR-100 are licensed under MIT.
Imagenet is licensed under Custom (non-commercial). All the experiments are run on NVIDIA V100 GPU. Here, we introduce our settings on synthetic dataset and R-ImageNet.

\begin{itemize}
  \item Synthetic datset: For $i=1\dots, 1000$, we independently generate $x_i\in \mathbb{R}^{200}$ from standard independent Gaussian, and normalize it to $\|x_i\|_2=1$, and we set the ground truth as $y_i=(1^Tx_i)^2$ for $i=1,\cdots, 1000$. In short, sample size $n=1000$, input dimension $d=200$.
  \item R-ImageNet: This is a  specifically constructed "restricted" version of ImageNet, with resolution $224 \times 224$.
  
    The vanilla ImageNet dataset spans 1000 object classes and contains 1,281,167 training images, 50,000 validation images and 100,000 test images. In our experiments, we use a subset of ImageNet, namely Restricted-ImageNet (R-ImageNet). Similar with \cite{inkawhich2020perturbing}, we  leverage the WordNet~\cite{miller1998wordnet} hierarchical structure of the dataset such that each class in the R-ImageNet is a superclass category composed of multiple ImageNet classes, noted in Table~\ref{imagenet-info} as ``components''. For example, the ``bird'' class of R-ImageNet (both
    the train and validation parts) is the aggregation of ImageNet-1k classes: [10: `brambling', 11: `goldfinch', 12: `house finch', 13: `junco', 14: `indigo bunting'], more details can be seen in Table~\ref{imagenet-info}.  As a result, there are 20  super classes which contain a total of 190 vanilla ImageNet classes. 
\begin{table}[htb]
 \caption{Classes used in the R-ImageNet dataset. The class ranges are inclusive.}
 \label{imagenet-info}
 \centering
 \begin{threeparttable}
  \begin{tabular}{cc}
   \hline Class name & Corresponding ImageNet components \\
   \hline bird & {$[10,11,12,13,14]$} \\
   turtle & {$[33,34,35,36,37]$} \\
   lizard & {$[42,43,44,45,46]$} \\
   snake & {$[60,61,62,63,64]$} \\
   spider & {$[72,73,74,75,76]$} \\
   crab & {$[118,119,120,121,122]$} \\
   dog & {$[205,206,207,208,209]$} \\
   cat & {$[281,282,283,284,285]$} \\
   bigcat & {$[289,290,291,292,293]$} \\
   beetle & {$[302,303,304,305,306]$} \\
   butterfly & {$[322,323,324,325,326]$} \\
   monkey & {$[371,372,373,374,375]$} \\
   fish & {$[393,394,395,396,397]$} \\
   fungus & {$[992,993,994,995,996]$} \\
   musical-instrument & {$[402,420,486,546,594]$} \\
   sportsball & {$[429,430,768,805,890]$} \\
   car-truck & {$[609,656,717,734,817]$} \\
   train & {$[466,547,565,820,829]$} \\
   clothing & {$[474,617,834,841,869]$} \\
   boat & {$[403,510,554,625,628]$} \\
   \hline
  \end{tabular}
 \end{threeparttable}
\end{table}
\end{itemize}

In each dataset, the neural network architectures are chosen as follows, all of the following cases satisfy $m\geq \frac{2n}{d}$ or $m_{L-1} \geq \frac{2n}{m_{L-2}}$, where $m_l$ is the width of the $l$-th layer.
\begin{itemize}
  \item Synthetic dataset: we use 1-hidden-layer neural networks with Tanh activation (except for the last layer, where the output dimension equals 1 and no Tanh applied). We study different widths of the hidden layer among $m=20,40,80,100,200,400,800,100,1200$.  All of these cases satisfy $m\geq \frac{2n}{d}$. 
  \item MNIST: we use 2-hidden-layer neural networks with ReLU activation (except for the last layer, where the output dimension equals the number of classes and no activation applied). The input dimension $d=784$, the width of the 1st layer is fixed with $m_1=784$ and we study different widths of the 2nd hidden layer among $m_2=64,128,256,512,784,1024$. All of these cases satisfy $m_{L-1} \geq \frac{2n}{m_{L-2}}$, where $m_l$ is the width of the $l$-th layer.
  \item CIFAR-10, CIFAR-100 and R-ImageNet: we use ResNet-18 and we try different number of channels in the 4th block (the i.e., the final block) among $m=64,128,256,512$ (for regular ResNet-18, the default number of channels in the 4st block should be 512). All of these cases satisfy $m_{L-1} \geq \frac{2n}{m_{L-2}}$, where $m_l$ is the width of the $l$-th layer.
\end{itemize}

In each dataset, the setup for algorithms are as follows: as for training regime, we apply the mirrored LeCun's initialization for all the neural network structures mentioned above, and for regular training, we use the regular LeCun's initialization. We use square loss for the synthetic dataset,  and multi-class cross entropy loss is used for the rest of the cases. During training, CIFAR-10, CIFAR-100 images are padded with 4 pixels of zeros on all sides, then randomly flipped (horizontally) and cropped. R-ImageNet images are randomly cropped during training and center-cropped during testing. Global mean and standard deviation are computed on all the training pixels and applied to normalize the inputs on each dataset. As it is required in problem \eqref{constrained}, in our training regime, the optimization variable for and the output layer is cut off to half, i.e.,  $v=(v_1,\cdots,v_\frac{m}{2})$, the other half is always $-v$; as for the hyperparameters of $B(v)$, we set $\zeta=0.001$ and $\kappa=1$. After each iteration, relevant parameters will projected in to their feasible sets. 
In addition to the general setup above, more customized hyperparameters are listed as follows:

\begin{itemize}
  \item Synthetic dataset: For both our training regime and the regular training regime, $B_\epsilon(w)$ constraint is added on the weights in the hidden layer with $\epsilon=0.1, 0.2, 0.4, 0.8, 1, 2, 4 ,8, 10, 1000$ ($\epsilon=1000$ is equivalent to the unconstrained updates for $w$).
  Gradient Descent with 0.9 momentum is used,  and we use different constant learning rates $lr_1, lr_2$ for hidden weights and outer weights, in all cases with different $m$ and $\epsilon$, we grid search learning rate $lr_1$=[1e-4,1e-3,5e-3,1e-2,5e-2,1e-1,5e-1], $lr_2$=[1e-4,1e-3,5e-3,1e-2,5e-2,1e-1,5e-1] and report the best results. The neural network is trained for 200000 iterations.  
  \item MNIST: For our training regime, $B_\epsilon(w)$ constraint is added on the weights in the 2nd layer with $\epsilon=0.1, 0.2, 0.4, 0.8, 1, 2, 4 ,8, 10, 1000$ ($\epsilon=1000$ is equivalent to the unconstrained updates for $w$). In each case, we either use Adam with 0.001 initial learning rate and 1e-4 weight decay, or Stochastic Gradient Descent (SGD) with 0.01 initial learning rate, 0.9 momentum and 5e-4 weight decay, and we report the best results.  For both training regime and the regular training regime, we use cosine annealing learning rate scheduling \cite{loshchilov2016sgdr} with $T_{\max}$=number of epochs , and the neural network is trained for 200 epochs and batch size of 64 is used.

  \item CIFAR-10, CIFAR-100: For our training regime , $B_\epsilon(w^0)$ constraint is added on the 4-th block with different constraint size among $\epsilon=0.1, 0.2, 0.4, 0.8, 1, 2, 4 ,8, 10, 1000$ ($\epsilon=1000$ is equivalent to the unconstrained updates for $w$). For both our training regime and the regular training regime, SGD with 0.1 initial learning rate, 0.9 momentum and 5e-4 weight decay is used, and we use cosine annealing learning rate scheduling with $T_{max}$=number of epochs, and the neural network is trained for 600 epochs and batch size of 128 is used. 
   
  \item R-ImageNet: For our training regime,  $B_\epsilon(w^0)$ constraint is added on the 4th block (i.e., the final block) with different constraint size among $\epsilon=0.01, 0.1, 1, 1000$ ($\epsilon=1000$ is equivalent to the unconstrained updates for $w$), and the weights in the 4th block are projected into the constraint after each mini-batch iteration. For both our training regime and the regular training regime, SGD with 0.1 initial learning rate, 0.9 momentum and 5e-4 weight decay is used, we use a stage-wise constant learning rate scheduling with a multiplicative factor of 0.1 on epoch 30, 60 and 90.
  The neural network is trained for 90 epochs and batch size of 256 is used. 
\end{itemize}

\subsection{Test Accuracy on MNIST}

Figure \ref{appendix:mnistfigure} shows the test accuracy in our training regime vs regular training regime in MNIST. With proper choice of $\epsilon$, our training regime leads to higher test accuracy.

\begin{figure}[htbp]
  \centering
      \includegraphics[width=3.2in]{ 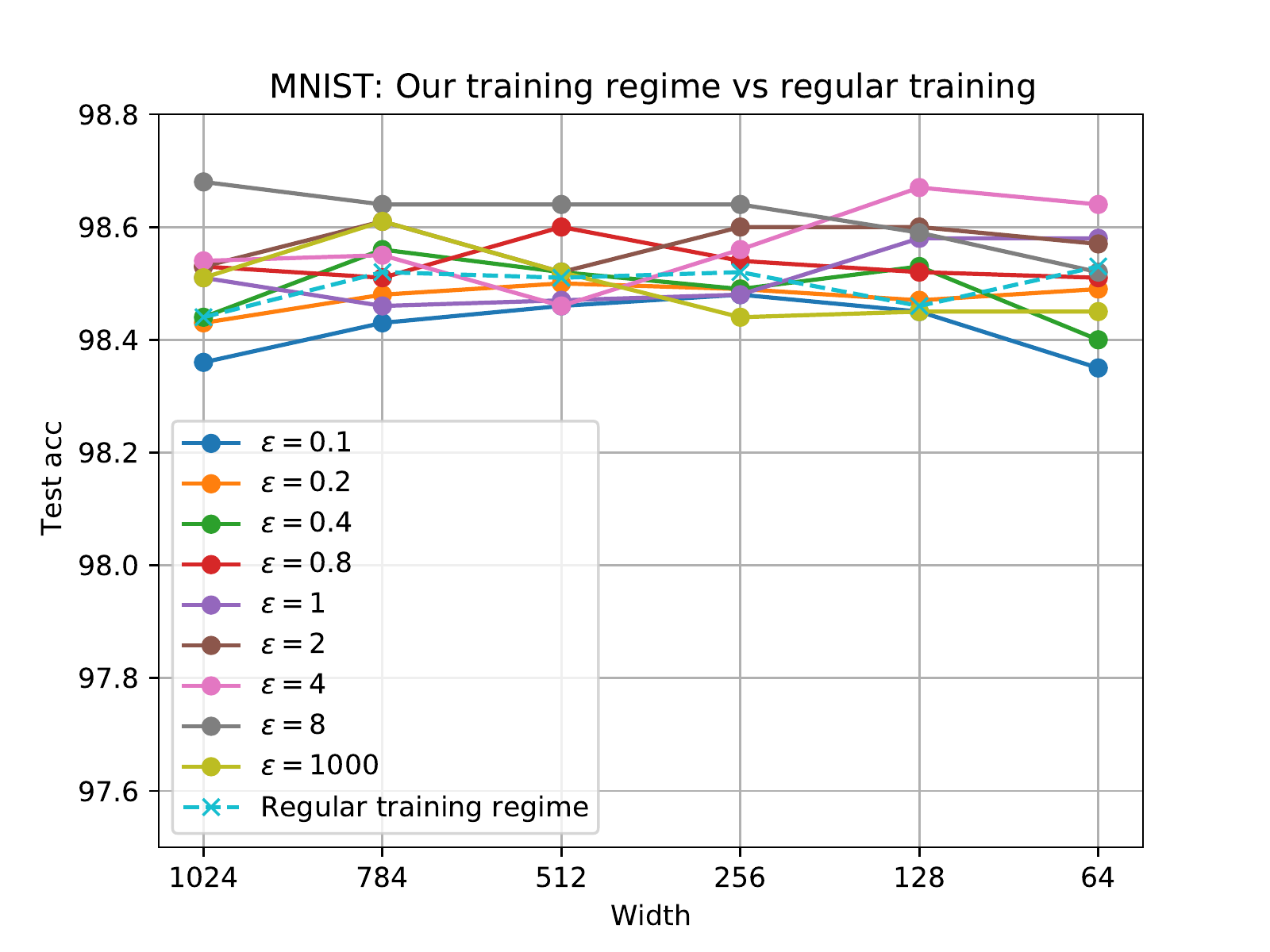}
  \caption{MNIST: test accuracy in our training regime with different $\epsilon$ vs regular training regime. \BLACK{In x-axis, width stands for width of 2nd layer (we use 2-hidden-layer neural nets here).}}
  \label{appendix:mnistfigure}
\end{figure}

\subsection{Test Accuracy on CIFAR-10 \& CIFAR-100 }

In CIFAR-10 and CIFAR-100 dataset, our training regime and regular training regime have similar performance (see Figure \ref{cifar:compare}). In several cases, our training regime leads to higher test accuracy. Here, regular training will not fail when we reduce the width of 4-th block, perhaps this is due to the strong expressivity of ResNet-18. In comparison, on a  more complicated dataset such as R-ImageNet, narrowing ResNet-18 will jeopardize the regular training (as illustrated in Section \ref{section:experiment} and the following subsection).

\begin{figure}[htbp]
  \centering
  \subfigure[CIFAR-10]{
    \begin{minipage}[t]{0.5\linewidth}
    \centering
    \includegraphics[width=2.8in]{ 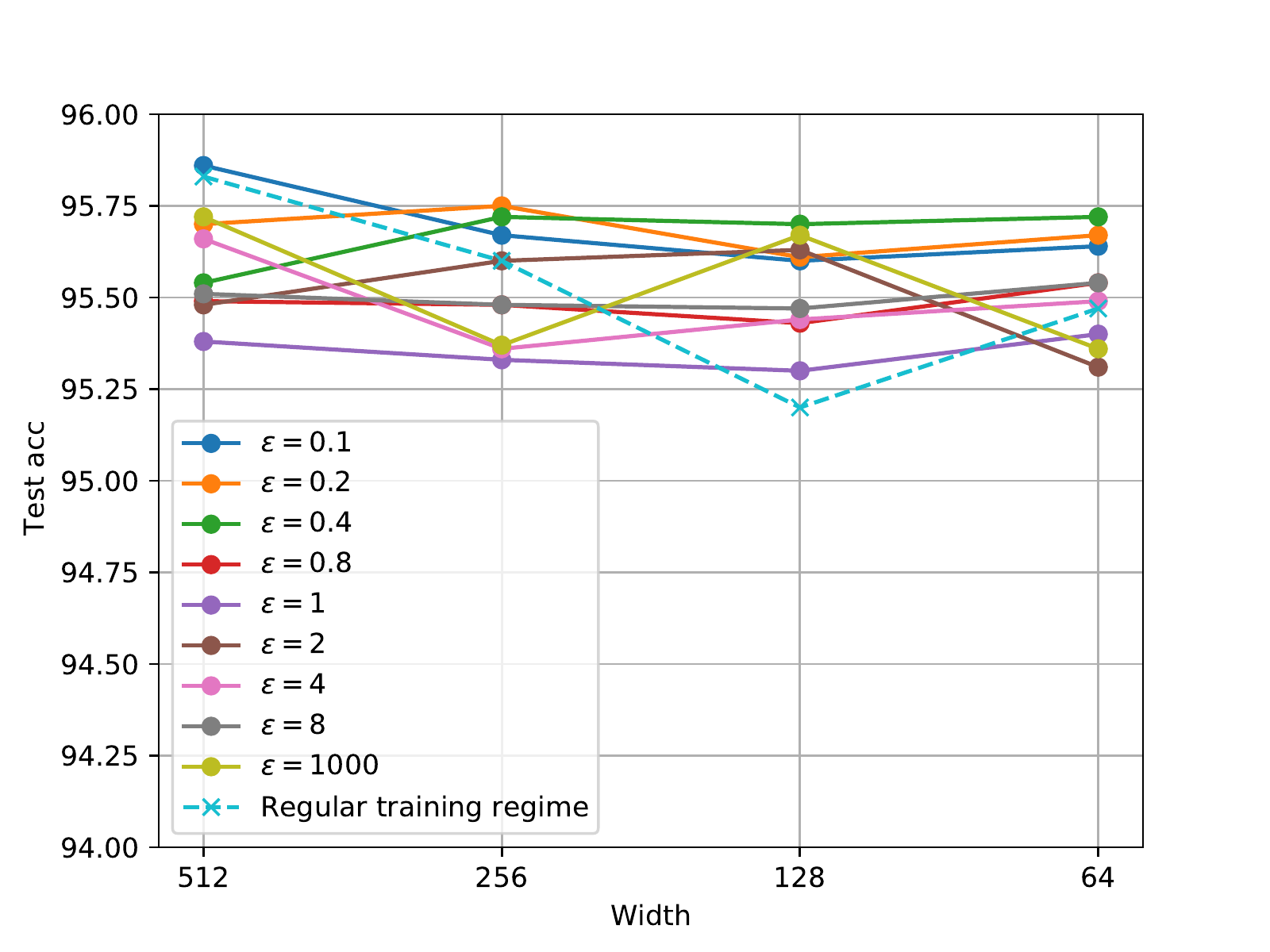}
    \end{minipage}%
    }%
    \subfigure[CIFAR-100]{
      \begin{minipage}[t]{0.5\linewidth}
      \centering
      \includegraphics[width=2.8in]{ 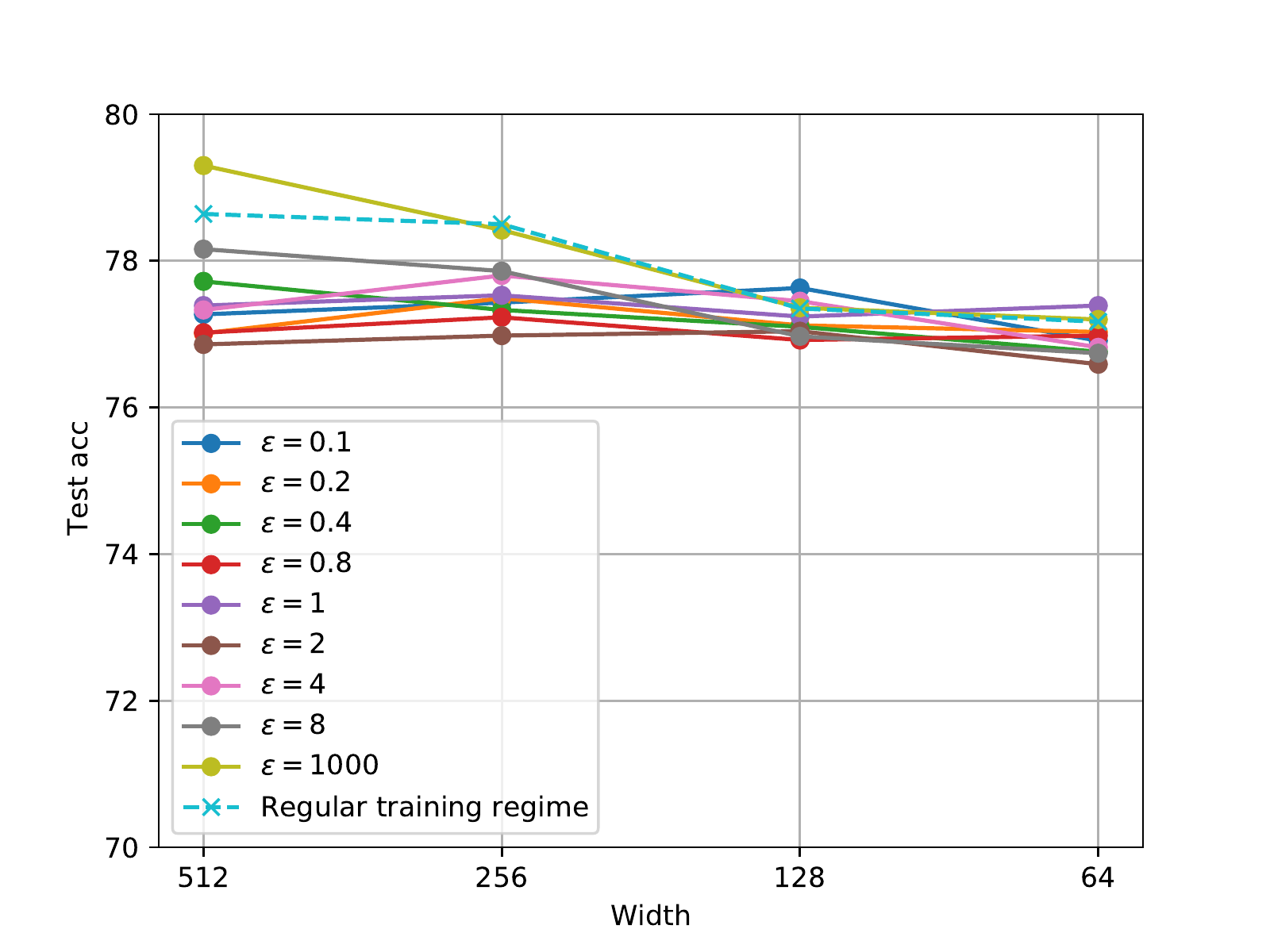}
      \end{minipage}%
      }%
  \centering
  \caption{CIFAR-10 \& CIFAR-100: test accuracy in our training regime with different $\epsilon$ vs regular training regime. \BLACK{In x-axis, width stands for the number of channels in the final CNN block of ResNet-18.}}
  \label{cifar:compare}
\end{figure}

\subsection{Test Accuracy on R-ImageNet}

In this subsection, Figure \ref{rimagenet:shaded} is the same figure as Figure \ref{rimagenet:testacc} in the full paper, but with 90\% confidence error bars (based on 5 seeds). Besides, Figure \ref{rimagenet:epoch} shows a selected result of Loss \& Accuracy per epoch in our training regime with $\epsilon=0.04$ \& regular training regime (here, we present the early-stopped results). As a result, our training regime can reduce the number of parameters in the 4-th block of ResNet-18 by up to 94\% while maintaining competitive test accuracy, especially when $\epsilon$ is small. In comparison, regular SGD does not perform well in narrow cases.

\begin{figure}[htbp]
\centering
 \includegraphics[width=0.5\textwidth]{ 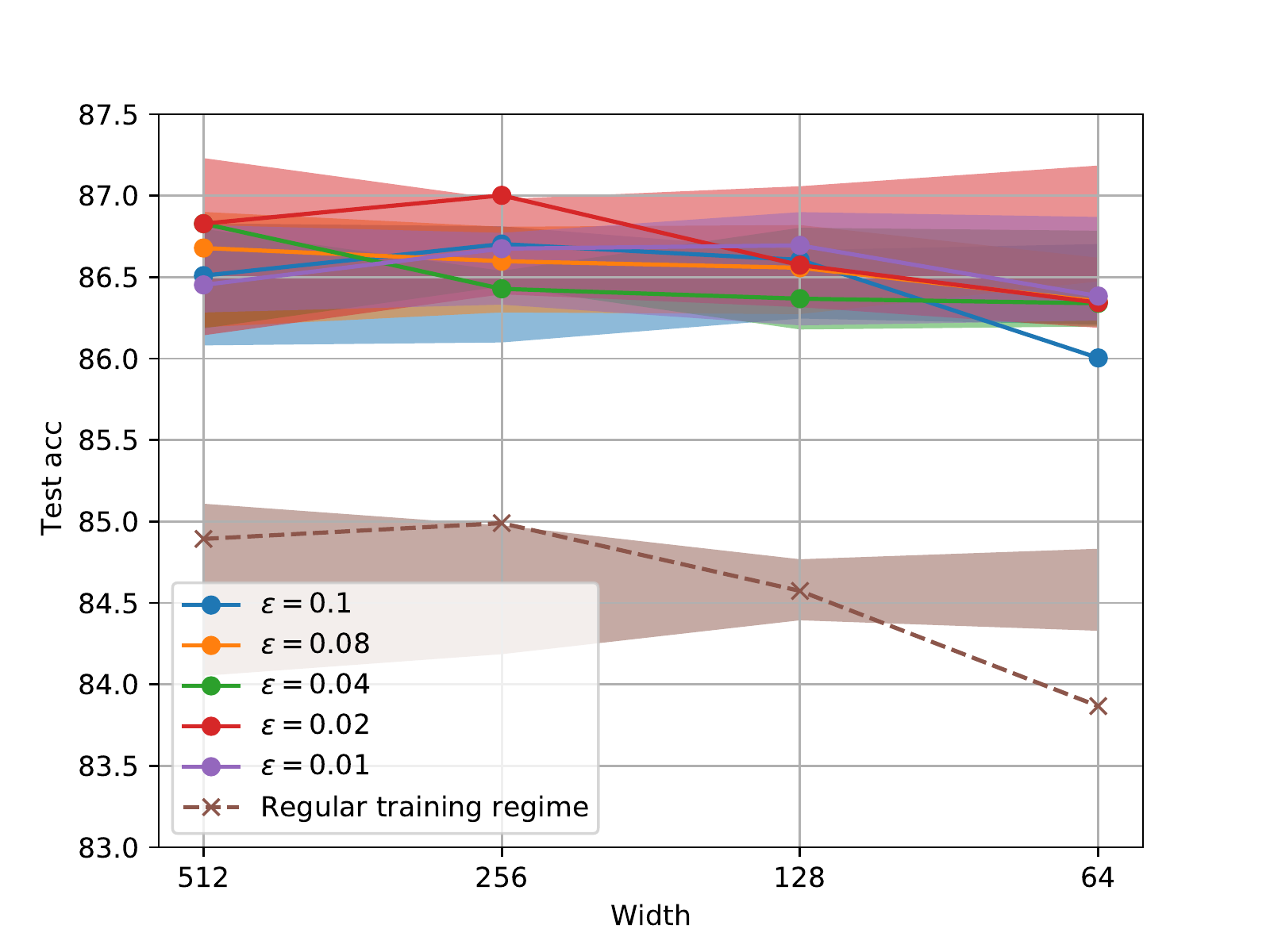} 
 \caption{R-ImageNet: test accuracy under our training regime with different $\epsilon$ vs regular training regime. \BLACK{In x-axis, width stands for the number of channels in the final CNN block of ResNet-18.} The solid \& dotted lines are averaged results over 5 seeds, the shaded areas indicates the 90\% confidence intervals. }
 \label{rimagenet:shaded} 
\end{figure}

\begin{figure}[htbp]
  \centering
  \subfigure[Loss]{
    \begin{minipage}[t]{0.5\linewidth}
    \centering
    \includegraphics[width=2.8in]{ 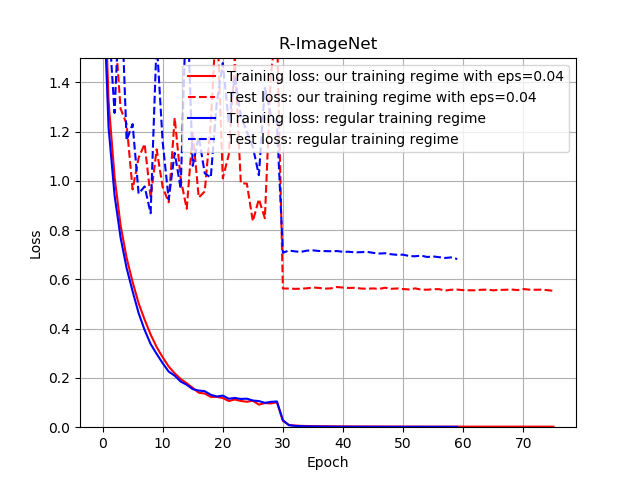}
    \end{minipage}%
    }%
    \subfigure[Accuracy]{
      \begin{minipage}[t]{0.5\linewidth}
      \centering
      \includegraphics[width=2.8in]{ 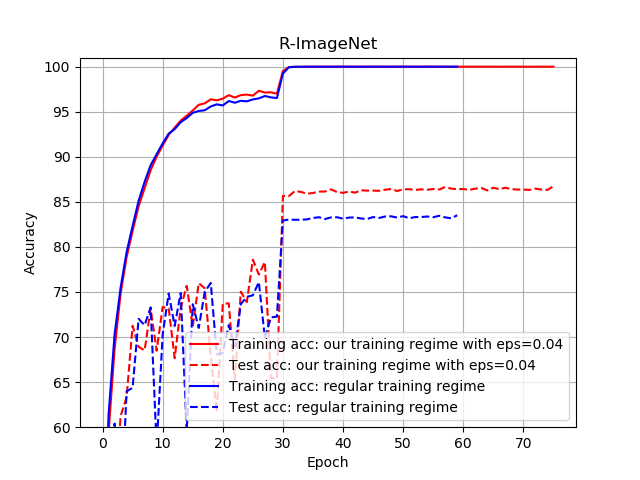}
      \end{minipage}%
      }%
  \centering
  \caption{R-ImageNet (selected): loss \& Accuracy per epoch in our training regime with $\epsilon=0.04$ \& regular training regime.}
  \label{rimagenet:epoch}
\end{figure}

\subsection{Results on Random-Labeled CIFAR-10}
\label{appendix:randomlabel}

 Since the main claim of our work is about memorization of {\it any} labels, we further explore the performance of our training regime  even when the labels are not the correct ones.
To do so, we further carry out experiments  on the random-labeled CIFAR-10, where all the labels are randomly shuffled (same as in Zhang et al. \cite{zhang2021understanding}).  We train a 1-hidden-layer network with width =1024, 2048, 4096 (smaller than $n$=50000) on the random-labeled CIFAR10 dataset. The hyperparameters in our constrained training regime \eqref{constrained} are $\epsilon=10$, $\kappa=1$ and initial learning rate $=$0.1. We use a stage-wise constant learning rate scheduling with a multiplicative factor of 0.1 on epoch 150, 225, 450. The result is shown in Table \ref{table:randomlabel}: after 1000 epochs, we can achieve more than 99\% train accuracy, almost perfectly fit the random labels. Note that even though ReLU does not fall into our analysis framework, it works a bit better than Tanh. Extending our results to ReLU activation would be our intriguing future work.

\begin{table}
  \caption{Results on the random-labeled CIFAR-10}
  \label{table:randomlabel}
  \centering
  \begin{tabular}{lllll}
    \toprule
     Width    & Epoch    & Activation &Train acc & Test acc  \\
    \midrule
    1024 & 1000 & ReLU & 0.9931 & 0.1011 \\
 2048 & 1000 & ReLU  & 0.9984 & 0.1022 \\
 4096 & 1000 & ReLU  & 0.9998 & 0.0962 \\
 1024 & 1000 & Tanh & 0.9872 & 0.0991 \\
 2048 & 1000 & Tanh & 0.9927 & 0.1024 \\
4096 & 1000 & Tanh & 0.9938 & 0.0962 \\
    \bottomrule
  \end{tabular}
\end{table}

\end{document}